\documentclass{article}
\usepackage[preprint]{neurips_2026}

\usepackage[utf8]{inputenc}
\usepackage[T5,T1]{fontenc}
\usepackage{hyperref}
\hypersetup{pdftitle={Training-Free Cultural Alignment of Large Language Models via Persona Disagreement}, pdfauthor={Duy-Minh Dao-Sy, Trung-Kiet Huynh, Tuan Nguyen, Chi-Nguyen Tran, Phu-Hoa Pham, Phu Quy Nguyen Lam, The Anh Han, Long Tran-Thanh}}
\usepackage{url}
\usepackage{titletoc}
\titlecontents{section}
  [2.5em]
  {\addvspace{0.4em}\bfseries}
  {\contentslabel{2.5em}}
  {}
  {\hfill\contentspage}
\titlecontents{subsection}
  [5.8em]
  {}
  {\contentslabel{3.3em}}
  {}
  {\titlerule*[0.5em]{.}\contentspage}

\usepackage{booktabs}
\usepackage{tabularx}
\usepackage{array}
\usepackage{amsfonts}
\usepackage{amsmath}
\usepackage{amssymb}
\usepackage{nicefrac}
\usepackage{xcolor}
\usepackage{graphicx}
\usepackage{float}
\usepackage{subcaption}
\usepackage{algorithm}
\usepackage{algpseudocode}
\usepackage{multirow}
\usepackage{makecell}
\usepackage{amsthm}
\usepackage{rotating}
\newtheorem{theorem}{Theorem}
\newtheorem{proposition}{Proposition}
\newtheorem{lemma}{Lemma}
\newtheorem{corollary}[theorem]{Corollary}

\definecolor{improvegreen}{RGB}{0,128,0}
\definecolor{worsered}{RGB}{180,0,0}
\providecommand{\gain}{}\renewcommand{\gain}[1]{\textcolor{improvegreen}{#1}}
\providecommand{\loss}{}\renewcommand{\loss}[1]{\textcolor{worsered}{#1}}

\title{%
  Training-Free Cultural Alignment of
  Large Language Models via Persona Disagreement%
}

\author{%
  Duy-Minh Dao-Sy$^{1,2,5,*}$\hspace{0.5em}%
  Trung-Kiet Huynh$^{1,2,5,*}$\hspace{0.5em}%
  Tuan Nguyen$^{3,*}$\hspace{0.5em}%
  Chi-Nguyen Tran$^{1,2,5}$ \\
  \textbf{Phu-Hoa Pham$^{1,2,5}$}\hspace{0.5em}%
  \textbf{Phu Quy Nguyen Lam$^{1,2,5}$}\hspace{0.5em}%
  \textbf{The Anh Han$^{4}$}\hspace{0.5em}%
  \textbf{Long Tran-Thanh$^{3,\dagger}$} \\[6pt]
  {\normalfont $^{1}$University of Science, Ho Chi Minh City, Vietnam} \\
  {\normalfont $^{2}$Vietnam National University, Ho Chi Minh City, Vietnam} \\
  {\normalfont $^{3}$University of Warwick, Coventry, United Kingdom} \\
  {\normalfont $^{4}$Teesside University, Middlesbrough, United Kingdom} \\
  {\normalfont $^{5}$CompactAI} \\[4pt]
  {\normalfont\small \texttt{\{23122041,23122039,23122044,23122030,23122048\}@student.hcmus.edu.vn}} \\
  {\normalfont\small \texttt{\{tuan.nguyen.1,long.tran-thanh\}@warwick.ac.uk}, \texttt{T.Han@tees.ac.uk}}
}

\begin{document}

\maketitle
{\let\thefootnote\relax\footnotetext{$^{*}$Equal contribution.\quad $^{\dagger}$Lead senior author.}}

{\let\thefootnote\relax\footnotetext{%
  Accepted at the 40th Conference on Neural Information Processing Systems (NeurIPS 2026).
}}

\begin{abstract}
Large language models now take part in decisions that rest on moral
preferences, and those preferences differ by country. A model's default
preferences sit closest to one population, so a single deployment serves every
country with one country's values. Existing fixes need per-culture training
data, a per-country reward model, or access to model internals, which the
countries furthest from that default rarely have. We show how far a country can
be reached without any of them.
Our method, DISCA, keeps the model's weights frozen, sees no country-specific
human targets, and reads only the model's score for each answer option. It
builds a country as a panel of four persona prompts grounded in World Values
Survey microdata. The panel's consensus sets a country-conditioned anchor, a
loss-averse refinement beyond that anchor is kept only when two independent
runs of it agree, and a running country-level prior pools the decisions already
made for the same country. On the 20-country MultiTP benchmark, DISCA lowers
misalignment to human preference vectors on all seven frozen backbones we
evaluate, by 10 to 24 percent on six of them, and a 14B backbone with DISCA
ends closer to the human preferences than a vanilla 70B backbone. DISCA beats
every training-free control we test on a shared grid, and on the two backbones
where we can replay the exact same model scores it also beats probability
averaging, logit consensus and majority vote over the same personas when those
rules receive no country prior. Switching
each stage off on the logged runs shows that the persona panel and the
country-level prior carry most of the gain. Beyond crash dilemmas, DISCA
improves social-norm classification on NormAd and moves answers closer to
GlobalOpinionQA survey distributions than vanilla decoding does.
\end{abstract}

\section{Introduction}
\label{sec:intro}

Large language models now take part in decisions that trade one person's
interests against another's. A content moderator asks a model whether a post
that mocks a religion should stay up~\citep{kumar2024watch, openai2023moderation}.
Language models already answer clinical questions~\citep{singhal2023clinical},
and a hospital assistant built on one may help a family decide how much risk an
ageing parent should be allowed to take. Language models also act as driving
decision makers~\citep{sha2023languagempc}, and a designer of an autonomous
vehicle may ask one how it should weigh the people it could harm in a
crash~\citep{awad2018moral}. None of these decisions has a
factual answer. Each one rests on a moral preference, and when a model is
involved, the preference that gets used is the model's default. The person
receiving the decision does not see that default and has limited ability to
change it~\citep{santurkar2023opinions}.

That default cannot suit everyone, because people do not share one set of
moral preferences. The Moral Machine experiment collected 40 million decisions
from 233 countries and territories on the same crash dilemmas and found that
the answers vary by country: whether to favour the young over the old, the many
over the few, or the fit over the frail~\citep{awad2018moral}. Language models
do not sit at a neutral point in this space. Their answers resemble those of
Western, educated, industrialised, rich, and democratic
(WEIRD)~\citep{henrich2010weirdest} populations more than those of other
countries~\citep{atari2023humans, durmus2023global}, they stereotype the moral
values of non-Western countries~\citep{zewail2025moral}, and post-training
alignment adds further shifts that were not intended~\citep{ryan2024unintended}.
MultiTP, which extends the Moral Machine protocol to 107 languages, scores 19
language models against each country's measured preferences and finds weak
agreement for most of them~\citep{jin2025multitp}. A model deployed in a hundred
countries carries the moral defaults of one population into all of them.

We see this as an alignment problem in the plain sense of the term. Alignment
work aims to make a model act on human values. That goal is only well defined
once we say whose values, and for a model that serves the whole world, they
differ by population~\citep{sorensen2024roadmap}. A model aligned to one
population is, by the same measure, misaligned for the others. The cost is also
unevenly spread. The countries whose preferences sit furthest from the model's
default are, in general, the countries with the least preference data, compute
and vendor access to correct it. Any remedy therefore has to reach those
countries too.

The existing fixes each ask for something the long tail does not have.
Fine-tuning and culture-aware adapters need culture-specific training data and
a training run for every culture~\citep{zhang2026culturemanager, yao2024caredio}.
Reward-guided decoding needs a reward model for every
country~\citep{khanov2024args, mudgal2024controlled}. Activation steering needs
write access to the model's internals~\citep{arditi2024refusal,
turner2023activation}. The cheapest fix is a prompt, and in our experiments a
paragraph that describes the country's aggregate values barely moves the model
away from its default (Table~\ref{tab:main_baseline_sanity}). We therefore work
under a deliberate constraint: the model weights stay frozen, the controller
receives no country-specific human targets at inference, and it has no reward
model and no access to activations. Its only interface is the model's score for
each answer option: the two decision-token logits on our main benchmark, and a
small set of option scores or candidate log-probabilities on the transfer
tasks. The human preference vectors are used only to score outputs after the
fact. The method does need one public resource, survey microdata to build the
panel. The World Values Survey covers about 66 countries and territories, so
DISCA reaches much further than per-country training, but not the whole long
tail.

Three observations led to our approach. First, a persona prompt is a strong
control, and in fact too strong. Sociodemographic personas shift a model's
moral decisions substantially more than the same attributes shift human
respondents, most of all for political personas~\citep{kim2025persona}. A
single persona therefore overshoots, and it reflects whichever person the
prompt writer chose rather than the country as a whole. Second, a country's
values are not a single point. World Values Survey (WVS)
microdata~\citep{haerpfer2022wvs} show that values inside one country vary with
age cohort and social position, as modernisation theory
predicts~\citep{inglehart2005modernization}. Third, survey researchers solved
this problem long ago with a panel. A pollster does not trust one respondent.
They ask many, report the average, and attach a margin of error, and a
difference smaller than that margin is not acted on. Together these change the
question we ask. Instead of asking how to move a model toward a country, we ask
how far any move should be allowed to go. Our central design principle is that
when the personas for a country disagree on a scenario, any refinement beyond
their consensus becomes unreliable, and the controller should keep less of it.

We build the country as a small survey panel. Picture asking four local guides
the way to the station. You walk in the direction their answers average out
to. If the four agree closely, you also let them refine the last few steps. If
they disagree, you stop at the average. \textbf{DISCA}
(Disagreement-Informed Steering for Cultural Alignment) does the same with a
frozen model. For one country and one dilemma, it runs four persona prompts
grounded in WVS microdata, three age cohorts and one country-wide aggregate,
and records how strongly the model favours one option over the other under
each persona (its decision-token logit gap). The panel then answers three
questions:
\begin{center}
\small
\setlength{\tabcolsep}{5pt}
\begin{tabular}{@{}p{0.27\linewidth}p{0.69\linewidth}@{}}
\toprule
Question & Answered by \\
\midrule
Which way should the model move, and roughly how far? & The panel consensus, which forms a country-conditioned anchor. \\
Can that step be refined without hurting any persona? & A Prospect-Theory importance-sampling refinement~\citep{kahneman1979prospect, tversky1992advances} that penalises a change helping some personas while harming others. \\
Is the refinement reliable on this scenario? & A dual-pass gate: two independent runs must agree, and the refinement shrinks toward zero when they do not. \\
\bottomrule
\end{tabular}
\end{center}
A fourth stage, a running country-level prior, pulls each decision toward the
average of the decisions the controller has already made for the same country.
A simple noise model makes the principle precise. If each persona's gap is the
country's true gap plus independent noise, then the best fraction of the
consensus step to apply, in the mean-squared-error sense, falls as the panel
variance rises (Proposition~\ref{prop:shrinkage}). DISCA follows that shape
with a statistic it can compute at runtime.

We evaluate mainly on MultiTP. Its per-country targets come from millions of
human judgments rather than from annotators or model judges. Its scenarios are
conjoint-matched, so the model's gap between the two decision tokens is a
preference on one moral attribute and nothing else, which is exactly the
quantity a score-level controller can adjust. Its scenarios are rendered in
each country's own language. And it is the hard case: a crash dilemma has no
factual answer and countries disagree on it. We then test whether the
signal carries over to five benchmarks with other output types and
populations.

Across 20 countries on four continents and seven open-weight backbones from
five model families, DISCA lowers misalignment for every backbone, by 10 to 24
percent for six and by 3 percent for the smallest
(Table~\ref{tab:main_macro_summary}). A 14B backbone with DISCA ends closer to
the human vectors than the 70B backbone with or without it. On a shared
evaluation grid DISCA beats five training-free controls: a WVS profile prompt,
a persona-ensemble prompt, activation steering, MC-Dropout and a fixed per-country logit offset
(Table~\ref{tab:main_baseline_sanity}). Given exactly the same persona logits,
it also beats probability averaging, logit consensus and majority vote when
those rules receive no country prior (Table~\ref{tab:same_logit_verified}). Because every stage of the controller is
logged per scenario, we can take it apart on its own runs, and we find that
the persona panel and the country-level prior carry most of the gain
(\S\ref{sec:ablation}). Beyond MultiTP, the panel improves on vanilla decoding
for NormAd social-norm classification, GlobalOpinionQA opinion distributions
and MAKIEval free-form generation (Table~\ref{tab:generalisation}).

\paragraph{Contributions.}
\begin{itemize}
\item \textbf{Setting.} We frame cultural alignment as an inference-time
problem for a frozen model with no per-country training, reward model,
activation access, or country-specific human targets.
\item \textbf{Method.} We give a survey-panel controller for that setting, with
a variance-aware shrinkage proposition that fixes the shape of its gate.
\item \textbf{Evidence.} We evaluate it on twenty countries, seven backbones,
five training-free controls, three same-panel aggregation rules and five
transfer benchmarks, and we decompose it on its own logged runs to show which
of its parts carry the gain.
\end{itemize}

\section{Related Work}
\label{sec:related}

\paragraph{Benchmarks and cultural evaluation.}
LLM alignment to MultiTP's country-specific Average Marginal Component Effects
(AMCEs)~\citep{jin2025multitp,
awad2018moral} varies widely across model families~\citep{takemoto2024moral,
ahmad2025largescale}, and model answers resemble WEIRD populations more than
others~\citep{atari2023humans, durmus2023global}, with post-training adding
unintended shifts~\citep{ryan2024unintended}. \citet{kim2025persona} show that
sociodemographic personas shift LLM moral decisions more than they shift
humans. This motivates our design: rather than using personas to steer toward
one viewpoint, we use them to \emph{measure} within-country disagreement and
turn it into a bounded correction.

\paragraph{Inference-time alignment.}
Pluralistic alignment can be steered at inference time by conditioning on a
target group~\citep{sorensen2024roadmap}, but existing approaches each give up one
of the constraints we target. Activation steering~\citep{arditi2024refusal,
turner2023activation} needs white-box access, reward-guided
decoding~\citep{khanov2024args, mudgal2024controlled} needs per-country reward
models, and culture-aware adapters~\citep{zhang2026culturemanager,
yao2024caredio} need per-culture training. Training-time prospect-theoretic
alignment~\citep{ethayarajh2024kto} uses the same value shape to shape a
trained policy; we apply it at test time, per scenario, over heterogeneous
personas. CultureSPA~\citep{xu2025culturespa} makes one model pluralistic by
fine-tuning on culture-conditioned survey data that it generates, which needs a
training run our setting rules out; we use its culture-aware prompt as a
baseline. \citet{yuan2025probabilistic} aggregate the moral judgments of
several LLMs with a reliability-weighted probabilistic model and then tune the
embeddings of the models that disagree; we compare against a training-free
version of their aggregation step (App.~\ref{app:decomposition}).

\paragraph{Panels, consensus and ensembles.}
Structured persona diversity has been used for robust collaborative
reasoning~\citep{dong2026choir}, and cultural consensus theory separates shared
cultural structure from over-regularised agreement~\citep{pothugunta2026cct}.
Self-consistency~\citep{wang2023selfconsistency}, LLM
debate~\citep{du2023improving, liang2024encouraging} and standard ensembles
usually average, vote or adjudicate over diverse outputs. DISCA instead uses the spread of
persona logit gaps to decide how much refinement beyond the consensus to keep.
Standard calibration, including temperature scaling and
MC-Dropout~\citep{gal2016dropout, kwon2025dropouts}, changes confidence in a
general way; DISCA applies a scenario-specific, loss-averse correction.
Cultural benchmarks such as MENAValues~\citep{zahraei2025menavalues}
(one region), CDEval~\citep{wang2024cdeval} (Hofstede dimensions) and
AlignCultura~\citep{kashyap2026aligncultura} (helpful, honest and harmless
data built on the UNESCO taxonomy) probe coverage and capability questions; we
evaluate on the first two.

\section{Disagreement-Informed Steering for Cultural Alignment}
\label{sec:method}

We build DISCA in three steps. \S\ref{sec:method:setup} fixes the setting.
\S\ref{sec:method:principle} states the principle the controller follows:
under a simple model of the panel, the fraction of the consensus correction
that minimises error falls as the panel disagrees. That fraction cannot be
computed from what the panel supplies. \S\ref{sec:method:estimator} therefore
builds a controller of the same shape out of quantities that can be computed
at runtime: a consensus anchor, a loss-averse refinement around it, a gate that
keeps the refinement only when two independent runs of it agree, and a running
country-level prior (\S\ref{sec:method:prior}).

\begin{figure}[ht!]
  \centering
  \includegraphics[width=\linewidth]{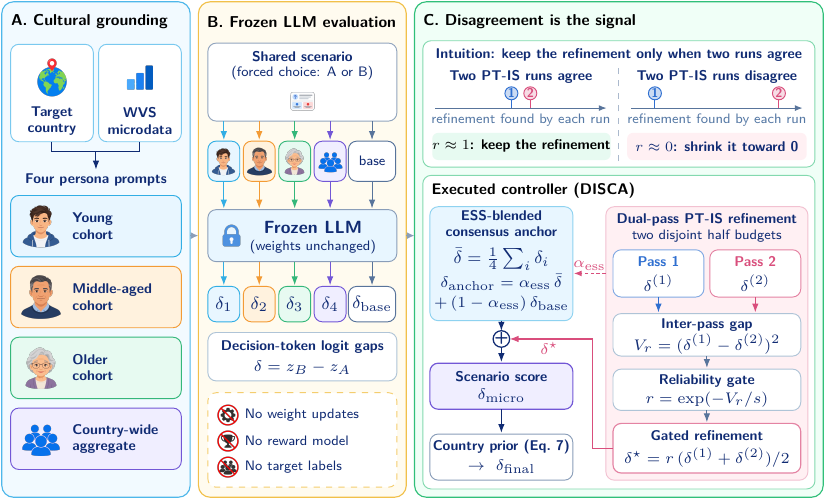}
  \caption{\textbf{DISCA overview.} \textbf{(A)} WVS microdata instantiate four
  persona prompts for country~$c$: three age cohorts and a country-wide
  aggregate. \textbf{(B)} The four personas and an unconditioned base prompt
  run through the frozen model on one scenario, giving the decision-token logit
  gaps $\delta_1,\ldots,\delta_4$ and $\delta_{\mathrm{base}}$.
  \textbf{(C)} The persona-mean gap $\bar\delta$ is blended with
  $\delta_{\mathrm{base}}$ by the effective sample size (ESS) of the importance
  weights, giving the consensus anchor. Prospect-Theory importance sampling
  (PT-IS) runs twice on disjoint half budgets. The squared gap $V_r$ between the
  two passes sets the reliability gate $r=\exp(-V_r/s)$, which scales their
  average into the gated refinement $\delta^\star$. Anchor plus $\delta^\star$
  gives the scenario score $\delta_{\mathrm{micro}}$
  (Eq.~\ref{eq:disca-micro}), and the country prior of
  Eq.~\ref{eq:disca-final-main} turns it into $\delta_{\mathrm{final}}$. Pseudocode in
  App.~\ref{app:disca_overview}.}
  \label{fig:disca_overall}
\end{figure}

\subsection{Setup}
\label{sec:method:setup}

We consider a frozen language model $f_\theta$ on a forced-choice task: each
input $x$ asks the model to choose between two options via decision tokens $A$
and $B$, with logits $z(x) = [z_a, z_b]$ at the decision-token position and
decision gap $\delta(x) = z_b - z_a$. For a target country $c$ we assume
score-level access to $f_\theta$ and a public description of $c$'s population,
namely World Values Survey microdata~\citep{haerpfer2022wvs}. The controller
receives no country-specific human targets at inference, no reward model, and
no write access to $\theta$. Held-out human references are used only for
evaluation. The goal is a per-scenario adjusted gap $\delta_{\mathrm{final}}(x,c)$
that is closer to what country $c$'s population would choose, measured against
a country-level reference vector (\S\ref{sec:experiments}).

\subsection{The disagreement principle}
\label{sec:method:principle}

Evaluate the same scenario under $N$ culturally grounded prompts for the same
target country, each producing a logit gap $\delta_i$ that is a noisy estimate
of what country~$c$'s population-level preference implies here. If the $N$
gaps cluster tightly, their average is a reliable target to move toward. If
they are spread out, the average is a compromise between views the panel itself
does not share, and moving all the way to it is a gamble. The principle is
therefore simple: the more the panel disagrees, the less of the correction
should be applied. We make it precise with the simplest model that captures
it. Let $\delta_h$ denote the unobserved gap consistent with country~$c$'s
population-level preference on that scenario, and model the personas as
\begin{equation*}
\delta_i = \delta_h + \eta_i,\qquad \mathbb{E}[\eta_i] = 0,\qquad \mathrm{Var}(\eta_i) = \tau^2,
\end{equation*}
with the $\eta_i$ independent and identically distributed. This is a
simplification, since all personas share $f_\theta$ and the scenario. Define the
consensus $\bar{\delta} = \frac{1}{N}\sum_i \delta_i$, the within-panel variance
$D^2 = \frac{1}{N-1}\sum_i (\delta_i - \bar{\delta})^2$, the consensus
adjustment $\Delta = \bar{\delta} - \delta_{\text{base}}$, and the oracle
adjustment $\Delta_h = \delta_h - \delta_{\text{base}}$, where
$\delta_{\text{base}}$ is the gap from the unconditioned base prompt on the
same scenario. A scaled correction $\hat\Delta(\gamma) = \gamma\Delta$ with
$\gamma \in [0,1]$ applies a fraction $\gamma$ of the consensus adjustment, and
we ask which fraction minimises
$\mathrm{MSE}(\gamma) = \mathbb{E}\big[(\hat\Delta(\gamma) - \Delta_h)^2\big]$.

\begin{proposition}[Variance-aware shrinkage; proof in App.~\ref{app:bias_variance}]
\label{prop:shrinkage}
Under the model above:
\begin{itemize}
    \item[\textup{(i)}] $D^2$ is unbiased for $\tau^2$, and $\mathrm{Var}(\Delta) = \tau^2 / N$.
    \item[\textup{(ii)}] Among estimators $\hat{\Delta}(\gamma) = \gamma\Delta$ with
    $\gamma \in [0, 1]$, the MSE-minimising one has
    \begin{equation*}
    \gamma^\star = \frac{\Delta_h^2}{\Delta_h^2 + \tau^2 / N},
    \end{equation*}
    which lies in $[0, 1]$ and decreases monotonically in $\tau^2$.
\end{itemize}
\end{proposition}

Part~(ii) is the principle in closed form: the optimal fraction of the
consensus step falls as the panel variance rises and never exceeds one. Part~(i)
says the panel itself estimates $\tau^2$ without bias through $D^2$. Two facts
stop us from using $\gamma^\star$ directly. First, it depends on $\Delta_h$,
which nothing in the panel estimates. Second, with $N{=}4$ personas $D^2$ is
unbiased but noisy on any single scenario: Theorem~\ref{thm:concentration}
(App.~\ref{app:theory:concentration}) bounds its deviation from $\tau^2$, and
with three degrees of freedom the bound is loose. The proposition therefore
fixes the shape a controller should have, bounded, decreasing in disagreement,
and smooth, rather than a coefficient to plug in.

\subsection{The DISCA estimator}
\label{sec:method:estimator}

The controller has four parts. A consensus anchor decides the direction and
most of the distance. A loss-averse refinement searches near the anchor for a
point that no persona strongly objects to. A dual-pass gate keeps that
refinement only when two independent runs of the search agree. A country-level
running prior pools the scores already produced for the same country.

We use $N{=}4$ personas per country: young, middle-aged and older cohorts, plus
a country-wide aggregate. Each is a native-language system prompt built from
cohort- and country-level aggregates of the WVS cultural profile of
\citet{greco2026personas} (App.~\ref{app:personas}). Under the noise model, the
marginal precision gain from one more persona decays as $\Theta(N^{-2})$
(Corollary~\ref{cor:persona-count}), and a sensitivity sweep over
$N\in\{2,\ldots,6\}$ on a three-country Phi-4 panel (USA, Vietnam, Germany)
finds that $N{=}4$ gives the lowest MIS (Table~\ref{tab:r3_persona_count}). Applying
the consensus test of cultural consensus theory~\citep{pothugunta2026cct} to
our panels, every country shows one coherent consensus structure, while the
four personas still split on about one scenario in ten, which is where the gate
acts (App.~\ref{app:cct_panel}).

Every gap is formed on a fixed scale. The two raw decision-token logits are
first divided by a per-attribute logit temperature $T_c$
(Table~\ref{tab:categories}), so $\delta_i = (z_{b,i}-z_{a,i})/T_c$. To remove
order bias, every scenario is evaluated under both the $(A,B)$ and the $(B,A)$
ordering, and $\delta_i = \tfrac{1}{2}(\delta_i^{(AB)} - \delta_i^{(BA)})$ is
the persona's gap; $\delta_{\text{base}}$ is formed the same way. The decision
temperature $T_{\text{dec}}$ enters only once, at the end, in
$p=\sigma(\delta/T_{\text{dec}})$.

\paragraph{Loss-averse importance sampling.}
\label{sec:method:estimator:pt}
The plain consensus adjustment can hide a cohort that strongly disagrees:
three cohorts mildly in favour and one strongly against average out to ``mildly
in favour'', and the model moves in a direction one cohort rejects. We want a
rule that prefers a direction nobody strongly dislikes, and frame it as
bargaining-inspired aggregation among cohorts~\citep{nash1950bargaining,
kalai1975other}, each a loss-averse player. We draw $K$ candidate perturbations
$\epsilon_k \sim \mathcal{N}(0, \sigma^2)$ around the consensus, giving
candidate states $\tilde\delta_k = \bar\delta + \epsilon_k$. The proposal scale
follows the panel,
\begin{equation*}
\sigma=\max\{\mathrm{sd}(\delta_1,\ldots,\delta_N),\,\sigma_0\},\qquad \sigma_0=0.3,
\end{equation*}
so the search is wider when the personas are further apart and never collapses
when they agree. Each candidate is scored by how much closer it lands to each
cohort gap and to the consensus than the base prompt did:
\begin{equation}
g_{i,k} = |\delta_{\text{base}} - \delta_i| - |\tilde\delta_k - \delta_i|, \qquad
g_{\text{cons},k} = |\delta_{\text{base}} - \bar\delta| - |\tilde\delta_k - \bar\delta|.
\label{eq:gains}
\end{equation}
A loss to any cohort should count more than a same-sized gain elsewhere, and
gains beyond what is already aligned should count less. The canonical rule
with both properties is the Kahneman-Tversky value
function~\citep{kahneman1979prospect, tversky1992advances},
\begin{equation}
v(z) =
\begin{cases} z^{\alpha} & z \geq 0, \\ -\kappa\,(-z)^{\alpha} & z < 0, \end{cases}
\label{eq:pt-utility}
\end{equation}
with curvature $\alpha \in (0, 1]$ and loss aversion $\kappa \geq 1$. We blend a
utilitarian verdict over the loss-averse players with a collective verdict
against the consensus direction,
\begin{equation}
U_{\text{total}}(\epsilon_k)
= (1 - \lambda_{\text{coop}})\,\frac{1}{N}\sum_{i=1}^{N} v\!\left(\frac{g_{i,k}}{\sigma}\right)
+ \lambda_{\text{coop}}\, v\!\left(\frac{g_{\text{cons},k}}{\sigma}\right),
\label{eq:util_total}
\end{equation}
with $\lambda_{\text{coop}}=0.7$ (swept in App.~\ref{app:r2_hparam_sensitivity}).
Dividing by $\sigma$ scores gains in units of the perturbation scale. Returning
the highest-utility candidate is unsafe, because an argmax over $K$ noisy scores
picks the candidate that drew favourable noise. We aggregate all candidates
with weights that grow with their utility,
\begin{equation}
\delta_{\text{PT-IS}} = \frac{\sum_{k=1}^{K} w_k\,\epsilon_k}{\sum_{k=1}^{K} w_k},
\qquad
w_k \propto \exp\!\bigl(U_{\text{total}}(\epsilon_k)/\eta\bigr),
\label{eq:pt-is-aggregation}
\end{equation}
where $\eta=0.5$ is a softmax temperature. This is a path-integral-style
exponentially weighted update~\citep{williams2017mppi, levine2018reinforcement}:
the best candidate dominates, but near-winners contribute too. We call it
\emph{loss-averse importance sampling} (PT-IS). The output is a convex
combination of the draws, so $|\delta_{\text{PT-IS}}| \le \max_k|\epsilon_k|$
whatever the backbone does (Theorem~\ref{thm:bounded}). If the importance
weights collapse, so that their normalised effective sample size (ESS) falls below
$\rho_{\mathrm{eff}}=0.1$, the pass returns zero.

\paragraph{Dual-pass reliability gate.}
\label{sec:method:estimator:gate}
Suppose two cohorts prefer a gap of $+0.5$ and two prefer $-0.5$. The utility
surface then has two competing pulls, and with a finite budget each run of
PT-IS drifts toward whichever camp its samples happen to favour. The
refinement then depends on the draw rather than on the panel. We could gate on
$D^2$, but a four-sample variance is too noisy to gate on scenario by scenario.
Instead we measure directly whether the refinement is repeatable. We run PT-IS
twice on the same scenario, on disjoint budgets of $K_{\text{half}}=K/2$
draws, producing $\delta^{(1)}_{\text{PT-IS}}$ and $\delta^{(2)}_{\text{PT-IS}}$,
and compute the squared inter-run gap
\begin{equation*}
V_r = \big(\delta^{(1)}_{\text{PT-IS}} - \delta^{(2)}_{\text{PT-IS}}\big)^2,
\qquad r = \exp(-V_r/s),
\end{equation*}
with bandwidth $s=0.04$. $V_r$ is zero when the runs agree and grows as they
diverge; it is a repeatability statistic computed from the same persona gaps
that define $D^2$, not an estimate of $\tau^2$. The exponential is the simplest
function with the shape the proposition asks for: bounded in $(0,1]$,
decreasing in the disagreement statistic, and smooth. The retained refinement
is the gated average of the two runs,
\begin{equation}
\delta^\star = r \cdot \frac{\delta^{(1)}_{\text{PT-IS}} + \delta^{(2)}_{\text{PT-IS}}}{2}.
\label{eq:disca}
\end{equation}
The two half-runs keep total compute equal to one full run. Whether $r$ in fact
falls as the panel spreads is an empirical question, which
\S\ref{sec:gate_diagnostics} answers.

\paragraph{ESS anchor blend.}
The anchor blends the consensus with the base gap. With
$\overline{\mathrm{ESS}}_m\in(0,1]$ the normalised effective sample size of the
importance weights in pass $m$, a standard diagnostic of whether a weighted
sample has degenerated, let
$\alpha_{\mathrm{ess}}=\min\{1,\max\{\rho_{\mathrm{eff}},\,\min_m\overline{\mathrm{ESS}}_m\}\}$.
The scenario-level score is
\begin{equation}
\delta_{\mathrm{micro}}
  = \alpha_{\mathrm{ess}}\,\bar{\delta}
  + (1-\alpha_{\mathrm{ess}})\,\delta_{\mathrm{base}}
  + \delta^\star .
\label{eq:disca-micro}
\end{equation}
When the importance weights collapse onto a few candidates,
$\alpha_{\mathrm{ess}}$ falls and the anchor leans toward the base gap. When
the two passes disagree and $r$ approaches zero, the refinement vanishes and
the score falls back to its anchor, not to $\delta_{\mathrm{base}}$.

\paragraph{Country-level running prior.}
\label{sec:method:prior}
The last stage pools information across the scenarios of one country. The
controller keeps a running country score $d_c$, initialised to zero, and emits
\begin{equation}
\delta_{\mathrm{final}}
  = a_t\, d_c + (1-a_t)\,\delta_{\mathrm{micro}},\qquad
  a_t=\begin{cases}0 & t<t_0,\\ 1-e^{-(t-t_0)/\tau} & t\ge t_0,\end{cases}
\label{eq:disca-final-main}
\end{equation}
and then updates the exponential moving average
$d_c \leftarrow (1-\beta)\,d_c + \beta\,\delta_{\mathrm{micro}}$ with
$\beta=0.1$. Here $t$ counts the scenarios already processed for $c$,
$t_0=50$ is a warm-up and $\tau=100$ sets how fast the prior takes over. The
prior uses only the controller's own earlier outputs; no human target and no
label enters it. It is, however, transductive: each decision depends on the
earlier test scenarios of the same country and on their order, and the
running score pools all six moral attributes. On the 310-scenario panel the
weight $a_t$ reaches $0.93$ by the last scenario. We report what this stage
contributes separately (\S\ref{sec:ablation}), because it carries a large share
of the gain.

\paragraph{Beyond two decision tokens.}
\label{sec:method:openended}
The controller needs one score per answer option under each view, not
specifically two tokens. With $K$ options, the base and persona scores are
centred $K$-vectors, the persona mean is the anchor, proposals move along the
directions in which the personas disagree, and the same two passes and gate run
in that space with $\ell_1$ gains and a squared-$\ell_2$ inter-pass gap
(Eq.~\ref{eq:kway-disca}, App.~\ref{app:normad_verified}). For free-form
generation the options are candidate passages, scored by their mean token
log-probability under each view (App.~\ref{app:makieval}). The $K$-way variant
does not use the country prior.

\section{Experiments}
\label{sec:experiments}
\label{sec:results}

We test DISCA on MultiTP~\citep{jin2025multitp}, the multilingual extension of
the Moral Machine trolley problem~\citep{awad2018moral}. Each scenario shows
two groups an autonomous vehicle could spare and forces a single decision
token, $A$ or $B$. It isolates one of six moral attributes (species, gender,
age, fitness, social value, number of lives) and matches the groups on
everything else, so the rate at which the model spares one side is its
preference on that attribute. The sparing probability under decoding
temperature $T_{\text{dec}}$ is $p_{\text{spare}}(x)=\sigma(\delta(x)/T_{\text{dec}})$,
and averaging it over the scenarios for one attribute and one country gives
that country's Average Marginal Component Effect (AMCE),
\begin{equation}
\hat{m}^{(c)}_d = \frac{1}{|\mathcal{S}^{(c)}_d|}
\sum_{x \in \mathcal{S}^{(c)}_d} p_{\text{spare}}(x).
\label{eq:amce-estimate}
\end{equation}
Stacking the six attributes gives a vector $\hat{m}^{(c)}$, which we compare to
the human AMCE vector $h^{(c)}$ aggregated from real Moral Machine
judgments~\citep{awad2018moral, jin2025multitp}. The human vector is a held-out
evaluation reference; it is never supplied to the controller. Our headline
metric is the \emph{misalignment score}
$\mathrm{MIS}(c) = \|\hat{m}^{(c)} - h^{(c)}\|_2$, the distance between the two
vectors with every component on the $[0,1]$ scale; lower is closer. We
average it with equal weight over twenty countries on four continents
(310 scenarios each), each scenario presented in the country's own language
with only $A$ and $B$ kept in English so that token extraction is consistent
across model families. A lower MIS means closer to what people in a country
reported, not that the answer is more appropriate for them. Preprocessing and
token elicitation are in Apps.~\ref{app:dataset} and~\ref{app:amce}.

We report seven open-weight backbones from five model families, 2B to 70B,
chosen to span families, sizes and conditioning regimes and to include failure
cases; the suite is a selected panel, not a random sample of models.
Llama-3.3-70B and Qwen3-VL-8B ran from 4-bit (bitsandbytes) checkpoints and
Qwen2.5-7B in BF16; for the other four the precision was not logged, and the
loader's default was 4-bit. Comparisons across sizes are
therefore also comparisons across precision. Every headline run uses
controller seed 42. Controller hyperparameters were set during prototyping on three backbones and five countries (BRA, CHN, DEU, JPN, USA) and then held fixed for every backbone and country in the headline sweep; the persona count $N{=}4$ follows the cohort design and is supported by a three-country sweep (App.~\ref{app:hyperparams}). We compare DISCA with
five training-free controls on the same 20-country Phi-4 grid
(Table~\ref{tab:main_baseline_sanity}; App.~\ref{app:baselines}), none of which
receives the human targets or updates weights.

\subsection{DISCA improves every backbone}
\label{sec:percountry_main}
\label{sec:robustness}

\begin{table*}[t]
\centering
\begin{minipage}[t]{0.40\textwidth}
\caption{\textbf{Training-free controls on the same Phi-4 grid} (20 countries,
310 scenarios each). MIS is the distance from the human AMCE vector; lower is
better. ``Reduction'' is $(\mathrm{MIS}_{\text{row}}-\mathrm{MIS}_{\text{DISCA}})/\mathrm{MIS}_{\text{row}}$.
No method receives the human target or updates weights.}
\label{tab:main_baseline_sanity}
\centering\scriptsize
\setlength{\tabcolsep}{3pt}
\resizebox{\linewidth}{!}{%
\begin{tabular}{lcc}
\toprule
Method & MIS $\downarrow$ & Reduction \\
\midrule
\textbf{DISCA (ours)} & \textbf{.346} & \\
Persona-ensemble prompt & .364 & \gain{5.0\%} \\
MC-Dropout~\citep{gal2016dropout} & .403 & \gain{14.1\%} \\
Activation steering~\citep{arditi2024refusal} & .430 & \gain{19.5\%} \\
Fixed logit offset & .439 & \gain{21.2\%} \\
WVS profile prompt & .452 & \gain{23.5\%} \\
Vanilla decoding & .454 & \gain{23.7\%} \\
\bottomrule
\end{tabular}}
\end{minipage}
\hfill
\begin{minipage}[t]{0.57\textwidth}
\caption{\textbf{Macro results on the 20-country MultiTP panel}, equal country
weight, controller seed 42. MIS: lower is better. ``Reduction'' is the relative
MIS reduction from vanilla. ``Win/20'' counts countries where DISCA has the
lower MIS. $r$ is the Pearson correlation between the DISCA and the human AMCE
vector, averaged over countries.}
\label{tab:main_macro_summary}
\centering\scriptsize
\setlength{\tabcolsep}{3pt}
\begin{tabular}{@{}lccccc@{}}
\toprule
Model & Vanilla & DISCA & Reduction & Win/20 & $r$ \\
\midrule
Llama-3.3-70B & .849 & .668 & \gain{21.4\%} & \textbf{20} & $-$0.01 \\
Magistral-Small-2509 (24B) & .403 & .350 & \gain{13.1\%} & 16 & 0.62 \\
Phi-4 (14B) & .454 & \textbf{.346} & \textbf{\gain{23.7\%}} & 18 & 0.56 \\
Qwen3-VL-8B & .575 & .466 & \gain{18.9\%} & \textbf{20} & 0.22 \\
Qwen2.5-7B & .453 & .362 & \gain{20.1\%} & 17 & 0.45 \\
Phi-3.5-mini (3.8B) & .638 & .571 & \gain{10.4\%} & 17 & $-$0.35 \\
Gemma-4-E2B (2B) & .495 & .478 & \gain{3.4\%} & 14 & 0.46 \\
\bottomrule
\end{tabular}
\end{minipage}
\end{table*}

DISCA lowers macro MIS on all seven backbones
(Table~\ref{tab:main_macro_summary}): by 10 to 24 percent on six of them, and by
3 percent on the 2B Gemma model. It improves between 14 and 20 of the 20
countries on every backbone. The gain does not depend on where the backbone
starts. Phi-4 has the largest relative gain and the lowest final MIS, while
Llama-3.3-70B starts furthest from the human vectors, improves by 21 percent,
and still ends furthest away. Two backbones, Llama-3.3-70B and Phi-3.5-mini,
keep a negative correlation with the human vector even as their MIS falls. MIS
measures the size of the residual and $r$ its shape, so a correction can shrink
every error without recovering the ordering of the six attributes. The
per-dimension audit (App.~\ref{app:perdim}) traces this mainly to the
utilitarian attribute. On Llama-3.3-70B, DISCA narrows the utilitarian error
but leaves every country sparing the larger group less often than people do,
and replacing that single component with the human value turns all eight
negative-$r$ countries positive. On Phi-3.5-mini the ordering failure is
broader, and five countries stay negative after the same replacement.

DISCA also beats all five training-free controls on the same Phi-4 grid
(Table~\ref{tab:main_baseline_sanity}). Describing the country's aggregate
values to the model in one prompt does nothing: the WVS profile prompt lands
within rounding of vanilla decoding. Answering as six sociodemographic personas
(a persona-ensemble prompt) closes most of the gap, and DISCA still ends lower. A white-box control, activation steering,
and two calibration-style controls, MC-Dropout and a fixed per-country logit
offset, all improve on vanilla but stay well behind DISCA. A 14B backbone with
DISCA (Phi-4) also ends closer to the human vectors than the 70B backbone with
or without it. This is mainly a difference between the two base models, since
vanilla Phi-4 already starts much closer than the 4-bit Llama-3.3-70B, so it
does not show that the controller closes a gap in scale.

The gains are geographically broad and the regressions stay visible
(Table~\ref{tab:percountry_p1}, App.~\ref{app:percountry_full}). On the two
backbones with complete three-seed scenario logs, Qwen2.5-7B and Qwen2.5-14B,
a one-sided paired signed-rank test over the 20 countries favours DISCA on 18
of 20 for both (Wilcoxon $p<10^{-4}$), and a country-then-scenario bootstrap
interval for the mean MIS reduction excludes zero
(App.~\ref{app:gate_diagnostics}). This test uses a separately decoded vanilla
run (macro MIS .4488 and .4682). Against the vanilla gap decoded together with
the personas (.4231 and .4032), DISCA still wins 15 of 20 countries on 7B but
only 10 of 20 on 14B, where the paired test is not significant
(App.~\ref{app:gate_diagnostics}).

\subsection{Which parts of DISCA carry the gain}
\label{sec:discussion}
\label{sec:ablation}

The persona panel and the country-level prior carry most of DISCA's gain.
Every stage of the controller is logged per scenario, so each stage can be
switched off after the fact on the headline runs without a new model call, and
the replay reproduces every published country score to within $10^{-15}$
(App.~\ref{app:decomposition}, Table~\ref{tab:decomposition}). The plain
persona consensus, which never receives the human targets, already lowers
macro MIS on five of the seven backbones. Adding the country prior to that
consensus gives a score at least as low as Full DISCA on six backbones.
Without the prior, DISCA still improves on vanilla for five backbones, but by
less on each of them except Magistral, where the prior hurts. The panel therefore supplies the
direction, and the prior turns many noisy scenario decisions into a stable
country-level estimate. Because MIS rewards less extreme decisions, part of the
gain is also reachable by an oracle temperature chosen with the targets
(App.~\ref{app:decomposition}), and the tail counts over individual
country-model cells are in App.~\ref{sec:tail_safety}.

\begin{table}[t]
\centering\small
\caption{\textbf{Same-panel aggregation rules on the same logits.} Three-seed
equal-country macro MIS on the 20-country MultiTP panel (lower is better). Every
rule reuses the logged vanilla and four-persona logits of the two fully logged
Qwen2.5 grids, so all columns have exactly the same model calls. These grids
are a separate evidence family from the headline runs of
Table~\ref{tab:main_macro_summary}.}
\label{tab:same_logit_verified}
\setlength{\tabcolsep}{4pt}
\begin{tabular}{lrrrr}
\toprule
Backbone & Full DISCA & Prob. average & Logit consensus & Majority vote \\
\midrule
Qwen2.5-7B & \textbf{.3563} & .3872 & .3877 & .4307 \\
Qwen2.5-14B & \textbf{.3950} & .4342 & .4393 & .4613 \\
\bottomrule
\end{tabular}
\end{table}

\paragraph{DISCA beats standard panel aggregation on the same logits.}
\label{sec:same_logit_main}
Given exactly the same model calls, DISCA does better than the rules an
ensemble would normally use to combine its members. On the two Qwen2.5 backbones with complete three-seed logs, we recompute
probability averaging, logit consensus and majority vote from exactly the
vanilla and four-persona logits that DISCA used. Full DISCA has the lowest MIS
on both backbones (Table~\ref{tab:same_logit_verified}). These rules receive no
country prior. Given the same prior, logit consensus (.3530 and .3903) and an
inverse-variance aggregate of the personas (.3500 and .3839) are slightly
better than Full DISCA (App.~\ref{app:aggregation_controls}). The advantage in
the table therefore comes from the prior, not from the aggregation rule.

\paragraph{The gate responds to panel disagreement.}
\label{sec:gate_diagnostics}
On the two fully logged grids (Qwen2.5-7B and Qwen2.5-14B, three seeds,
$18{,}600$ scenarios each), the reliability weight $r$ falls as the persona
variance $D^2$ rises (Spearman $\rho=-0.16$ and $-0.45$), and the association
holds with countries as clusters (country-level $\rho=-0.86$ and $-0.99$;
Table~\ref{tab:gate_diag_corr}). The gate therefore works as designed. What it
controls is small, a median of $1.0\%$ and $0.5\%$ of the persona-mean gap $\bar\delta$.
Removing the gate alone raises macro MIS slightly on all seven backbones.
Removing the refinement altogether lowers macro MIS slightly on all seven, by
at most $0.0042$ (Table~\ref{tab:decomposition}). The gated refinement
therefore does not improve mean MIS. Matched on the prior, the controller's smaller
tail losses than consensus come from the ESS anchor blend, not from the
refinement (App.~\ref{sec:tail_safety}). Disagreement also does not locate base-model
error (App.~\ref{app:gate_diagnostics}).

\paragraph{The panel makes modest use of the country it is given.}
\label{sec:prompt_surface_main}
Giving the controller a random or the most distant country's profile worsens
macro MIS on both Qwen2.5 backbones, by .001 to .003 at 7B and by about .02 at
14B, where only the most distant profile has an interval that excludes zero
(App.~\ref{app:profile_specificity}). The panel therefore makes modest,
backbone-dependent use of the country's identity. Instruction wording moves the score more than descriptor order
does, so prompt form is a real sensitivity (App.~\ref{app:prompt_surface}).

\subsection{Beyond binary MultiTP}
\label{sec:openended_main}
\label{sec:beyond_multitp_main}

The persona panel carries over to tasks with other output types and
populations. Table~\ref{tab:generalisation} collects three of them, each a
separate protocol. Every block also reports the plain persona mean ensemble on
the same persona calls, which separates what the panel supplies from what the
controller adds.

\begin{table}[t]
\centering\small
\caption{\textbf{Results beyond binary MultiTP.} Each block is a separate
protocol and is not pooled with the others. \emph{Mean ens.} is the persona
mean ensemble on the same persona calls (probability averaging on
GlobalOpinionQA). NormAd: country condition, DISCA averaged over three controller seeds, accuracy and
macro-F1 ($\uparrow$). GlobalOpinionQA: restricted protocol, three seeds,
Jensen-Shannon distance ($\downarrow$). MAKIEval: country-entity counts in
the country-unnamed condition, exact paired sign test against vanilla.
Protocols in App.~\ref{app:out_of_family}.}
\label{tab:generalisation}
\label{tab:beyond_multitp}
\setlength{\tabcolsep}{4pt}\scriptsize
\resizebox{\linewidth}{!}{%
\begin{tabular}{@{}llrrrl@{}}
\toprule
Benchmark and backbone & Metric & Vanilla & Mean ens. & DISCA & DISCA vs.\ vanilla \\
\midrule
\multicolumn{6}{@{}l}{\emph{Forced-choice and distribution tasks}} \\
NormAd, 622 items / 17 countries, Qwen2.5-7B & Acc. $\uparrow$ & .5543 & .5756 & \textbf{.5762} & \gain{+.0219} \\
 & Macro-F1 $\uparrow$ & .4613 & .4849 & \textbf{.4859} & \gain{+.0247} \\
NormAd, 622 items / 17 countries, Qwen2.5-14B & Acc. $\uparrow$ & .5299 & .5497 & \textbf{.5541} & \gain{+.0242} \\
 & Macro-F1 $\uparrow$ & .5014 & .5038 & \textbf{.5088} & \gain{+.0074} \\
GlobalOpinionQA, 969 questions / 15 countries, Qwen2.5-7B & JSD $\downarrow$ & .3666 & .3585 & \textbf{.3192} & \gain{$-$.0474} \\
\midrule
\multicolumn{6}{@{}l}{\emph{MAKIEval free-form generation}: 16 countries, 6 topics, country-unnamed condition} \\
Qwen2.5-7B, 3 seeds, 288 prompts & Entities & 8 & \textbf{49} & 47 & \gain{+39}; $p = 1.5\times10^{-5}$ \\
Qwen2.5-14B, 1 seed, 96 prompts & Entities & 4 & \textbf{23} & 21 & \gain{+17}; $p = 7.3\times10^{-4}$ \\
Phi-4 (14B), 1 seed, 96 prompts & Entities & 4 & \textbf{22} & 18 & \gain{+14}; $p = 4.2\times10^{-3}$ \\
Qwen3-8B, 1 seed, 96 prompts & Entities & 0 & \textbf{7} & 5 & $+5$; $p = .125$, n.s. \\
\bottomrule
\end{tabular}}
\end{table}

\paragraph{Forced-choice and distribution tasks.} DISCA improves on vanilla
decoding in both forced-choice and distribution-matching transfer. On NormAd, a
three-way social-norm benchmark, it raises accuracy and macro-F1 over vanilla
in the country-context condition for both Qwen backbones and is the best
method in that condition, with the persona mean ensemble close behind
(App.~\ref{app:normad_verified}). As on MultiTP, the panel supplies most of the
gain. When the prompt already states the relevant value or rule of thumb,
vanilla decoding is the stronger choice (App.~\ref{app:normad_verified}). On
GlobalOpinionQA~\citep{durmus2023global}, under a restricted protocol of 969
questions and 15 countries (App.~\ref{app:goqa_tiera}), DISCA lowers the
distance to each country's opinion distribution relative to both vanilla and
probability averaging. This protocol runs the binary controller with the
country prior, which probability averaging does not receive. On a second
protocol, 560 multi-option Pew Global Attitudes questions read with the
$K$-way controller and no prior, DISCA is still closer than vanilla on both
Qwen2.5 backbones but probability averaging is closer still
(App.~\ref{app:goqa_tiera}).

\paragraph{MAKIEval: supplying missing country context.}
\label{sec:makieval_main}
The panel also supplies country context that a free-form prompt is missing.
MAKIEval tests a setting where the model writes a short passage. In its
country-unnamed condition the prompt does not say which country it is about, so
the persona panel is the only country-specific signal. Every arm selects from
one shared pool of five candidate passages (App.~\ref{app:makieval}). Using the
benchmark's released lexicon of country-typical entities, we count how many
entities typical of the target country and topic appear in the selected
passages. Selecting with the persona panel raises that count on every
backbone; the persona mean ensemble raises it at least as much as DISCA, and
no difference between the two is significant. The panel therefore supplies a
country signal the prompt is missing. When the prompt already names the
country, the count moves the other way. A blinded LLM judge prefers the DISCA
passage over vanilla for cultural fit on 117 of 124 decided country-unnamed
pairs from Qwen2.5-7B.

\paragraph{Where transfer stops.} Transfer has a boundary. On MENAValues and
CDEval, a prompt that names the country or a plain persona average does better
than DISCA (App.~\ref{app:mena_cdeval}). Their targets lie further from what
the World Values Survey measures, and half of the MENAValues countries have no
microdata of their own. When the target is a full answer distribution over
several options and DISCA runs without the prior, plain probability averaging
matches it better (App.~\ref{app:goqa_tiera}). Such a target rewards
averaging, while DISCA is built to return one calibrated decision.

\section{Limitations and Ethical Considerations}
\label{sec:limitations}

Each limit below says what the evidence teaches and the work it points to.
App.~\ref{app:limitations} lists further limits of the evaluation itself:
survey coverage, controller seeds and weight precision.

\paragraph{Personas are statistical aggregates, not people.} Each persona
summarises World Values Survey statistics for an age cohort or for the whole
country. It describes no individual, and presenting it as one would flatten
identity and risk stereotyping. The panel works because it measures how much a
country's own survey groups disagree, not because any persona is a faithful
person. Testing persona directions directly against individual-level survey
responses would show how far each persona can be trusted on its own.

\paragraph{Country is a coarse proxy for culture.} National aggregation hides
sub-national, diasporic, indigenous, migrant, political and religious
plurality, and survey coverage is fixed in time and uneven across countries.
The same panel construction applies to any group with survey microdata, so
panels built from regional or community surveys are a direct next step.

\paragraph{Alignment is not legitimacy.} MIS, accuracy and Jensen-Shannon
distance measure agreement with a descriptive target~\citep{awad2018moral,
jin2025multitp}. They do not show that a majority preference ought to govern a
decision. DISCA makes a model better at reporting what a population said.
Deciding when that report should shape an action needs normative review and
the people affected, which future deployments should build in.

\paragraph{Minority risk remains.} Loss aversion penalises a correction that
harms one persona, but it does not forbid it. On the two Qwen2.5 grids, at
least one of the four personas has a negative minimum internal utility on
51.4\% of Qwen2.5-7B scenarios and 53.3\% of Qwen2.5-14B scenarios
(App.~\ref{app:worst_group_boundary}). These are diagnostics over synthetic
prompts, not demographic harm rates, but they show that the controller does
not protect every persona. An explicit worst-persona constraint is the natural
extension, because an inference-time controller deployed naively can encode a
harmful majority view.

\paragraph{No consequential individual use.} DISCA should not drive decisions
about individuals in healthcare, education, hiring, moderation sanctions,
religion or family matters, and it must never infer a person's values from
their nationality. This work uses public survey summaries and released
benchmark data, and it collects no new human-subject data. A system using
DISCA should prefer a user's own stated preferences when they are available,
send high-impact decisions to domain experts, and treat strong panel
disagreement as a reason to abstain or escalate.

\paragraph{The country prior is transductive.} The prior pools the
controller's earlier decisions for the same country, so each decision depends
on the queries that came before it and on their order. Its large contribution
(\S\ref{sec:ablation}) says that a stable country-level estimate is worth
having, and a prior estimated once from a held-out set of country prompts
would keep that benefit for one-query-at-a-time deployment. Because MIS
rewards less extreme decisions, a calibration control chosen without the
targets should also sit next to DISCA in future comparisons
(App.~\ref{app:decomposition}).

\paragraph{DISCA needs option scores.} The controller reads a score for each
answer option, the two decision-token logits on MultiTP. Text-only interfaces
do not always expose these scores. Estimating them from repeated sampled
answers would extend the method to such interfaces, at a higher call cost.

\section{Conclusion}
\label{sec:conclusion}

Cultural adaptation of a frozen model need not be a training problem. A panel
of four survey-grounded personas, read at decoding time, lowers misalignment
with country-level human preferences on all seven backbones of the 20-country
MultiTP evaluation, by 10 to 24 percent on six of them. With DISCA, a 14B
backbone ends closer to the human vectors than a vanilla 70B backbone. DISCA
beats every training-free control on a shared grid and, on the same logits,
the standard ways of aggregating a panel when they receive no country prior.
Taking the method apart on its own
runs shows that the persona panel and the country-level prior carry most of
the gain, while the disagreement gate behaves as designed on a small
refinement. The panel signal carries over to NormAd, GlobalOpinionQA and
MAKIEval, and it reaches its boundary on benchmarks whose targets the World
Values Survey does not measure well. The decomposition also tells future work
which parts of such a controller to keep: a survey panel and a stable
country-level estimate, built from richer surveys for the countries the panel
does not yet reach.

\section*{Acknowledgements and Disclosure of Funding}
We thank the anonymous reviewers and the area chair for their constructive feedback, which improved this paper. Experiments were run on the free GPU quota provided by Kaggle and on a single H100 GPU. This work received no other third-party funding, and the authors declare no competing interests.

\clearpage
{\sloppy
\emergencystretch=4em
\hbadness=10000
\vbadness=10000
\bibliographystyle{abbrvnat}
\bibliography{references}
}

\newpage
\appendix
\setcounter{section}{0}
\renewcommand{\thesection}{A\arabic{section}}

\startcontents[appendix]
\printcontents[appendix]{}{0}{\section*{Appendix Contents}\setcounter{tocdepth}{2}}
\vspace{0.5em}
\hrule
\vspace{1em}
\newpage

\section{MultiTP Dimensions and DISCA Inference Pseudocode}
\label{app:disca_overview}

The pipeline schematic is Figure~\ref{fig:disca_overall} in \S\ref{sec:method}. Table~\ref{tab:categories} lists the six MultiTP moral dimensions and the per-attribute logit temperatures $T_c$ used by DISCA (validated in App.~\ref{app:hyperparams}). Algorithm~\ref{alg:disca} states the exact inference routine, including where each temperature acts. $T_c$ divides the two raw decision-token logits before any gap is formed, so every quantity downstream, including the proposal scale $\sigma$ and the utility of Eq.~\ref{eq:util_total}, lives on the $T_c$-scaled scale. The value $3.0$ listed as the default logit divisor in Table~\ref{tab:hyperparams} is the $T_c$ used for a category without an entry in Table~\ref{tab:categories}; no MultiTP attribute falls back to it. $T_{\text{dec}}$ is applied once, to the final score. Scenarios are processed in a fixed shuffled order (seed 42) and the country state is reset for every country.

\begin{algorithm}[H]
\caption{DISCA Inference: Scenario $\mathbf{x}$ in Country $c$}
\label{alg:disca}
\begin{algorithmic}[1]
\Require Frozen LLM $f_\theta$; personas $\{\mathbf{s}_i\}_{i=1}^{N}$; attribute $c(\mathbf{x})$;
         hyperparameters $K_{\text{half}},\sigma_0,\lambda_{\text{coop}},\alpha,\kappa,\eta,\rho_{\mathrm{eff}},s,T_{\text{dec}},\{T_c\},\beta,t_0,\tau$;
         country state $(d_c, t)$, initialised to $(0,0)$ once per country
\Ensure $p_{\text{spare}} \in [0,1]$
\State $T \leftarrow T_{c(\mathbf{x})}$ \Comment{per-attribute logit temperature, Table~\ref{tab:categories}}
\For{$\text{order} \in \{\text{orig}, \text{swap}\}$}
  \State Batched forward: $f_\theta$ on base + $N$ personas; read raw logits $(z_a,z_b)$
  \State $\delta_i^{(\text{ord})} \leftarrow (z_{b,i}-z_{a,i})/T$ for the base and every persona
\EndFor
\State $\delta_{\text{base}} \leftarrow (\delta_{\text{base}}^{(\text{orig})} - \delta_{\text{base}}^{(\text{swap})})/2$
\State $\delta_i \leftarrow (\delta_i^{(\text{orig})} - \delta_i^{(\text{swap})})/2$
\State Consensus $\bar{\delta} \leftarrow \tfrac{1}{N}\sum_i \delta_i$; proposal scale $\sigma \leftarrow \max\{\mathrm{sd}(\delta_1,\ldots,\delta_N),\sigma_0\}$
\For{$m \in \{1,2\}$}
  \For{$k = 1$ \textbf{to} $K_{\text{half}}$}
    \State Sample $\epsilon_k \sim \mathcal{N}(0, \sigma^2)$; compute $U_{\text{total}}(\epsilon_k)$ via Eqs.~\ref{eq:gains},~\ref{eq:util_total}
  \EndFor
  \State $w^{(m)} \leftarrow \operatorname{softmax}(U_{\text{total}}/\eta)$; $\overline{\mathrm{ESS}}_m \leftarrow \big(K_{\text{half}}\sum_k (w_k^{(m)})^2\big)^{-1}$
  \State $\delta^{\star}_m \leftarrow \sum_k w_k^{(m)} \epsilon_k$ if $\overline{\mathrm{ESS}}_m \geq \rho_{\mathrm{eff}}$ else $0$
\EndFor
\State $r \leftarrow \exp(-(\delta^{\star}_1-\delta^{\star}_2)^2/s)$; \quad $\delta^{\star} \leftarrow r\,(\delta^{\star}_1+\delta^{\star}_2)/2$
\State $\alpha_{\text{ess}} \leftarrow \min\{1,\max\{\rho_{\mathrm{eff}},\min_m\overline{\mathrm{ESS}}_m\}\}$
\State $\delta_{\text{micro}} \leftarrow \alpha_{\text{ess}}\,\bar{\delta} + (1-\alpha_{\text{ess}})\,\delta_{\text{base}} + \delta^{\star}$
\State $a_t \leftarrow 0$ if $t<t_0$ else $1-e^{-(t-t_0)/\tau}$; \quad $\delta_{\text{final}} \leftarrow a_t\,d_c + (1-a_t)\,\delta_{\text{micro}}$
\State $d_c \leftarrow (1-\beta)\,d_c + \beta\,\delta_{\text{micro}}$; \quad $t \leftarrow t+1$ \Comment{country prior, \S\ref{sec:method:prior}}
\State \Return $p_{\text{spare}} = \sigma(\delta_{\text{final}} / T_{\text{dec}})$, oriented to the preferred side
\end{algorithmic}
\end{algorithm}

\begin{table}[H]
\caption{The six moral dimensions in MultiTP. $T_c$ is the per-attribute logit temperature used by our method; values reflect empirical logit magnitudes and are validated in App.~\ref{app:hyperparams}.}
\label{tab:categories}
\centering\scriptsize
\setlength{\tabcolsep}{4pt}
\begin{tabular}{lllc}
\toprule
Dimension & Preferred & Contra & $T_c$ \\
\midrule
Species & Human & Animal & 4.0 \\
Gender & Female & Male & 3.5 \\
Age & Young & Elderly & 1.5 \\
Fitness & Fit & Unfit & 1.5 \\
Social Value & High (exec.) & Low (homeless) & 1.5 \\
Utilitarianism & More lives & Fewer lives & 1.5 \\
\bottomrule
\end{tabular}
\end{table}

\section{Proof of Proposition~\ref{prop:shrinkage}}
\label{app:bias_variance}

Recall the agent model from~ \S\ref{sec:method:principle}: for a single 
scenario,
\begin{equation*}
\delta_i = \delta_h + \eta_i, \qquad i = 1, \ldots, N,
\end{equation*}
where $\delta_h$ is a fixed (but unknown) population-preference 
logit gap, and the noise terms $\eta_i$ are i.i.d.\ with 
$\mathbb{E}[\eta_i] = 0$ and $\mathrm{Var}(\eta_i) = \tau^2$. We 
also fix $\delta_{\text{base}}$, the gap from the base prompt 
(deterministic given the scenario), and define
\begin{equation*}
\bar\delta = \frac{1}{N}\sum_{i=1}^{N} \delta_i, \quad
D^2 = \frac{1}{N-1}\sum_{i=1}^{N}(\delta_i - \bar\delta)^2, \quad
\Delta = \bar\delta - \delta_{\text{base}}, \quad
\Delta_h = \delta_h - \delta_{\text{base}}.
\end{equation*}
Throughout, expectations and variances are taken over the noise 
$\eta_1, \ldots, \eta_N$.

\subsection*{Part (i): $\mathbb{E}[D^2] = \tau^2$ and
$\mathrm{Var}(\Delta) = \tau^2/N$}

Since $\delta_i - \bar\delta = (\delta_h + \eta_i) - (\delta_h + \bar\eta) = \eta_i - \bar\eta$, the population term $\delta_h$ contributes equally to every $\delta_i$ and cancels in the differences, so
\begin{equation}
D^2 = \frac{1}{N-1}\sum_{i=1}^{N} (\eta_i - \bar\eta)^2,
\qquad \bar\eta := \tfrac{1}{N}\textstyle\sum_j \eta_j .
\label{eq:D2-noise}
\end{equation}
This is the textbook sample variance of $N$ i.i.d.\ mean-zero variables of variance $\tau^2$, whose unbiasedness is standard: expanding the square gives $\sum_i \eta_i^2 - N\bar\eta^2$, and taking expectations with $\mathbb{E}[\eta_i^2] = \tau^2$ and $\mathbb{E}[\bar\eta^2] = \mathrm{Var}(\bar\eta) = \tau^2/N$ leaves $(N-1)\tau^2$, so
\begin{equation}
\mathbb{E}[D^2] = \tau^2 .
\label{eq:D2-unbiased}
\end{equation}
Because $\delta_{\text{base}}$ is deterministic, $\mathrm{Var}(\Delta) = \mathrm{Var}(\delta_h + \bar\eta) = \mathrm{Var}(\bar\eta) = \tau^2/N$. The two quantities play distinct roles: $D^2$ estimates $\tau^2$, the variance of \emph{one} agent's noise, and is what we can compute from data; $\mathrm{Var}(\Delta) = \tau^2/N$ is the variance of the \emph{mean} of $N$ agents' noise, and is what enters $\gamma^\star$ below.

\subsection*{Part (ii): MSE-minimising shrinkage}

From Part (i), $\mathbb{E}[\Delta] = \Delta_h$. Writing the error as $\gamma\Delta - \Delta_h = \gamma(\Delta - \Delta_h) + (\gamma-1)\Delta_h$, the cross term in the squared expansion vanishes because $\mathbb{E}[\Delta - \Delta_h] = 0$, leaving the standard bias-variance form
\begin{equation*}
\mathrm{MSE}(\gamma) \;=\; \gamma^2\,\frac{\tau^2}{N} \;+\; (\gamma-1)^2\Delta_h^2 .
\end{equation*}
Differentiating, $\tfrac{d}{d\gamma}\mathrm{MSE}(\gamma) = 2[\gamma(\tau^2/N + \Delta_h^2) - \Delta_h^2]$, which vanishes at
\begin{equation}
\gamma^\star = \frac{\Delta_h^2}{\Delta_h^2 + \tau^2/N},
\label{eq:gamma-star-proof}
\end{equation}
and the second derivative $2(\Delta_h^2 + \tau^2/N) > 0$ makes this the unique minimum. Four readings of Eq.~\ref{eq:gamma-star-proof} carry the design:
\begin{itemize}
\item The term $\gamma^2\tau^2/N$ is the \emph{variance} of $\gamma\Delta$: it grows with how confidently we apply the consensus, scaled by the consensus's own sampling variance. The term $(\gamma-1)^2\Delta_h^2$ is the \emph{squared bias} from shrinking: zero at $\gamma=1$, quadratic as $\gamma \to 0$. Larger $\gamma$ buys bias with variance, and $\gamma^\star$ is where the trade balances.
\item $\gamma^\star \in [0,1]$ automatically, since the numerator is non-negative and the denominator at least as large.
\item Writing $u := \tau^2/N$, $d\gamma^\star/du = -\Delta_h^2/(\Delta_h^2+u)^2 \leq 0$, so $\gamma^\star$ is monotone-decreasing in the consensus variance.
\item The two boundary regimes are the method: as $\tau^2 \to 0$ (agents agree) $\gamma^\star \to 1$ and the consensus applies in full; as $\tau^2 \to \infty$ or $\Delta_h \to 0$ (signal-to-noise collapses) $\gamma^\star \to 0$ and we do not adjust at all.
\end{itemize}
\subsection*{Corollary~\ref*{cor:persona-count}: Marginal Value of Additional Personas}

\begin{corollary}[Marginal value of additional personas]
\label{cor:persona-count}
Under the agent model of Proposition~\ref{prop:shrinkage}, the marginal reduction in consensus variance from adding the $(N{+}1)$-th persona is
\begin{equation}
\Delta\mathrm{Var}(N \to N{+}1) \;=\; \frac{\tau^2}{N} - \frac{\tau^2}{N+1} \;=\; \frac{\tau^2}{N(N+1)} \;=\; \Theta(N^{-2}).
\label{eq:marginal-persona}
\end{equation}
Equivalently, the optimal shrinkage coefficient $\gamma^\star(N) = \Delta_h^2/(\Delta_h^2 + \tau^2/N)$ from Proposition~\ref{prop:shrinkage}(ii) satisfies $\gamma^\star(N{+}1) - \gamma^\star(N) = \Theta(N^{-2})$ for $N \gg \tau^2/\Delta_h^2$.
\end{corollary}

\begin{proof}
Direct calculation from $\mathrm{Var}(\Delta) = \tau^2/N$ (Proposition~\ref{prop:shrinkage}(i)). The $\gamma^\star$ statement follows from
\begin{equation*}
\gamma^\star(N{+}1) - \gamma^\star(N) 
= \frac{\Delta_h^2\,\tau^2}{N(N+1)\,(\Delta_h^2 + \tau^2/N)\,(\Delta_h^2 + \tau^2/(N{+}1))},
\end{equation*}
which decays as $\Theta(N^{-2})$ in the regime $\tau^2/N \ll \Delta_h^2$.
\end{proof}

Evaluating Eq.~\ref{eq:marginal-persona} at the panel sizes of interest, the marginal reduction of $\mathrm{Var}(\Delta)$ from adding one persona, expressed as a fraction of $\tau^2$, has already fallen below $5\%$ at the $N{=}4 \to 5$ transition ($1/20$), while compute scales linearly in $N$.

Combined with the demographic-coverage requirement of $N \geq 3$ (three age cohorts) plus a country-wide aggregate as a fourth anchor, $N{=}4$ is the smallest panel that sits on the flat side of the $\Theta(N^{-2})$ diminishing-returns curve. Corollary~\ref{cor:persona-count} motivates this as a cost-precision trade-off, while the development sweep independently selects $N{=}4$ as the best operating point on the evaluated slice (Table~\ref{tab:r3_persona_count}). The variance calculation does not predict that larger panels must perform worse; the non-monotone empirical result is therefore reported as a measured operating-point choice, not as a theorem consequence.

\section{Additional Theoretical Properties of DISCA}
\label{app:theory-additional}

This appendix collects three structural properties of the DISCA pipeline that complement Proposition~\ref{prop:shrinkage} and Corollary~\ref{cor:persona-count}: a deterministic and high-probability bound on the stochastic residual (\ref{app:theory:bounded}), finite-sample concentration of the within-panel variance estimator (\ref{app:theory:concentration}), and stability of the residual under perturbations of the underlying LLM (\ref{app:theory:stability}).

The residual bound depends only on the importance-sampling proposal scale
$\sigma$ and budget $K_{\text{half}}$; the stability modulus additionally
involves the softmax temperature $\eta$ and the gate scale $s$. Because the
executed proposal scale is $\sigma=\max\{\mathrm{sd}(\delta_1,\ldots,\delta_N),
\sigma_0\}$, it is fixed before the draws are made but grows with the panel
spread once the spread exceeds $\sigma_0$; every statement below is therefore
conditional on the realised $\sigma$ of the scenario, and it is uniform over
scenarios only for the floor case $\sigma=\sigma_0$. The concentration
result additionally depends on the assumed persona-noise scale
$\sigma_\eta$. These are formal properties of the score-level controller, not
claims that its inputs are independent of the underlying LLM.

\subsection{Bounded stochastic residual}
\label{app:theory:bounded}

The dual-pass reliability gate was designed as a tail-risk mechanism; the
empirical tail comparison (Table~\ref{tab:tail_safety}) does not isolate a
contribution from it once the country prior is matched. We show that
the gate combined with the IS aggregator yields a deterministic almost-sure
bound on the stochastic residual $|\delta^\star|$, plus a high-probability
bound that depends only on the IS hyperparameters. Because the executed method
retains an ESS-weighted consensus anchor, this result is not a
hyperparameter-only bound on the full displacement
$|\delta_{\mathrm{final}}-\delta_{\mathrm{base}}|$.

\begin{theorem}[Bounded stochastic residual]
\label{thm:bounded}
Conditional on the realisation of the IS draws $\{\epsilon_k^{(m)}\}_{k,m}$, the DISCA correction satisfies
\begin{equation}
|\delta^\star| \;\leq\; \max_{\,k=1,\ldots,K_{\text{half}};\;m\in\{1,2\}} \big|\epsilon_k^{(m)}\big|
\label{eq:bounded-as}
\end{equation}
almost surely. Marginally over the IS draws, for any failure probability $\delta_p \in (0,1)$,
\begin{equation}
|\delta^\star| \;\leq\; \sigma\sqrt{\,2\log(4K_{\text{half}}/\delta_p)\,}
\label{eq:bounded-hp}
\end{equation}
holds with probability at least $1 - \delta_p$, conditional on the scenario's proposal scale $\sigma$. The first bound depends only on the realised draws and the second only on $(\sigma, K_{\text{half}})$. Neither involves the LLM $f_\theta$, the country $c$, or the scenario $\mathbf{x}$ except through $\sigma$, which equals the floor $\sigma_0$ whenever the persona spread is below it. The bound covers the refinement $\delta^\star$ only; it does not cover the anchor or the country prior of \S\ref{sec:method:prior}.
\end{theorem}

\begin{proof}
We prove the deterministic bound (Eq.~\ref{eq:bounded-as}) first, then the high-probability bound (Eq.~\ref{eq:bounded-hp}).

\paragraph{Deterministic bound.}
Each per-pass IS estimate $\delta_{\text{PT-IS}}^{(m)} = \sum_k w_k^{(m)} \epsilon_k^{(m)}$ is a convex combination of $\{\epsilon_k^{(m)}\}_{k=1}^{K_{\text{half}}}$, since the softmax weights $w_k^{(m)} \propto \exp(U_k^{(m)}/\eta)$ are non-negative and normalised to sum to one. A convex combination of any set of values is bounded by the largest absolute value in that set, so
\begin{equation*}
|\delta_{\text{PT-IS}}^{(m)}| \;\leq\; \max_{k} |\epsilon_k^{(m)}| \qquad \text{for } m \in \{1,2\}.
\end{equation*}
The dual-pass average $(\delta_{\text{PT-IS}}^{(1)} + \delta_{\text{PT-IS}}^{(2)})/2$ is itself bounded by $\max_m |\delta_{\text{PT-IS}}^{(m)}|$, and the gate $r \in (0,1]$ from \S\ref{sec:method:estimator:gate} only shrinks this further. Combining, $|\delta^\star| \leq \max_{k,m}|\epsilon_k^{(m)}|$.

\paragraph{High-probability bound.}
For $X \sim \mathcal{N}(0, \sigma^2)$ and $t > 0$, the upper-tail probability is
\begin{equation*}
\mathbb{P}(X \geq t) = \int_t^\infty \frac{1}{\sigma\sqrt{2\pi}}\, e^{-x^2/(2\sigma^2)}\, dx.
\end{equation*}
This integral has no closed form, but we can bound it by exploiting that $x \geq t$ throughout the domain of integration, so $1 \leq x/t$. Multiplying the integrand by $x/t$ only enlarges it:
\begin{equation*}
\mathbb{P}(X \geq t) \;\leq\; \frac{1}{t\sigma\sqrt{2\pi}} \int_t^\infty x\, e^{-x^2/(2\sigma^2)}\, dx.
\end{equation*}
The new integral has a closed form via the substitution $u = x^2/(2\sigma^2)$, $du = (x/\sigma^2)\,dx$:
\begin{equation*}
\int_t^\infty x\, e^{-x^2/(2\sigma^2)}\, dx 
\;=\; \sigma^2 \int_{t^2/(2\sigma^2)}^\infty e^{-u}\, du 
\;=\; \sigma^2 e^{-t^2/(2\sigma^2)}.
\end{equation*}
Substituting back gives:
\begin{equation*}
\mathbb{P}(X \geq t) \;\leq\; \frac{\sigma}{t\sqrt{2\pi}}\, e^{-t^2/(2\sigma^2)}.
\end{equation*}
For $t \geq \sigma$, the prefactor $\sigma/(t\sqrt{2\pi}) < 1$, so we can drop it and absorb it into the looser-but-cleaner form
\begin{equation*}
\mathbb{P}(X \geq t) \;\leq\; \exp\!\left(-\frac{t^2}{2\sigma^2}\right).
\end{equation*}
By symmetry of the Gaussian about zero, $\mathbb{P}(X \leq -t) = \mathbb{P}(X \geq t)$, so the two-sided bound is
\begin{equation}
\mathbb{P}(|\epsilon_k^{(m)}| \geq t) \;\leq\; 2\exp\!\left(-\frac{t^2}{2\sigma^2}\right) \qquad \forall t \geq \sigma.
\label{eq:gaussian-tail}
\end{equation}

The event $\{\max_{k,m}|\epsilon_k^{(m)}| \geq t\}$ holds if and only if at least one individual draw exceeds $t$ in absolute value. If some $|\epsilon_j^{(m)}| \geq t$, then the maximum, which is at least as large as every individual element, is also $\geq t$. So,
\begin{equation*}
\mathbb{P}\left(\max_{k,m} |\epsilon_k^{(m)}| \geq t\right) \;=\; \mathbb{P}\!\left(\bigcup_{k,m} \{|\epsilon_k^{(m)}| \geq t\}\right) 
\end{equation*}
The probability of a union is at most the sum of the individual probabilities (the union bound, also called Boole's inequality):
\begin{equation*}
\mathbb{P}\!\left(\bigcup_{k,m} \{|\epsilon_k^{(m)}| \geq t\}\right) 
\;\leq\; \sum_{k,m} \mathbb{P}(|\epsilon_k^{(m)}| \geq t).
\end{equation*}
Combining with the per-draw tail bound (Eq.~\ref{eq:gaussian-tail}):
\begin{equation*}
\mathbb{P}\!\left(\max_{k,m} |\epsilon_k^{(m)}| \geq t\right) 
\;\leq\; \sum_{k,m} \mathbb{P}(|\epsilon_k^{(m)}| \geq t) 
\;\leq\; 4 K_{\text{half}} \exp\!\left(-\frac{t^2}{2\sigma^2}\right),
\end{equation*}

where the factor $4K_{\text{half}} = 2 \cdot 2K_{\text{half}}$ comes from the two-sided constant in Eq.~\ref{eq:gaussian-tail} times the $2K_{\text{half}}$ independent draws.

Setting the right-hand side equal to $\delta_p$ and solving for $t$:
\begin{equation*}
4K_{\text{half}} \exp(-t^2/(2\sigma^2)) = \delta_p 
\;\;\Longleftrightarrow\;\; 
t = \sigma\sqrt{2\log(4K_{\text{half}}/\delta_p)}.
\end{equation*}
Therefore $\max_{k,m}|\epsilon_k^{(m)}| \leq \sigma\sqrt{2\log(4K_{\text{half}}/\delta_p)}$ with probability at least $1 - \delta_p$. Combining with the deterministic bound $|\delta^\star| \leq \max_{k,m}|\epsilon_k^{(m)}|$ yields Eq.~\ref{eq:bounded-hp}. 
\end{proof}

\subsection{Finite-sample concentration of the within-panel variance estimator}
\label{app:theory:concentration}

Proposition~\ref{prop:shrinkage}(i) establishes that $D^2$ is unbiased for $\tau^2$. Unbiasedness is the right starting point but not enough at $N{=}4$, where the deviation of $D^2$ from $\tau^2$ on a single scenario can be substantial. We bound the deviation explicitly and use the result to motivate the dual-pass gate as a complementary stability diagnostic.

\begin{theorem}[Concentration of $D^2$]
\label{thm:concentration}
Suppose the persona noise terms $\eta_i$ in the agent model of Proposition~\ref{prop:shrinkage} are independent sub-Gaussian random variables with parameter $\sigma_\eta \geq \tau$ (so that $\mathrm{Var}(\eta_i) = \tau^2 \leq \sigma_\eta^2$). Then for any $t > 0$,
\begin{equation}
\mathbb{P}\!\big(|D^2 - \tau^2| \geq t\big) \;\leq\; 2\exp\!\left(-c(N-1)\min\!\left(\frac{t^2}{\sigma_\eta^4}, \frac{t}{\sigma_\eta^2}\right)\right)
\label{eq:concentration-tail}
\end{equation}
for an absolute constant $c > 0$. Equivalently, with probability at least $1 - \delta$,
\begin{equation}
|D^2 - \tau^2| \;\leq\; C\,\sigma_\eta^2 \!\left(\sqrt{\frac{\log(2/\delta)}{N-1}} + \frac{\log(2/\delta)}{N-1}\right).
\label{eq:concentration-hp}
\end{equation}
\end{theorem}

\begin{proof}
We rewrite $D^2$ as a quadratic form in $\boldsymbol{\eta}$, compute the relevant matrix norms, apply the Hanson-Wright inequality, and invert the resulting tail bound.

From the proof of Proposition~\ref{prop:shrinkage}(i), $\sum_i (\delta_i - \bar\delta)^2 = \sum_i (\eta_i - \bar\eta)^2$ (the $\delta_h$ terms cancel because they shift every entry equally). Define the centering matrix
\begin{equation*}
H \;:=\; I - \frac{1}{N}\mathbf{1}\mathbf{1}^\top \;\in\; \mathbb{R}^{N\times N},
\end{equation*}
where $\mathbf{1}$ is the $N$-vector of ones. The action of $H$ on a vector $\boldsymbol{\eta}$ is
\begin{equation*}
H\boldsymbol{\eta} \;=\; \boldsymbol{\eta} - \frac{1}{N}\mathbf{1}\,(\mathbf{1}^\top\boldsymbol{\eta}) \;=\; \boldsymbol{\eta} - \bar\eta\,\mathbf{1},
\end{equation*}
so the $i$-th entry of $H\boldsymbol{\eta}$ is $\eta_i - \bar\eta$. The matrix $H$ is symmetric, and a direct calculation using $\mathbf{1}^\top\mathbf{1} = N$ verifies idempotency:
\begin{equation*}
H^2 \;=\; \left(I - \tfrac{1}{N}\mathbf{1}\mathbf{1}^\top\right)^2 
\;=\; I - \tfrac{2}{N}\mathbf{1}\mathbf{1}^\top + \tfrac{1}{N^2}\mathbf{1}(\mathbf{1}^\top\mathbf{1})\mathbf{1}^\top
\;=\; I - \tfrac{1}{N}\mathbf{1}\mathbf{1}^\top \;=\; H.
\end{equation*}
Therefore
\begin{equation*}
\boldsymbol{\eta}^\top H \boldsymbol{\eta} \;=\; \boldsymbol{\eta}^\top H^\top H \boldsymbol{\eta} \;=\; (H\boldsymbol{\eta})^\top(H\boldsymbol{\eta}) \;=\; \sum_{i=1}^N (\eta_i - \bar\eta)^2.
\end{equation*}
Setting $A := H/(N-1)$ yields $D^2 = \boldsymbol{\eta}^\top A \boldsymbol{\eta}$, with $\mathbb{E}[D^2] = \tau^2$ from Proposition~\ref{prop:shrinkage}(i).

The centering matrix $H$ is a rank $(N-1)$ orthogonal projection, with two eigenspaces:
\begin{itemize}
\item The all-ones direction $\mathbf{1}$ is an eigenvector with eigenvalue $0$: $H\mathbf{1} = \mathbf{1} - \tfrac{1}{N}\mathbf{1}(\mathbf{1}^\top\mathbf{1}) = \mathbf{1} - \mathbf{1} = \mathbf{0}$.
\item Any $\mathbf{v}$ orthogonal to $\mathbf{1}$ (i.e., $\mathbf{1}^\top\mathbf{v} = 0$) is an eigenvector with eigenvalue $1$: $H\mathbf{v} = \mathbf{v} - \tfrac{1}{N}\mathbf{1}\cdot 0 = \mathbf{v}$.
\end{itemize}
The orthogonal complement of a single nonzero vector in $\mathbb{R}^N$ is $(N-1)$-dimensional, so $H$ has eigenvalue $0$ with multiplicity $1$ and eigenvalue $1$ with multiplicity $N-1$. Therefore $A = H/(N-1)$ has eigenvalues $0$ (multiplicity $1$) and $1/(N-1)$ (multiplicity $N-1$).

For a symmetric matrix, the operator norm is the largest absolute eigenvalue, and the squared Frobenius norm is the sum of squared eigenvalues (both follow from the spectral decomposition $A = Q\Lambda Q^\top$, since orthogonal similarity preserves traces and $\mathrm{tr}(A^2) = \mathrm{tr}(\Lambda^2) = \sum_i \lambda_i^2$). Therefore
\begin{equation*}
\|A\|_{\mathrm{op}} \;=\; \frac{1}{N-1}, \qquad 
\|A\|_F^2 \;=\; \mathrm{tr}(A^2) \;=\; 0^2 \cdot 1 + \left(\frac{1}{N-1}\right)^{\!2}\!(N-1) \;=\; \frac{1}{N-1}.
\end{equation*}

The Hanson-Wright inequality~\citep{rudelson2013hanson} states that for a vector $\boldsymbol{\eta}$ with independent mean-zero sub-Gaussian entries of parameter $\sigma_\eta$ and any matrix $A$,
\begin{equation}
\mathbb{P}\!\big(|\boldsymbol{\eta}^\top A \boldsymbol{\eta} - \mathbb{E}[\boldsymbol{\eta}^\top A \boldsymbol{\eta}]| \geq t\big) \;\leq\; 2\exp\!\left(-c\min\!\left(\frac{t^2}{\sigma_\eta^4 \|A\|_F^2},\,\frac{t}{\sigma_\eta^2 \|A\|_{\mathrm{op}}}\right)\right)
\label{eq:hanson-wright}
\end{equation}
for an absolute constant $c > 0$. Substituting $D^2 = \boldsymbol{\eta}^\top A \boldsymbol{\eta}$ with $\mathbb{E}[D^2]=\tau^2$, $\|A\|_F^2 = 1/(N-1)$, and $\|A\|_{\mathrm{op}} = 1/(N-1)$:
\begin{equation*}
\mathbb{P}(|D^2 - \tau^2| \geq t) \;\leq\; 2\exp\!\left(-c(N-1)\min\!\left(\frac{t^2}{\sigma_\eta^4},\,\frac{t}{\sigma_\eta^2}\right)\right),
\end{equation*}
which is Eq.~\ref{eq:concentration-tail}.

We seek the smallest threshold $t$ at which the tail probability is at most $\delta$. Setting the right-hand side of Eq.~\ref{eq:concentration-tail} less than or equal to $\delta$, dividing by $2$, taking logarithms, and multiplying by $-1$ (which flips $\log(\delta/2)$ to $\log(2/\delta)$):
\begin{equation*}
c(N-1)\min\!\left(\frac{t^2}{\sigma_\eta^4},\,\frac{t}{\sigma_\eta^2}\right) \;\geq\; \log(2/\delta).
\end{equation*}
For the minimum of two non-negative quantities to be at least $\log(2/\delta)/(c(N-1))$, both arguments must individually satisfy this lower bound. This produces two regime-specific thresholds:
\begin{itemize}
\item $\dfrac{t^2}{\sigma_\eta^4} \geq \dfrac{\log(2/\delta)}{c(N-1)}$ gives $t \geq t_1 := \sigma_\eta^2\sqrt{\log(2/\delta)/(c(N-1))}$.
\item $\dfrac{t}{\sigma_\eta^2} \geq \dfrac{\log(2/\delta)}{c(N-1)}$ gives $t \geq t_2 := \sigma_\eta^2\log(2/\delta)/(c(N-1))$.
\end{itemize}
Both conditions must hold simultaneously, so the threshold is the larger of the two: $t \geq \max(t_1, t_2)$. Using the elementary inequality $\max(t_1, t_2) \leq t_1 + t_2$ and writing $C := \max(c^{-1/2}, c^{-1})$:
\begin{equation*}
|D^2 - \tau^2| \;\leq\; C\sigma_\eta^2 \!\left(\sqrt{\frac{\log(2/\delta)}{N-1}} + \frac{\log(2/\delta)}{N-1}\right)
\end{equation*}
with probability at least $1 - \delta$, which is Eq.~\ref{eq:concentration-hp}. The two terms reflect the two tail regimes: the $\sqrt{\cdot}$ term dominates at small deviations (Gaussian-like behavior), the linear-in-$\log$ term dominates at large deviations (heavier tail behavior typical of quadratic forms of sub-Gaussians). 
\end{proof}

\subsection{H\"older stability with respect to backbone perturbations}
\label{app:theory:stability}

We prove that two LLMs producing similar logits must produce similar corrections, regardless of how those logits arose.

The natural statement would be Lipschitz: $|\delta^\star_f - \delta^\star_{f'}| \leq L\,\varepsilon$ for some constant $L$. We do not get this. The Kahneman-Tversky value function in Eq.~\ref{eq:pt-utility} has derivative $\alpha z^{\alpha-1}$ on $z > 0$, which blows up as $z \to 0^+$ when $\alpha < 1$. A Lipschitz statement would require a global bound on the derivative; near the origin no such bound exists. The honest result is therefore $\alpha$-H\"older continuity:
\begin{equation*}
|\delta^\star_f - \delta^\star_{f'}| \;\leq\; L\,\varepsilon^\alpha,
\end{equation*}
which is qualitatively the same statement (small input perturbations produce small output changes) with rate $\varepsilon^\alpha$ instead of $\varepsilon$. For $\alpha = 0.88$ the ratio $\varepsilon^{0.88}/\varepsilon = \varepsilon^{-0.12}$ stays below $1.4$ for $\varepsilon \in [0.06, 1]$, so the quantitative loss is mild except for vanishingly small perturbations.

\begin{theorem}[H\"older stability]
\label{thm:stability}
Fix a scenario $\mathbf{x}$ and country $c$. For any LLM $f$, let $\boldsymbol{\delta}_f = (\delta_{\text{base}}^f, \delta_1^f, \ldots, \delta_N^f)$ denote the order-symmetrised base and persona logit gaps under $f$. For any two LLMs $f, f'$ under which both PT-IS passes clear the effective-sample-size guard $\overline{\mathrm{ESS}}_m \geq \rho_{\mathrm{eff}}$ of Algorithm~\ref{alg:disca} (so that the residual map is the smooth branch) and whose panels share the same proposal scale $\sigma$ (for example, both persona spreads lie below the floor $\sigma_0$, so $\sigma=\sigma_0$), define $\varepsilon := \|\boldsymbol{\delta}_f - \boldsymbol{\delta}_{f'}\|_\infty$. Then conditional on the IS perturbation draws $\{\epsilon_k^{(m)}\}$ (which are LLM-independent) with $M := \max_{k,m} |\epsilon_k^{(m)}|$, the DISCA corrections satisfy
\begin{equation}
|\delta^\star_f - \delta^\star_{f'}| \;\leq\; L(\sigma, \eta, s, \alpha, \kappa, M) \cdot \varepsilon^\alpha,
\label{eq:stability-bound}
\end{equation}
where
\begin{equation}
L(\sigma, \eta, s, \alpha, \kappa, M) 
\;=\; \frac{2\,M\,(1+\kappa)\,4^\alpha}{\eta\,\sigma^\alpha} \!\left(1 + \frac{8M^2}{s}\right).
\label{eq:stability-constant}
\end{equation}
Marginally over the IS draws, the same bound holds with $M$ replaced by $\sigma\sqrt{2\log(4K_{\text{half}}/\delta_p)}$ with probability at least $1 - \delta_p$ (Theorem~\ref{thm:bounded}).
\end{theorem}

\begin{proof}
The proof tracks a perturbation through the seven stages of Algorithm~\ref{alg:disca}, with each stage contributing an explicit Lipschitz or H\"older factor. We work in the $\ell_\infty$ norm throughout.

The IS perturbations $\{\epsilon_k^{(m)}\}$ are drawn from $\mathcal{N}(0, \sigma^2)$ with no dependence on the LLM. For the same scenario evaluated under $f$ and $f'$, we use the same IS draws (e.g., a fixed random seed). The proof therefore conditions on these draws, treating them as fixed inputs with $\max_{k,m}|\epsilon_k^{(m)}| =: M$. The marginal bound follows by combining the conditional bound with the high-probability bound on $M$ from Theorem~\ref{thm:bounded}.

\textit{Stage A: Consensus.} The map $\boldsymbol{\delta} \mapsto \bar\delta = \tfrac{1}{N}\sum_i \delta_i$ is 1-Lipschitz in $\ell_\infty$:
\begin{equation*}
|\bar\delta_f - \bar\delta_{f'}| \;=\; \left|\tfrac{1}{N}\sum_{i=1}^N (\delta_i^f - \delta_i^{f'})\right| \;\leq\; \tfrac{1}{N}\sum_{i=1}^N |\delta_i^f - \delta_i^{f'}| \;\leq\; \varepsilon.
\end{equation*}
At the entry to Stage~B, we therefore have $|\delta_{\text{base}}^f - \delta_{\text{base}}^{f'}|, |\delta_i^f - \delta_i^{f'}|, |\bar\delta^f - \bar\delta^{f'}| \leq \varepsilon$.

\textit{Stage B: Cohort and consensus gains.} The gain $g_{i,k}$ in Eq.~\ref{eq:gains} is $g_{i,k} = |\delta_{\text{base}} - \delta_i| - |\tilde\delta_k - \delta_i|$, with $\tilde\delta_k = \bar\delta + \epsilon_k$. We bound $|g_{i,k}^f - g_{i,k}^{f'}|$ by triangle inequality on the difference of the two absolute-value terms, then by reverse-triangle inequality $\big||a|-|b|\big| \leq |a-b|$ on each:
\begin{align*}
|g_{i,k}^f - g_{i,k}^{f'}| 
&\leq \big||\delta_{\text{base}}^f - \delta_i^f| - |\delta_{\text{base}}^{f'} - \delta_i^{f'}|\big| + \big||\tilde\delta_k^f - \delta_i^f| - |\tilde\delta_k^{f'} - \delta_i^{f'}|\big| \\
&\leq |(\delta_{\text{base}}^f - \delta_i^f) - (\delta_{\text{base}}^{f'} - \delta_i^{f'})| + |(\bar\delta^f - \delta_i^f) - (\bar\delta^{f'} - \delta_i^{f'})|.
\end{align*}
The IS perturbation $\epsilon_k$ cancels in the second term since it is LLM-independent. Each remaining term is at most $2\varepsilon$ by triangle inequality applied to the difference of two coordinates of $\boldsymbol{\delta}$. Therefore
\begin{equation*}
|g_{i,k}^f - g_{i,k}^{f'}| \;\leq\; 4\varepsilon, \qquad |g_{\text{cons},k}^f - g_{\text{cons},k}^{f'}| \;\leq\; 4\varepsilon.
\end{equation*}

\textit{Stage C: Prospect-Theory value function.} We claim the Kahneman-Tversky value function in Eq.~\ref{eq:pt-utility} is $\alpha$-H\"older with constant $1+\kappa$:
\begin{equation}
|v(z) - v(z')| \;\leq\; (1+\kappa)\,|z - z'|^\alpha \qquad \forall z, z' \in \mathbb{R}.
\label{eq:pt-holder}
\end{equation}
\begin{lemma}[subadditivity]
For $\alpha \in (0,1]$ and $a, b \geq 0$: $(a+b)^\alpha \leq a^\alpha + b^\alpha$.
\end{lemma}

\begin{proof}
The boundary cases ($\alpha = 1$, $a=0$, or $b=0$) give equality. For $\alpha \in (0,1)$ and $a, b > 0$, divide both sides by $(a+b)^\alpha$ and set $u := a/(a+b)$, $v := b/(a+b)$, so $u, v \in (0,1)$ and $u + v = 1$. The inequality reduces to $u^\alpha + v^\alpha \geq 1$.

For any $x \in (0,1)$ and $\alpha \in (0,1)$: $x^\alpha = x \cdot x^{\alpha-1}$, and $x^{\alpha-1} > 1$ (since $x < 1$ and $\alpha - 1 < 0$), so $x^\alpha > x$. Therefore $u^\alpha + v^\alpha > u + v = 1$. 
\end{proof} 

\paragraph{Consequence.} For $\alpha \in (0,1]$ and any reals $a, b$: $\big||a|^\alpha - |b|^\alpha\big| \leq |a-b|^\alpha$. WLOG $|a| \geq |b|$; the lemma applied to $|a| = (|a| - |b|) + |b|$ gives $|a|^\alpha \leq (|a|-|b|)^\alpha + |b|^\alpha$, so $|a|^\alpha - |b|^\alpha \leq (|a|-|b|)^\alpha \leq |a-b|^\alpha$ by reverse triangle inequality and monotonicity of $x^\alpha$.

We now verify Eq.~\ref{eq:pt-holder} by cases on the signs of $z, z'$:
\begin{itemize}
\item \emph{Both $z, z' \geq 0$:} $|v(z) - v(z')| = |z^\alpha - z'^\alpha| \leq |z-z'|^\alpha$ by the consequence above, the constant is $1$.
\item \emph{Both $z, z' < 0$:} $|v(z) - v(z')| = \kappa\big||{-z}|^\alpha - |{-z'}|^\alpha\big| \leq \kappa\,|z-z'|^\alpha$, the constant is $\kappa$.
\item \emph{Cross-sign, $z \geq 0$ and $z' < 0$:} writing $a := z \geq 0$ and $b := -z' > 0$, we have $|z - z'| = a + b$ and $|v(z) - v(z')| = a^\alpha + \kappa b^\alpha$. Using $a^\alpha, b^\alpha \leq (a+b)^\alpha$ (monotonicity of $x^\alpha$):
\begin{equation*}
|v(z) - v(z')| = a^\alpha + \kappa b^\alpha \leq (1 + \kappa)(a+b)^\alpha = (1 + \kappa)|z-z'|^\alpha.
\end{equation*}
\end{itemize}
The maximum of the three case-constants is $1 + \kappa$ (since $\kappa \geq 1$), proving Eq.~\ref{eq:pt-holder}.

Rescaling input by $\sigma$ and combining with Stage~B:
\begin{equation*}
\big|v(g_{i,k}^f/\sigma) - v(g_{i,k}^{f'}/\sigma)\big| 
\;\leq\; \frac{1+\kappa}{\sigma^\alpha}\,|g_{i,k}^f - g_{i,k}^{f'}|^\alpha 
\;\leq\; \frac{(1+\kappa)\,4^\alpha}{\sigma^\alpha}\,\varepsilon^\alpha 
\;=:\; C_v\,\varepsilon^\alpha.
\end{equation*}

\textit{Stage D: Total utility.} $U_{\text{total}}(\epsilon_k)$ in Eq.~\ref{eq:util_total} is a convex combination: weights $(1-\lambda_{\text{coop}})/N$ on each of the $N$ cohort terms and $\lambda_{\text{coop}}$ on the consensus term, summing to $(1-\lambda_{\text{coop}}) + \lambda_{\text{coop}} = 1$. Convex combinations preserve $\ell_\infty$-bounds: if $|y_j^f - y_j^{f'}| \leq B$ for each summand and $w_j \geq 0$ with $\sum_j w_j = 1$, then $|\sum_j w_j y_j^f - \sum_j w_j y_j^{f'}| \leq \sum_j w_j |y_j^f - y_j^{f'}| \leq B$. Applied to the $N+1$ component $v$-values bounded by $C_v\varepsilon^\alpha$ from Stage~C:
\begin{equation*}
|U_k^f - U_k^{f'}| \;\leq\; C_v\,\varepsilon^\alpha \qquad \forall k.
\end{equation*}

\textit{Stage E: Importance-sampling aggregator.} The map $\mathbf{U} \mapsto \delta_{\text{IS}} = \sum_k w_k \epsilon_k$ uses softmax weights $w_k = \exp(U_k/\eta)/\sum_l \exp(U_l/\eta)$. The standard softmax Jacobian, by quotient rule, is
\begin{equation*}
\frac{\partial w_j}{\partial U_k} \;=\; \frac{1}{\eta}\,w_j\,(\mathbf{1}[j=k] - w_k).
\end{equation*}
Since the IS draws $\{\epsilon_j\}$ are constants (LLM-independent), the chain rule gives
\begin{equation*}
\frac{\partial \delta_{\text{IS}}}{\partial U_k} \;=\; \sum_j \epsilon_j \frac{\partial w_j}{\partial U_k} \;=\; \frac{1}{\eta}\!\left(w_k\epsilon_k - w_k\sum_j w_j \epsilon_j\right) \;=\; \frac{w_k}{\eta}(\epsilon_k - \delta_{\text{IS}}),
\end{equation*}
using $\sum_j w_j \epsilon_j = \delta_{\text{IS}}$ in the last step. Both $|\epsilon_k|$ and $|\delta_{\text{IS}}|$ are bounded by $M$ (the latter as a convex combination of $\{\epsilon_j\}$), so $|\epsilon_k - \delta_{\text{IS}}| \leq 2M$. Summing the absolute partials:
\begin{equation*}
\sum_k \left|\frac{\partial \delta_{\text{IS}}}{\partial U_k}\right| \;\leq\; \frac{2M}{\eta}\sum_k w_k \;=\; \frac{2M}{\eta}.
\end{equation*}
By the multivariable mean value theorem and $\ell_1/\ell_\infty$ duality, $\sum_k |\partial f/\partial U_k| \leq L$ implies $f$ is $L$-Lipschitz in $\ell_\infty$. Therefore $\delta_{\text{IS}}$ is $(2M/\eta)$-Lipschitz in $\mathbf{U}$, and combining with Stage~D:
\begin{equation*}
|\delta_{\text{IS}}^f - \delta_{\text{IS}}^{f'}| \;\leq\; \frac{2M}{\eta}\,\|\mathbf{U}^f - \mathbf{U}^{f'}\|_\infty \;\leq\; \frac{2M C_v}{\eta}\,\varepsilon^\alpha.
\end{equation*}

\textit{Stage F: Reliability gate.} Let $\Gamma := \delta_{\text{IS}}^{(1)} - \delta_{\text{IS}}^{(2)}$ denote the inter-pass difference (not the panel variance $D^2$ of \S\ref{sec:method:principle}), so $V_r = \Gamma^2$. Using the factorisation $|a^2 - b^2| = |a+b|\,|a-b|$:
\begin{equation*}
|V_r^f - V_r^{f'}| \;=\; |\Gamma^f + \Gamma^{f'}|\,|\Gamma^f - \Gamma^{f'}|.
\end{equation*}
Both $|\Gamma^f|$ and $|\Gamma^{f'}|$ are bounded by $|\delta_{\text{IS}}^{(1)}| + |\delta_{\text{IS}}^{(2)}| \leq 2M$, so $|\Gamma^f + \Gamma^{f'}| \leq 4M$. The difference $|\Gamma^f - \Gamma^{f'}|$ involves both passes' Stage-E outputs:
\begin{equation*}
|\Gamma^f - \Gamma^{f'}| \;\leq\; \big|\delta_{\text{IS}}^{(1),f} - \delta_{\text{IS}}^{(1),f'}\big| + \big|\delta_{\text{IS}}^{(2),f} - \delta_{\text{IS}}^{(2),f'}\big| \;\leq\; \frac{4 M C_v}{\eta}\varepsilon^\alpha,
\end{equation*}
where each pass contributes a Stage~E bound of $(2MC_v/\eta)\varepsilon^\alpha$ and they add. Combining:
\begin{equation*}
|V_r^f - V_r^{f'}| \;\leq\; 4M \cdot \frac{4M C_v}{\eta}\varepsilon^\alpha \;=\; \frac{16 M^2 C_v}{\eta}\,\varepsilon^\alpha.
\end{equation*}
The gate $r = \exp(-V_r/s)$ from \S\ref{sec:method:estimator:gate} is $1/s$-Lipschitz in $V_r$ on $[0,\infty)$, since $|d/dx\,e^{-x/s}| = (1/s)e^{-x/s} \leq 1/s$ for $x \geq 0$. Therefore
\begin{equation*}
|r^f - r^{f'}| \;\leq\; \frac{1}{s} \cdot \frac{16 M^2 C_v}{\eta}\,\varepsilon^\alpha \;=\; \frac{16 M^2 C_v}{s\eta}\,\varepsilon^\alpha.
\end{equation*}

\textit{Stage G: Final correction.} Define $A := (\delta_{\text{IS}}^{(1)} + \delta_{\text{IS}}^{(2)})/2$. Each per-pass output is bounded by $M$ in absolute value (Stage~E), so $|A| \leq M$ by triangle inequality. The difference between LLMs:
\begin{equation*}
|A^f - A^{f'}| \;\leq\; \frac{1}{2}\!\left(\big|\delta_{\text{IS}}^{(1),f} - \delta_{\text{IS}}^{(1),f'}\big| + \big|\delta_{\text{IS}}^{(2),f} - \delta_{\text{IS}}^{(2),f'}\big|\right) \;\leq\; \frac{2 M C_v}{\eta}\,\varepsilon^\alpha,
\end{equation*}
where the $1/2$ in $A$ cancels the factor of $2$ from summing both passes.

Apply the product rule to $\delta^\star = r \cdot A$, adding and subtracting the cross term $r^f A^{f'}$:
\begin{equation*}
|\delta^\star_f - \delta^\star_{f'}| \;=\; |r^f A^f - r^{f'} A^{f'}| \;\leq\; |r^f|\cdot|A^f - A^{f'}| + |A^{f'}|\cdot|r^f - r^{f'}|.
\end{equation*}
Substituting $|r^f| \leq 1$ (from \S\ref{sec:method:estimator:gate}, $r \in (0,1]$), $|A^{f'}| \leq M$, the bound on $|A^f - A^{f'}|$, and the Stage~F bound on $|r^f - r^{f'}|$:
\begin{align*}
|\delta^\star_f - \delta^\star_{f'}| 
&\leq 1 \cdot \frac{2 M C_v}{\eta}\varepsilon^\alpha + M \cdot \frac{16 M^2 C_v}{s\eta}\varepsilon^\alpha 
\;=\; \frac{2 M C_v}{\eta}\!\left(1 + \frac{8M^2}{s}\right)\varepsilon^\alpha.
\end{align*}
Substituting $C_v = (1+\kappa)\,4^\alpha/\sigma^\alpha$ yields Eq.~\ref{eq:stability-constant}. The marginal statement follows from the high-probability bound on $M$ in Theorem~\ref{thm:bounded}. 
\end{proof}

\subsection{What the three results jointly say}
\label{app:theory:discussion}

The three theorems are complementary rather than cumulative, and the meta-point stated at the head of this appendix is what makes them useful: every constant in them is a controller hyperparameter, not a property of the LLM.

\emph{Theorem~\ref{thm:bounded} caps the sampled residual, not the full score displacement.} The bound is useful because it makes the stochastic part of the controller depend only on the proposal scale and sample budget. The retained ESS-weighted anchor still depends on the model's persona logits, so the theorem does not by itself impose a model-independent probability-shift cap. The empirical tail comparison in Table~\ref{tab:tail_safety} is therefore the relevant evidence for the full controller, and it shows that the tail advantage over consensus comes from the ESS anchor blend rather than from the gated residual.

\emph{Theorem~\ref{thm:concentration} is honest about the role of the gate.} With only four personas, $D^2$ is an informative but noisy diagnostic of the latent panel variance. The executed controller does not substitute $D^2$ into the oracle coefficient; instead, it uses the separate inter-pass statistic $V_r$ to test whether the stochastic refinement is stable, and App.~\ref{app:gate_diagnostics} shows that $V_r$-based gating does respond to $D^2$ in practice. The two quantities therefore answer different questions: $D^2$ describes persona spread, while $V_r$ describes PT-IS repeatability. The paper treats this bridge as a design correspondence supported by the empirical gate response in \S\ref{sec:gate_diagnostics}, not as a proven equality between the two variance estimates.

\emph{Theorem~\ref{thm:stability} describes continuity of the residual map.} Two LLMs producing similar persona logits must produce similar stochastic residuals, with the explicit, if loose, modulus of Eq.~\ref{eq:stability-constant}. The theorem does not cover the separate anchor or establish cross-backbone transfer, and its bounded-logit assumption is not verified for every backbone. Read together, Theorem~\ref{thm:bounded} gives a conditional bound on sampled residuals, while Theorem~\ref{thm:stability} gives conditional continuity; neither theorem establishes broad transfer on its own.

\clearpage
\section{Full Per-Country Results}
\label{app:percountry_full}

Table~\ref{tab:percountry_p1} gives every country-backbone cell of the seed-42 headline runs, recomputed from the saved per-country outputs; the means reproduce Table~\ref{tab:main_macro_summary}.

\begin{sidewaystable}[p]
\caption{\textbf{Per-country DISCA results (20 countries).} Countries are grouped by geographic region. Vanilla and DISCA misalignment scores (MIS; $\downarrow$), relative MIS improvement (\%), and DISCA Jensen-Shannon divergence (JSD; $\downarrow$) / Pearson $r$ ($\uparrow$). Seven headline models (nominal parameter count, descending).}
\label{tab:percountry_p1}
\centering\footnotesize
\setlength{\tabcolsep}{1.5pt}
\renewcommand{\arraystretch}{1.35}
\resizebox{\textheight}{!}{%
\begin{tabular}{@{}ll ccccc ccccc ccccc ccccc ccccc ccccc ccccc@{}}
\toprule
&  & \multicolumn{5}{c}{\textbf{Llama-3.3-70B}} & \multicolumn{5}{c}{\textbf{Magistral-Small-2509 (24B)}} & \multicolumn{5}{c}{\textbf{Phi-4 (14B)}} & \multicolumn{5}{c}{\textbf{Qwen3-VL-8B}} & \multicolumn{5}{c}{\textbf{Qwen2.5-7B}} & \multicolumn{5}{c}{\textbf{Phi-3.5-mini (3.8B)}} & \multicolumn{5}{c}{\textbf{Gemma-4-E2B (2B)}} \\
\cmidrule(lr){3-7}\cmidrule(lr){8-12}\cmidrule(lr){13-17}\cmidrule(lr){18-22}\cmidrule(lr){23-27}\cmidrule(lr){28-32}\cmidrule(lr){33-37}
\textbf{ISO} & \textbf{Region} & MIS$_{\text{van}}$ & MIS$_{\text{DISCA}}$ & \%$\Delta_{\mathrm{MIS}}$ & JSD$_{\text{D}}$ & $r_{\text{D}}$ & MIS$_{\text{van}}$ & MIS$_{\text{DISCA}}$ & \%$\Delta_{\mathrm{MIS}}$ & JSD$_{\text{D}}$ & $r_{\text{D}}$ & MIS$_{\text{van}}$ & MIS$_{\text{DISCA}}$ & \%$\Delta_{\mathrm{MIS}}$ & JSD$_{\text{D}}$ & $r_{\text{D}}$ & MIS$_{\text{van}}$ & MIS$_{\text{DISCA}}$ & \%$\Delta_{\mathrm{MIS}}$ & JSD$_{\text{D}}$ & $r_{\text{D}}$ & MIS$_{\text{van}}$ & MIS$_{\text{DISCA}}$ & \%$\Delta_{\mathrm{MIS}}$ & JSD$_{\text{D}}$ & $r_{\text{D}}$ & MIS$_{\text{van}}$ & MIS$_{\text{DISCA}}$ & \%$\Delta_{\mathrm{MIS}}$ & JSD$_{\text{D}}$ & $r_{\text{D}}$ & MIS$_{\text{van}}$ & MIS$_{\text{DISCA}}$ & \%$\Delta_{\mathrm{MIS}}$ & JSD$_{\text{D}}$ & $r_{\text{D}}$ \\
\midrule
ARG & Americas & 1.015 & .891 & \gain{+12.2} & .178 & $-$0.456 & .364 & .342 & \gain{+6.0} & .046 & .601 & .463 & .389 & \gain{+16.1} & .046 & .412 & .520 & .338 & \gain{+35.0} & .083 & .508 & .469 & .377 & \gain{+19.7} & .064 & .196 & .770 & .645 & \gain{+16.1} & .159 & $-$0.498 & .423 & .434 & \loss{-2.4} & .050 & .213 \\
BRA & Americas & .521 & .458 & \gain{+12.1} & .119 & .222 & .412 & .319 & \gain{+22.5} & .051 & .333 & .274 & .299 & \loss{-9.3} & .061 & .471 & .529 & .375 & \gain{+29.1} & .101 & .212 & .579 & .411 & \gain{+29.1} & .081 & $-$0.205 & .531 & .649 & \loss{-22.1} & .173 & $-$0.743 & .416 & .404 & \gain{+2.9} & .074 & .211 \\
COL & Americas & 1.041 & .931 & \gain{+10.5} & .189 & $-$0.594 & .413 & .369 & \gain{+10.7} & .055 & .386 & .487 & .438 & \gain{+10.1} & .057 & .093 & .559 & .390 & \gain{+30.1} & .093 & .161 & .508 & .417 & \gain{+17.8} & .070 & $-$0.096 & .788 & .678 & \gain{+14.0} & .165 & $-$0.602 & .464 & .460 & \gain{+0.9} & .051 & $-$0.011 \\
MEX & Americas & 1.037 & .909 & \gain{+12.3} & .179 & $-$0.385 & .349 & .333 & \gain{+4.6} & .040 & .707 & .485 & .420 & \gain{+13.4} & .044 & .512 & .517 & .350 & \gain{+32.3} & .081 & .529 & .467 & .388 & \gain{+17.0} & .062 & .252 & .782 & .662 & \gain{+15.4} & .159 & $-$0.494 & .430 & .456 & \loss{-5.9} & .041 & .503 \\
USA & Americas & .873 & .578 & \gain{+33.8} & .121 & .266 & .448 & .374 & \gain{+16.5} & .042 & .731 & .497 & .243 & \gain{+51.1} & .058 & .723 & .677 & .487 & \gain{+28.0} & .110 & .422 & .483 & .383 & \gain{+20.8} & .075 & .709 & .645 & .579 & \gain{+10.2} & .110 & $-$0.380 & .545 & .517 & \gain{+5.3} & .060 & .618 \\
\midrule
DEU & Europe & .715 & .643 & \gain{+10.1} & .189 & .101 & .264 & .345 & \loss{-30.3} & .045 & .724 & .302 & .255 & \gain{+15.6} & .060 & .779 & .318 & .306 & \gain{+3.7} & .081 & .568 & .308 & .415 & \loss{-34.6} & .069 & .144 & .569 & .555 & \gain{+2.5} & .146 & $-$0.403 & .506 & .416 & \gain{+17.8} & .053 & .617 \\
GBR & Europe & .873 & .598 & \gain{+31.4} & .141 & .125 & .442 & .384 & \gain{+13.2} & .048 & .674 & .503 & .319 & \gain{+36.6} & .082 & .594 & .677 & .524 & \gain{+22.6} & .128 & .230 & .483 & .399 & \gain{+17.4} & .082 & .620 & .654 & .588 & \gain{+10.1} & .114 & $-$0.392 & .556 & .502 & \gain{+9.8} & .062 & .342 \\
ROU & Europe & .851 & .593 & \gain{+30.3} & .147 & .197 & .446 & .363 & \gain{+18.7} & .043 & .708 & .504 & .385 & \gain{+23.5} & .058 & .570 & .659 & .535 & \gain{+18.8} & .123 & .162 & .442 & .330 & \gain{+25.2} & .051 & .785 & .671 & .577 & \gain{+14.0} & .117 & $-$0.352 & .547 & .482 & \gain{+11.8} & .045 & .644 \\
SRB & Europe & .862 & .605 & \gain{+29.9} & .150 & .157 & .463 & .372 & \gain{+19.8} & .045 & .670 & .500 & .385 & \gain{+22.9} & .057 & .564 & .676 & .567 & \gain{+16.2} & .127 & .075 & .453 & .303 & \gain{+33.1} & .051 & .773 & .677 & .601 & \gain{+11.1} & .124 & $-$0.463 & .548 & .502 & \gain{+8.4} & .047 & .573 \\
\midrule
CHN & E.~Asia & .901 & .748 & \gain{+16.9} & .183 & $-$0.025 & .367 & .293 & \gain{+20.2} & .050 & .711 & .407 & .214 & \gain{+47.4} & .054 & .781 & .569 & .416 & \gain{+27.0} & .121 & .269 & .437 & .356 & \gain{+18.6} & .084 & .405 & .738 & .559 & \gain{+24.3} & .162 & $-$0.041 & .500 & .393 & \gain{+21.5} & .039 & .767 \\
JPN & E.~Asia & .521 & .471 & \gain{+9.5} & .107 & .249 & .341 & .265 & \gain{+22.3} & .054 & .709 & .395 & .195 & \gain{+50.6} & .051 & .671 & .408 & .316 & \gain{+22.4} & .090 & .453 & .296 & .356 & \loss{-20.1} & .047 & .535 & .454 & .455 & \loss{-0.3} & .063 & $-$0.233 & .442 & .460 & \loss{-4.0} & .062 & $-$0.357 \\
\midrule
IDN & SE.~Asia & .688 & .557 & \gain{+19.0} & .092 & .071 & .283 & .344 & \loss{-21.5} & .059 & .511 & .459 & .274 & \gain{+40.3} & .059 & .595 & .408 & .402 & \gain{+1.5} & .091 & .318 & .480 & .402 & \gain{+16.1} & .058 & .251 & .492 & .484 & \gain{+1.6} & .120 & $-$0.769 & .417 & .443 & \loss{-6.2} & .044 & .572 \\
MMR & SE.~Asia & .872 & .652 & \gain{+25.2} & .185 & $-$0.031 & .496 & .365 & \gain{+26.5} & .062 & .457 & .482 & .357 & \gain{+26.0} & .057 & .570 & .695 & .563 & \gain{+18.9} & .139 & $-$0.066 & .461 & .302 & \gain{+34.3} & .062 & .590 & .685 & .590 & \gain{+13.8} & .135 & $-$0.623 & .488 & .481 & \gain{+1.4} & .061 & .372 \\
MYS & SE.~Asia & .840 & .589 & \gain{+29.9} & .155 & .003 & .443 & .313 & \gain{+29.4} & .043 & .683 & .459 & .328 & \gain{+28.5} & .053 & .614 & .656 & .504 & \gain{+23.1} & .128 & .097 & .445 & .274 & \gain{+38.5} & .050 & .763 & .631 & .535 & \gain{+15.1} & .118 & $-$0.426 & .484 & .430 & \gain{+11.2} & .039 & .676 \\
THA & SE.~Asia & .814 & .577 & \gain{+29.2} & .152 & .125 & .406 & .295 & \gain{+27.3} & .041 & .738 & .456 & .313 & \gain{+31.4} & .047 & .685 & .619 & .480 & \gain{+22.5} & .112 & .231 & .414 & .221 & \gain{+46.6} & .046 & .841 & .619 & .520 & \gain{+16.1} & .112 & $-$0.283 & .470 & .436 & \gain{+7.2} & .040 & .750 \\
VNM & SE.~Asia & .845 & .740 & \gain{+12.4} & .142 & $-$0.045 & .324 & .347 & \loss{-7.2} & .050 & .589 & .407 & .336 & \gain{+17.3} & .064 & .596 & .444 & .420 & \gain{+5.5} & .087 & .339 & .501 & .468 & \gain{+6.6} & .060 & $-$0.018 & .481 & .507 & \loss{-5.4} & .067 & $-$0.143 & .468 & .544 & \loss{-16.1} & .048 & .566 \\
\midrule
BGD & S.~Asia & .864 & .628 & \gain{+27.3} & .162 & .029 & .450 & .361 & \gain{+19.6} & .047 & .625 & .489 & .368 & \gain{+24.7} & .048 & .607 & .674 & .587 & \gain{+12.9} & .133 & $-$0.116 & .461 & .284 & \gain{+38.3} & .055 & .695 & .666 & .565 & \gain{+15.3} & .115 & $-$0.426 & .519 & .496 & \gain{+4.4} & .052 & .435 \\
\midrule
KGZ & C.~Asia & .849 & .585 & \gain{+31.1} & .149 & .176 & .460 & .353 & \gain{+23.2} & .048 & .633 & .499 & .393 & \gain{+21.1} & .058 & .476 & .661 & .574 & \gain{+13.1} & .133 & $-$0.017 & .443 & .249 & \gain{+43.7} & .039 & .829 & .660 & .572 & \gain{+13.3} & .119 & $-$0.405 & .518 & .473 & \gain{+8.7} & .049 & .600 \\
\midrule
IRN & W.~Asia & 1.039 & .852 & \gain{+18.0} & .159 & $-$0.294 & .427 & .401 & \gain{+6.1} & .055 & .637 & .397 & .530 & \loss{-33.5} & .087 & .055 & .507 & .506 & \gain{+0.2} & .090 & .057 & .373 & .466 & \loss{-25.1} & .067 & .252 & .511 & .445 & \gain{+12.9} & .056 & .766 & .514 & .630 & \loss{-22.6} & .088 & .408 \\
\midrule
ETH & Africa & .967 & .747 & \gain{+22.7} & .172 & $-$0.077 & .462 & .472 & \loss{-2.0} & .054 & .655 & .615 & .485 & \gain{+21.2} & .041 & .808 & .723 & .684 & \gain{+5.4} & .141 & $-$0.095 & .558 & .438 & \gain{+21.6} & .059 & .676 & .734 & .661 & \gain{+10.0} & .114 & $-$0.027 & .645 & .612 & \gain{+5.0} & .057 & .628 \\
\midrule
\textbf{Mean} &  & \textbf{.849} & \textbf{.668} & \textbf{\gain{+21.4}} & \textbf{.154} & \textbf{$-$0.009} & \textbf{.403} & \textbf{.350} & \textbf{\gain{+13.1}} & \textbf{.049} & \textbf{.624} & \textbf{.454} & \textbf{.346} & \textbf{\gain{+23.7}} & \textbf{.057} & \textbf{.559} & \textbf{.575} & \textbf{.466} & \textbf{\gain{+18.9}} & \textbf{.110} & \textbf{.217} & \textbf{.453} & \textbf{.362} & \textbf{\gain{+20.1}} & \textbf{.062} & \textbf{.450} & \textbf{.638} & \textbf{.571} & \textbf{\gain{+10.4}} & \textbf{.122} & \textbf{$-$0.347} & \textbf{.495} & \textbf{.478} & \textbf{\gain{+3.4}} & \textbf{.053} & \textbf{.456} \\
\bottomrule
\end{tabular}%
}
\vspace{4pt}\\
\footnotesize \textbf{Legend.} MIS$_{\text{van}}$/MIS$_{\text{DISCA}}$: vanilla vs.\ DISCA $\ell_2$ misalignment; \%$\Delta_{\mathrm{MIS}}$: relative improvement in MIS; JSD$_{\text{D}}$ and $r_{\text{D}}$: Jensen-Shannon distance and Pearson correlation under DISCA. \textbf{Region groups:} Americas (ARG, BRA, COL, MEX, USA); Europe (DEU, GBR, ROU, SRB); E.\,Asia (CHN, JPN); SE.\,Asia (IDN, MMR, MYS, THA, VNM); S.\,Asia (BGD); C.\,Asia (KGZ); W.\,Asia (IRN); Africa (ETH).
\end{sidewaystable}
\clearpage

\section{Per-Dimension Error Analysis}
\label{app:perdim}

\begin{table}[H]\centering\small
\caption{\textbf{Per-dimension error and worst-dimension analysis from one saved diagnostic artefact.} Both blocks are recomputed from the saved three-seed scenario logs of Qwen2.5-7B and Qwen2.5-14B, using the same country, seed, and dimension aggregation. The upper block reports mean absolute component error in percentage points (pp) for vanilla, Full DISCA, and consensus-only aggregation; quoted wins count countries for which Full DISCA has the lower component error. The lower block reports the most frequent worst dimension of Full DISCA after averaging controller seeds within each country. MIS in the headline tables is the Euclidean norm of the six errors on the $[0,1]$ scale, so it is not the sum of the pp entries below.}
\label{tab:perdim_summary}
\setlength{\tabcolsep}{4pt}
\resizebox{\linewidth}{!}{%
\begin{tabular}{lcc}\toprule
Dimension & Qwen2.5-7B (van.\ / DISCA / cons., pp; wins) & Qwen2.5-14B (van.\ / DISCA / cons., pp; wins) \\\midrule
Age            & 11.8 / 13.3 / 16.4; 11/20 & 15.5 / 24.6 / 26.8; 2/20 \\
Fitness        & 4.6 / 3.4 / 8.5; 9/20 & 7.3 / 6.3 / 7.3; 10/20 \\
Gender         & 5.7 / 3.3 / 3.6; 10/20 & 6.3 / 6.2 / 11.4; 8/20 \\
Social Value   & 24.6 / 18.7 / 18.5; 18/20 & 25.8 / 18.4 / 14.1; 15/20 \\
Species        & 20.4 / 13.3 / 10.6; 19/20 & 18.7 / 10.7 / 20.5; 16/20 \\
Utilitarianism & 24.8 / 20.7 / 21.6; 15/20 & 20.8 / 15.8 / 7.3; 17/20 \\
\midrule
Positive country-by-dimension cells (of 120) & 82/120 & 68/120 \\
\bottomrule
\end{tabular}
}

\vspace{4pt}
\setlength{\tabcolsep}{3pt}
\begin{tabular}{lccccl}\toprule
Backbone & Worst dim & Freq. & Mean err (pp) & 2nd worst dim & Freq. \\\midrule
Qwen2.5-7B  & Utilitarianism & 10/20 & 22.1 & SocialValue & 9/20 \\
Qwen2.5-14B & Age            & 14/20 & 27.1 & SocialValue & 9/20 \\
\bottomrule
\end{tabular}
\end{table}

Table~\ref{tab:perdim_summary} shows where DISCA helps using one consistent
metric: mean absolute component error in percentage points, computed from the
same scenario rows for vanilla, Full DISCA, and consensus-only aggregation.
The lower block uses the same rows to identify the worst residual dimension in
each country. The headline MIS remains a separate country-level Euclidean norm
on the normalized $[0,1]$ vectors.

The component gains are uneven rather than universal. Full DISCA improves the
largest residuals in Social Value, Species, and Utilitarianism on both saved
grids, while Age worsens on both backbones in the macro component average.
The strict country-level win counts are 82/120 and 68/120, respectively, so
the table does not support a claim that every dimension or every country
improves. Consensus-only is strong on some individual dimensions, especially
on 14B Utilitarianism, but without the country prior its country-level tail is worse (Table~\ref{tab:tail_safety}).

Four structural insights emerge from these tables.

First, \textbf{the residual bottleneck is backbone-dependent}. Utilitarianism
is the most frequent worst dimension for Qwen2.5-7B, whereas Age is the most
frequent worst dimension for Qwen2.5-14B. This is a diagnostic pattern, not a
general law about model size.

Second, \textbf{the gate and consensus solve different component problems}.
Consensus-only can have a lower mean error on one dimension while still
incurring a larger country-level tail, because MIS combines all six dimensions
and the gate acts on the refinement before that norm is formed.

Third, \textbf{the gains are geographically broad}. The per-country table shows improvements in both Western and non-Western settings, while keeping the regressions visible. This breadth argues against a result driven by one country or one regional cluster, but the uneven cells also show why the method should be used with model- and scenario-level monitoring.

Fourth, and least comfortably, the per-dimension picture resolves an apparent paradox: negative Pearson correlation can remain even when MIS improves. MIS measures distance in the six-dimensional AMCE vector, whereas $r$ measures its shape and ordering, so a correction can reduce residual magnitude without recovering every relative preference. Table~\ref{tab:perdim_headline} repeats the analysis on the seven headline backbones, from the same saved runs as Table~\ref{tab:main_macro_summary}. Almost every backbone under-weights the utilitarian attribute after DISCA, sparing the larger group less often than people do. On Llama-3.3-70B this single error carries nearly half of the squared residual, and replacing it with the human value turns all eight negative-$r$ countries positive. On Phi-3.5-mini it is also the largest error, but the model sits below the human value on four of the six attributes on average, and five countries stay negative after the same replacement. DISCA narrows the utilitarian error on six of the seven backbones without reversing its sign.

\begin{table}[H]\centering\scriptsize
\caption{Per-dimension audit of the seven headline backbones (seed-42 runs of Table~\ref{tab:main_macro_summary}; 20 countries each). ``Util.\ MAE'' is the mean absolute utilitarian error in percentage points; ``Util.\ below human'' counts countries in which DISCA spares the larger group less often than people do; ``Util.\ share'' is the mean fraction of the squared residual that comes from the utilitarian attribute; the last two columns replace only the utilitarian component with the human value.}
\label{tab:perdim_headline}
\setlength{\tabcolsep}{3pt}
\resizebox{\linewidth}{!}{%
\begin{tabular}{@{}lccccccc@{}}
\toprule
Backbone & Mean $r$ (van.\,$\to$\,DISCA) & Countries $r{<}0$ & Util.\ MAE (van.\,$\to$\,DISCA) & Util.\ below human & Util.\ share & Mean $r$, Util.\ fixed & Countries $r{<}0$, Util.\ fixed \\
\midrule
Llama-3.3-70B & $.00\to-.01$ & 8 & $59.4\to44.8$ & 20 & .47 & .40 & 0 \\
Magistral-Small-2509 & $.53\to.62$ & 0 & $15.5\to15.0$ & 19 & .20 & .64 & 0 \\
Phi-4 & $.08\to.56$ & 0 & $21.4\to10.7$ & 15 & .12 & .59 & 0 \\
Qwen3-VL-8B & $.28\to.22$ & 4 & $36.1\to19.5$ & 19 & .19 & .37 & 0 \\
Qwen2.5-7B & $.35\to.45$ & 3 & $25.4\to15.8$ & 20 & .21 & .56 & 0 \\
Phi-3.5-mini & $-.38\to-.35$ & 19 & $46.2\to39.1$ & 20 & .47 & .10 & 5 \\
Gemma-4-E2B & $.25\to.46$ & 2 & $21.3\to21.8$ & 20 & .22 & .48 & 0 \\
\bottomrule
\end{tabular}}
\end{table}

\section{Component Decomposition}
\label{app:decomposition}

This appendix holds the full component decomposition summarised in
\S\ref{sec:ablation}. Table~\ref{tab:decomposition} switches off one stage at a
time on the logged headline runs. Its \emph{Consensus + prior} row is the
strongest arm on three backbones, which is why we describe the panel and the
prior, not the gated refinement, as the source of the gain.

\begin{table}[H]
\caption{\textbf{What each stage of DISCA contributes.} Macro MIS on the
20-country MultiTP panel (lower is better), recomputed from the logged
per-scenario controller state of the seed-42 headline runs with stages
switched off; no model is re-run, and the \emph{Full DISCA} row reproduces
Table~\ref{tab:main_macro_summary}. \emph{Consensus} is the persona-mean gap
alone. \emph{Prior} is the country-level running prior. \emph{No gated
refinement} keeps the ESS anchor blend and the prior and drops $\delta^\star$.
\emph{Ungated refinement} sets $r{=}1$. Bold marks the lowest value per
backbone.}
\label{tab:decomposition}
\centering\scriptsize
\setlength{\tabcolsep}{4pt}
\resizebox{\textwidth}{!}{%
\begin{tabular}{@{}lccccccc@{}}
\toprule
Arm & Llama-3.3-70B & Magistral (24B) & Phi-4 (14B) & Qwen3-VL-8B & Qwen2.5-7B & Phi-3.5-mini & Gemma-4-E2B \\
\midrule
Vanilla & .849 & .403 & .454 & .575 & .453 & .638 & .495 \\
Consensus & .794 & .330 & .378 & .564 & .373 & .649 & .508 \\
Consensus + prior & \textbf{.643} & .350 & .344 & .465 & \textbf{.338} & \textbf{.562} & .480 \\
\midrule
DISCA without prior & .830 & \textbf{.302} & .359 & .549 & .389 & .668 & .498 \\
DISCA, ungated refinement & .670 & .351 & .347 & .469 & .364 & .572 & .481 \\
DISCA, no gated refinement & .665 & .349 & \textbf{.344} & \textbf{.464} & .359 & .570 & \textbf{.474} \\
Full DISCA & .668 & .350 & .346 & .466 & .362 & .571 & .478 \\
\bottomrule
\end{tabular}}
\end{table}

\paragraph{How the replay works.} The headline runs log, for every scenario,
the persona consensus $\bar\delta$, the two half-pass refinements
$\delta^\star_1,\delta^\star_2$, the gate $r$, the scenario score
$\delta_{\mathrm{micro}}$ of Eq.~\ref{eq:disca-micro}, the running country score
$d_c$, and the prior weight $a_t$. Running the recursion of
Eq.~\ref{eq:disca-final-main} over the logged scores in their processing order
reproduces the logged decision probabilities to $10^{-7}$ and every published
country MIS to $3\times10^{-16}$. Each arm of Table~\ref{tab:decomposition}
then replaces one quantity before the same recursion: $\bar\delta$ for the
consensus rows, $\delta_{\mathrm{micro}}-\delta^\star$ for the row without the
gated refinement, and $\delta_{\mathrm{micro}}-\delta^\star+
(\delta^\star_1+\delta^\star_2)/2$ for the ungated row. The same replay on the
three-seed Qwen2.5-7B and Qwen2.5-14B grids of App.~\ref{app:gate_diagnostics}
gives the same ordering: consensus with the prior reaches $.3530$ and $.3903$
against $.3563$ and $.3950$ for Full DISCA.

\paragraph{Calibration sensitivity.} Because MIS compares average sparing
rates, it rewards decisions that are less extreme than the base model's. To
measure how much of a gain a calibration step alone could produce, we divide
the vanilla gap by a single temperature $T$ before the sigmoid
(Table~\ref{tab:calibration_sweep}). The best $T$ is chosen with the human
targets, so this is an oracle diagnostic, not a deployable control. It still
matters for reading the results: at $T{=}8$ vanilla beats Full DISCA on four
of the seven backbones. Part of what DISCA gains on MIS can therefore be
reached by making the base model less confident, and a calibration control
chosen without the targets belongs next to DISCA in any future comparison.

\begin{table}[H]\centering\small
\caption{Vanilla with its decision gap divided by a temperature $T$ (macro MIS, lower is better, seed-42 headline runs). $T{=}1$ is plain vanilla. Choosing $T$ uses the human targets, so the rows are an oracle diagnostic.}
\label{tab:calibration_sweep}
\setlength{\tabcolsep}{3.5pt}
\begin{tabular}{@{}lrrrrrrr@{}}
\toprule
& Llama-70B & Magistral & Phi-4 & Qwen3-VL & Qwen2.5-7B & Phi-3.5 & Gemma-E2B \\
\midrule
Vanilla, $T{=}1$ & .849 & .403 & .454 & .575 & .453 & .638 & .495 \\
$T{=}2$ & .840 & .359 & .438 & .552 & .440 & .621 & .477 \\
$T{=}3$ & .826 & .348 & .420 & .527 & .422 & .601 & .460 \\
$T{=}5$ & .788 & .350 & .393 & .478 & .389 & .563 & .436 \\
$T{=}8$ & .720 & .363 & .381 & .426 & .359 & .520 & .425 \\
\midrule
Full DISCA & .668 & .350 & .346 & .466 & .362 & .571 & .478 \\
\bottomrule
\end{tabular}
\end{table}

\paragraph{A reliability-weighted aggregation baseline.} \citet{yuan2025probabilistic}
combine the moral judgments of several LLMs by estimating a reliability for
each judge with expectation maximisation and weighting judges by it. We apply
the training-free half of that idea to the four personas of one model: within
a country, each persona's sparing probability is the latent consensus plus
Gaussian noise with a persona-specific variance, and the aggregate weights
personas by inverse variance (their embedding-tuning step needs gradients and
is outside our setting). On the two fully logged Qwen grids, this aggregate
reaches macro MIS $.3844$ (7B) and $.4281$ (14B) without the prior, behind Full
DISCA ($.3563$, $.3950$), and $.3500$ and $.3839$ with the same country prior,
ahead of it. The estimated weights often collapse onto one persona, because
personas of the same country frequently give near-identical gaps, so the
estimate is unstable at $N{=}4$; the comparison nonetheless repeats the
pattern of Table~\ref{tab:decomposition}: matched on the prior, a plain
aggregate of the same persona signals does at least as well as the controller.

\paragraph{Gate switched off on the headline runs.}
\label{app:component_ablation}
Setting $r{=}1$ while keeping every other stage raises macro MIS on all
seven backbones, by $.0005$ to $.0028$
(Table~\ref{tab:decomposition}, row \emph{DISCA, ungated refinement}). The gate
therefore helps, but by a small amount, because the refinement it controls is
small.

\section{Post-Hoc Diagnostics}
\label{app:r2_reliability_audit}
\label{app:r2_logit_conditioning}
\label{app:rank_agreement}

This appendix collects diagnostics recomputed from the logged per-scenario controller state, with no new model call: seeds and tail behaviour (\S\ref{sec:multiseed}), the response of the gate to panel disagreement with paired country tests (\S\ref{sec:scenario_disagreement_correction}), same-panel aggregation controls, profile specificity, and a consensus check of the panels. The shape-versus-amplitude diagnosis of negative Pearson $r$ is in App.~\ref{app:perdim}.

\subsection{Stability and tail safety}
\label{sec:multiseed}

\paragraph{Seeds.} The seven headline runs of Table~\ref{tab:main_macro_summary}
use controller seed 42, and only these seed-42 runs are reported for them. The
two Qwen2.5 backbones of the diagnostics below were run with three controller
seeds $\{42,101,2026\}$; on these grids the seed-to-seed standard deviation of
macro MIS is about $.0002$, far below any difference discussed in the paper.

\paragraph{Tail safety and scope.}
\label{sec:tail_safety}

The tail comparison is recomputed from the same six saved scenario logs (two
backbones, three controller seeds) used for the gate diagnostics. It covers the two complete saved grids,
Qwen2.5-7B and Qwen2.5-14B, with three controller seeds and 20 countries per
backbone, for 40 country-model cells in total. Replacing Full DISCA with
consensus only removes the PT-IS refinement while keeping the same persona
logits and scenario rows. Against that arm, Full DISCA has the larger mean
MIS reduction and the smaller worst degradation and cell-wise spread
(Table~\ref{tab:tail_safety}). That comparison, however, differs in two
stages at once, because the consensus-only arm also has no country prior. The
lower block matches the prior. Consensus with the prior has the largest mean
reduction and hurts one more cell; the controller without the gated refinement
matches Full DISCA on every tail statistic. The tail advantage of the
controller over consensus therefore comes from the ESS anchor blend, and the
gated refinement adds nothing measurable to it. This is a descriptive
comparison on two backbones.

\begin{table}[H]
\centering\small
\caption{Tail-risk comparison from one saved diagnostic artefact. The table
covers 40 country-model cells: two backbones, three controller seeds, and 20
countries. Country-level MIS is computed after averaging controller seeds;
positive $\Delta\mathrm{MIS}$ means improvement over vanilla. ``Cells hurt''
counts negative improvements; the final two columns are absolute worst negative
improvement and standard deviation across the same 40 cells.
The table is deliberately restricted to the two backbones with complete
matched Full DISCA and consensus-only rows. The lower block is recomputed from
the same logged rows with the country prior replayed (App.~\ref{app:decomposition}).}
\label{tab:tail_safety}
\resizebox{\linewidth}{!}{%
\begin{tabular}{lcccc}
\toprule
Variant & Mean $\Delta$MIS $\uparrow$ & Cells hurt $\downarrow$ & Worst-case degradation $\downarrow$ & Std across cells $\downarrow$ \\
\midrule
Full DISCA      & $0.0828$ & $\mathbf{4/40}$  & $0.0648$ & $\mathbf{0.0578}$ \\
Consensus only, no prior & $0.0450$ & $13/40$ & $0.1428$ & $0.0842$ \\
\midrule
Consensus + prior & $\mathbf{0.0869}$ & $5/40$ & $0.0872$ & $0.0738$ \\
No gated refinement, with prior & $0.0827$ & $\mathbf{4/40}$ & $\mathbf{0.0646}$ & $0.0580$ \\
DISCA without prior & $0.0508$ & $11/40$ & $0.1273$ & $0.0756$ \\
\bottomrule
\end{tabular}
}
\end{table}

\subsection{Gate diagnostics and paired country tests on the saved scenario grids}
\label{app:gate_diagnostics}
\label{sec:scenario_disagreement_correction}

This section reports the mechanism diagnostics summarised in
\S\ref{sec:gate_diagnostics}. The analysis covers the two backbones with
complete saved scenario grids, Qwen2.5-7B and Qwen2.5-14B: three controller
seeds $\{42,101,2026\}$, 20 countries, six dimensions, and 310 scenarios per
country, giving $18{,}600$ scenario rows per backbone and $37{,}200$ in total.
\textbf{Why these two backbones.} Qwen2.5-7B and Qwen2.5-14B are the only
backbones for which complete three-seed scenario grids with every internal
controller quantity were saved; Qwen2.5-14B is not in the headline suite
because no headline-protocol run exists for it. The diagnostics use them for
that reason, not because they were selected on performance.
\textbf{Provenance boundary.} The source archive records seeds, countries, and
scenario identifiers but not the model revision, so these numbers are
diagnostics on the saved grids and are not described as revision-controlled
results. Their absolute values are not substituted for the headline table.

\paragraph{Definitions.}
For each scenario we use the four logged persona decision gaps $g_1,\dots,g_4$
and define panel disagreement as their sample variance,
$D^2=\mathrm{Var}_{\mathrm{sample}}(g_1,g_2,g_3,g_4)$. The retained refinement
is reconstructed from the two logged half-budget estimates and the gate as
$\delta^\star = r\,(\delta^\star_1+\delta^\star_2)/2$. Its size relative to the
persona-mean gap $\bar\delta$ is $q=|\delta^\star|/|\bar\delta|$, left
undefined when the denominator is zero. The vanilla model preference rate (MPR)
is the probability assigned to the human-preferred side for each country and
dimension; the human MPR is reconstructed from the paired vanilla MPR and the
logged target offset. MIS is the Euclidean distance between the
six-dimensional model and human MPR vectors, as in the main text.

\paragraph{Gate behaviour.}
Table~\ref{tab:gate_diag_summary} gives the typical gate values and their
association with $D^2$; Table~\ref{tab:gate_diag_full} gives the full
distribution audit; Table~\ref{tab:gate_diag_corr} gives both correlation
coefficients with $p$-values. The reliability weight $r$ is high on the typical
scenario and decreases as disagreement grows, on both backbones. The
effective-support anchor blend $\alpha_{\mathrm{ess}}$ does not follow one
common monotone pattern. The ratio $q$ has a heavy upper tail caused by
scenarios whose persona-mean gap $\bar\delta$ is near zero; the median and 95th percentile
describe the typical scenario, and the 99.5th percentile is reported so the
tail is visible.
\begin{figure*}[t]
  \centering
  \includegraphics[width=\textwidth]{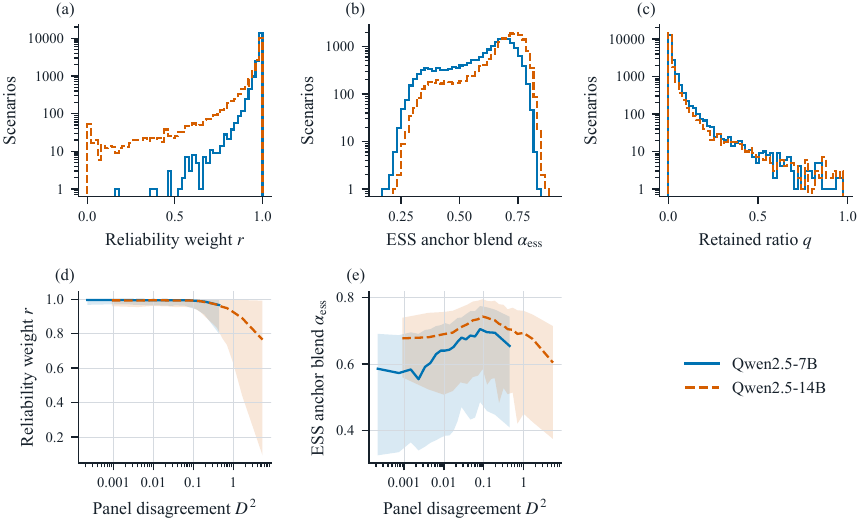}
  \caption{Gate behaviour on the two complete saved scenario grids
  ($18{,}600$ scenario rows per backbone). (a) to (c) Distributions, on a log
  count scale, of the reliability weight $r$, the ESS anchor blend
  $\alpha_{\mathrm{ess}}$ and the retained refinement ratio
  $q=|\delta^\star|/|\bar\delta|$; the $q$ histograms are
  clipped at each backbone's 99.5th percentile for readability. (d), (e) $r$
  and $\alpha_{\mathrm{ess}}$ against panel disagreement $D^2$: lines are
  medians over 20 equal-count bins of $D^2$ and bands span the 10th to 90th
  percentiles. Summary statistics and correlations use all values.}
  \label{fig:gate_diag}
\end{figure*}

\begin{table}[H]
\centering\small
\caption{Typical gate values and their association with panel disagreement.
$q$ is the retained refinement divided by the persona-mean gap $\bar\delta$. Correlations are
Spearman $\rho$ over all scenario rows; Table~\ref{tab:gate_diag_corr} gives
their country-clustered uncertainty.}
\label{tab:gate_diag_summary}
\setlength{\tabcolsep}{4pt}
\begin{tabular}{lrrrrrr}
\toprule
Backbone & median $r$ & median $\alpha_{\mathrm{ess}}$ & median $q$ & $q_{95}$ & $\rho(D^2,r)$ & $\rho(D^2,\alpha_{\mathrm{ess}})$\\
\midrule
Qwen2.5-7B  & 0.994 & 0.645 & 0.0101 & 0.1300 & $-0.160$ & $0.337$\\
Qwen2.5-14B & 0.985 & 0.706 & 0.0054 & 0.1018 & $-0.454$ & $-0.077$\\
\bottomrule
\end{tabular}
\end{table}

\begin{table}[H]
\centering\scriptsize
\caption{Full distribution audit for the gate quantities and panel
disagreement, over all $18{,}600$ scenario rows per backbone.}
\label{tab:gate_diag_full}
\setlength{\tabcolsep}{3.5pt}
\begin{tabular}{llrrrrrrrr}
\toprule
Quantity & Backbone & Mean & SD & $q_{05}$ & $q_{25}$ & Median & $q_{75}$ & $q_{95}$ & $q_{99.5}$\\
\midrule
$r$ & Qwen2.5-7B & 0.9825 & 0.0343 & 0.9312 & 0.9803 & 0.9936 & 0.9986 & 0.9999 & 1.0000\\
$r$ & Qwen2.5-14B & 0.9366 & 0.1370 & 0.6918 & 0.9467 & 0.9848 & 0.9969 & 0.9999 & 1.0000\\
$\alpha_{\mathrm{ess}}$ & Qwen2.5-7B & 0.6033 & 0.1299 & 0.3422 & 0.5221 & 0.6448 & 0.7023 & 0.7583 & 0.7948\\
$\alpha_{\mathrm{ess}}$ & Qwen2.5-14B & 0.6732 & 0.1110 & 0.4090 & 0.6407 & 0.7061 & 0.7479 & 0.7892 & 0.8180\\
$q$ & Qwen2.5-7B & 0.0490 & 0.4428 & 0.0007 & 0.0037 & 0.0101 & 0.0272 & 0.1300 & 0.9690\\
$q$ & Qwen2.5-14B & 0.0384 & 0.3228 & 0.0003 & 0.0020 & 0.0054 & 0.0163 & 0.1018 & 0.9760\\
$D^2$ & Qwen2.5-7B & 0.0631 & 0.1705 & 0.0005 & 0.0039 & 0.0156 & 0.0529 & 0.2779 & 1.0449\\
$D^2$ & Qwen2.5-14B & 0.6504 & 2.0737 & 0.0025 & 0.0328 & 0.1362 & 0.4529 & 2.5942 & 14.1955\\
\bottomrule
\end{tabular}
\end{table}

\begin{table}[H]
\centering\small
\caption{Association between panel disagreement $D^2$ and the gate quantities
with country clustering. Rows are the $18{,}600$ scenario rows per backbone,
but rows of one country share a panel and a seed, so uncertainty comes from a
country-cluster bootstrap (2{,}000 resamples of the 20 countries, keeping all
seeds and scenarios of a drawn country). ``Cluster CI'' is the 95\%
percentile interval of the row-level Spearman $\rho$, and $p$ is the
two-sided bootstrap sign test (smallest attainable value $0.001$). The last
column correlates the 20 country means.}
\label{tab:gate_diag_corr}
\setlength{\tabcolsep}{4pt}
\resizebox{\linewidth}{!}{%
\begin{tabular}{llrrlrr}
\toprule
Backbone & Gate quantity & Pearson $r$ & Spearman $\rho$ & Cluster CI for $\rho$ & Cluster $p$ & Country-level $\rho$\\
\midrule
Qwen2.5-7B & $r$ & $-0.387$ & $-0.160$ & $[-0.208,-0.112]$ & $0.001$ & $-0.862$\\
Qwen2.5-7B & $\alpha_{\mathrm{ess}}$ & $0.059$ & $0.337$ & $[0.282,0.390]$ & $0.001$ & $0.014$\\
Qwen2.5-14B & $r$ & $-0.524$ & $-0.454$ & $[-0.515,-0.378]$ & $0.001$ & $-0.986$\\
Qwen2.5-14B & $\alpha_{\mathrm{ess}}$ & $-0.217$ & $-0.077$ & $[-0.170,0.037]$ & $0.215$ & $-0.705$\\
\bottomrule
\end{tabular}}
\end{table}

\paragraph{Disagreement does not locate vanilla error.}
For each country and dimension we correlate mean panel disagreement with the
absolute distance between the vanilla MPR and the reconstructed human target
(Table~\ref{tab:gate_diag_errorloc}). Pooled within a backbone, the
association is absent: Pearson $r=0.064$ ($p=0.489$) for Qwen2.5-7B and
$r=0.020$ ($p=0.826$) for Qwen2.5-14B; Spearman $\rho=0.056$ ($p=0.546$) and
$0.170$ ($p=0.064$). A few individual dimensions show an uncorrected
association (Utilitarianism on Qwen2.5-7B, Social value on Qwen2.5-14B) that
does not survive as a backbone-level pattern. The supported reading is that
shrinkage limits the effect of an uncertain panel without the panel's variance
identifying where the base model is wrong.

\begin{table}[H]
\centering\scriptsize
\caption{Country-level correlation between mean panel disagreement and the
vanilla model's distance from the human target, 20 countries per row. The
$p$-values are descriptive and not adjusted across the 12 tests.}
\label{tab:gate_diag_errorloc}
\begin{tabular}{llrrrr}
\toprule
Backbone & Dimension & Pearson $r$ & Pearson $p$ & Spearman $\rho$ & Spearman $p$\\
\midrule
Qwen2.5-7B & Age & 0.011 & 0.964 & 0.263 & 0.262\\
 & Fitness & $-0.106$ & 0.656 & 0.066 & 0.782\\
 & Gender & 0.060 & 0.802 & 0.099 & 0.677\\
 & Social value & 0.027 & 0.910 & 0.065 & 0.787\\
 & Species & 0.057 & 0.812 & 0.211 & 0.373\\
 & Utilitarianism & $-0.492$ & 0.028 & $-0.311$ & 0.182\\
\midrule
Qwen2.5-14B & Age & 0.005 & 0.985 & 0.254 & 0.280\\
 & Fitness & $-0.112$ & 0.638 & $-0.035$ & 0.885\\
 & Gender & $-0.104$ & 0.664 & 0.251 & 0.286\\
 & Social value & $-0.308$ & 0.187 & $-0.531$ & 0.016\\
 & Species & 0.443 & 0.051 & 0.371 & 0.107\\
 & Utilitarianism & $-0.089$ & 0.709 & $-0.374$ & 0.104\\
\bottomrule
\end{tabular}
\end{table}

\paragraph{Paired tests across countries.}
For each backbone we average the three controller seeds within each country
and test the 20 paired country differences in MIS with a one-sided Wilcoxon
signed-rank test, positive meaning vanilla MIS exceeds DISCA MIS
(Table~\ref{tab:gate_diag_paired}). The bootstrap resamples the 20 countries
with replacement, then the 310 scenarios within each country and dimension with
replacement, applying the same draws to both methods before seed averaging;
5{,}000 replicates, seed 20260915, percentile interval from the 2.5th to the
97.5th percentile. All 5{,}000 replicates have a positive mean difference.
Table~\ref{tab:gate_diag_countries} exposes the per-country pairing; the two
negative entries per backbone are the only countries where DISCA has larger MIS
than vanilla in this analysis. Throughout this analysis the vanilla reference is a separately
decoded vanilla run on the same scenarios (macro MIS .4488 for 7B and .4682 for
14B). The saved grids also hold a vanilla gap decoded in the same pass as the
personas, from the same logits as the same-logit audit of
Table~\ref{tab:same_logit_verified}, and it has lower MIS (.4231 and .4032). Against that reference, DISCA wins 15 of 20 countries on
Qwen2.5-7B (mean $\Delta\mathrm{MIS}$ .0668, one-sided Wilcoxon $p=.0086$) and
10 of 20 on Qwen2.5-14B (mean $\Delta\mathrm{MIS}$ .0082, $p=.39$), so the
country-level advantage on 14B depends on which vanilla run is the reference.

\begin{table}[H]
\centering\small
\caption{Paired country tests against the separately decoded vanilla run
(macro MIS .4488 and .4682). Against the vanilla gap decoded with the personas,
the 14B result is 10 wins of 20 (see text). A positive difference favours DISCA.}
\label{tab:gate_diag_paired}
\setlength{\tabcolsep}{4pt}
\begin{tabular}{lrrrrrr}
\toprule
Backbone & Wins & Losses & Mean $\Delta\mathrm{MIS}$ & Median $\Delta\mathrm{MIS}$ & Wilcoxon $p$ & Bootstrap 95\% CI\\
\midrule
Qwen2.5-7B  & 18 & 2 & 0.0925 & 0.0982 & $1.81\times10^{-5}$ & $[0.0762,\,0.1447]$\\
Qwen2.5-14B & 18 & 2 & 0.0731 & 0.0786 & $6.68\times10^{-6}$ & $[0.0399,\,0.0984]$\\
\bottomrule
\end{tabular}
\end{table}

\begin{table}[H]
\centering\small
\caption{Country-level MIS differences (vanilla minus DISCA) after averaging
the three controller seeds. Positive values favour DISCA. ISO alpha-3 codes.}
\label{tab:gate_diag_countries}
\begin{tabular}{lrr|lrr}
\toprule
Country & Qwen2.5-7B & Qwen2.5-14B & Country & Qwen2.5-7B & Qwen2.5-14B\\
\midrule
ARG & 0.0764 & 0.0535 & MEX & 0.0910 & 0.0061\\
BGD & 0.1731 & 0.1351 & MMR & 0.1450 & 0.1489\\
BRA & 0.0585 & $-0.0231$ & MYS & 0.1512 & 0.1289\\
CHN & 0.0708 & $-0.0051$ & ROU & 0.1442 & 0.0947\\
COL & 0.0775 & 0.0364 & SRB & 0.1473 & 0.0997\\
DEU & $-0.0110$ & 0.0253 & THA & 0.1586 & 0.1291\\
ETH & 0.1458 & 0.0388 & USA & 0.1055 & 0.1097\\
GBR & 0.1058 & 0.0861 & VNM & 0.0203 & 0.0711\\
IDN & 0.0290 & 0.0481 & IRN & $-0.0648$ & 0.0652\\
JPN & 0.0579 & 0.1006 & KGZ & 0.1684 & 0.1135\\
\bottomrule
\end{tabular}
\end{table}

\subsection{Same-panel aggregation controls}
\label{app:aggregation_controls}

The same-logit controls of \S\ref{sec:same_logit_main} reuse the logged vanilla and four-persona logits
of the two three-seed Qwen2.5 grids, so every comparator has exactly the same
model calls as DISCA. Their values are in Table~\ref{tab:same_logit_verified}
of the main text. They are a separate evidence family from the headline
runs and are not substituted for them.

Full DISCA is lower than each same-logit comparator on both audited backbones.
These comparators receive no country prior, however. Given the same prior,
logit consensus reaches $.3530$ and $.3903$ on the same two backbones, and the
inverse-variance aggregate of App.~\ref{app:decomposition} reaches $.3500$ and
$.3839$, both slightly below Full DISCA ($.3563$ and $.3950$). This audit
therefore does not show an aggregation advantage beyond the prior.

Three Qwen2.5-7B values of Full DISCA appear in this paper, and each stays in
its own evidence family: $.362$ is the headline run, $.3563$ is the three-seed
mean of the saved diagnostic grids, and $.3561$ is the seed-42 run of those
grids used as the matched row of the profile-swap audit
(Table~\ref{tab:profile_swap_verified}). None is substituted for another.

\subsection{Country-profile specificity}
\label{app:profile_specificity}

The profile-swap audit asks whether the controller uses the identity of the
country profile it is given. It holds the backbone, the controller, the
20-country MultiTP slice and controller seed 42 fixed, and replaces each target
country's WVS profile with either a random profile from another country (no
country keeps its own) or the most distant available profile. The matched
column is the seed-42 run of the same saved grid used in
App.~\ref{app:gate_diagnostics}.

\begin{table}[H]\centering\small
\caption{Profile-swap audit (controller seed 42, 20 countries, 310 scenarios
per country). Macro MIS, lower is better. The difference is wrong minus
matched, so a positive value means the matched profile is better; the interval
is a paired bootstrap over the 20 countries (20{,}000 draws). ``Matched better''
counts countries where the matched profile has the lower MIS.}
\label{tab:profile_swap_verified}
\setlength{\tabcolsep}{4pt}
\resizebox{\linewidth}{!}{%
\begin{tabular}{llrrrr}
\toprule
Backbone & Profile given to DISCA & Macro MIS $\downarrow$ & Wrong $-$ matched & 95\% interval & Matched better \\
\midrule
Qwen2.5-7B & Matched country & \textbf{.3561} & & & \\
 & Random other country & .3571 & $+.0010$ & $[-.0039,\,.0059]$ & 11/20 \\
 & Most distant country & .3589 & $+.0028$ & $[-.0018,\,.0074]$ & 11/20 \\
\midrule
Qwen2.5-14B & Matched country & \textbf{.3949} & & & \\
 & Random other country & .4124 & $+.0176$ & $[-.0002,\,.0357]$ & 12/20 \\
 & Most distant country & .4162 & $+.0213$ & $[.0010,\,.0442]$ & 10/20 \\
\bottomrule
\end{tabular}}
\end{table}

All four wrong-profile conditions move macro MIS in the worse direction. The
effect is small on Qwen2.5-7B, where both intervals include zero, and larger
on Qwen2.5-14B, where the most distant profile has an interval above zero. The
country-level picture is mixed: the matched profile is better in only 10 to 12
of the 20 countries. The audit therefore supports a modest, backbone-dependent
use of the profile identity. It does not show that four personas represent
every member of a country.

A second check points the same way. A permutation placebo scores each
country's DISCA preference vector against the human vector of a different,
randomly assigned country (1,000 random reassignments per backbone and
controller seed, on the same two Qwen2.5 grids) and asks whether the true
assignment gives a lower macro MIS than the reassigned ones. It does not: the
lower-is-better permutation $p$-values are large, about $0.75$ to $1.00$, and
vanilla decoding behaves the same way. The MultiTP gain alone therefore does not
establish that each country's profile is specific to that country. What the
evidence supports is that the content of the profiles is relevant, not that
each profile is a faithful portrait of its country.

\subsection{Panel consensus under Cultural Consensus Theory}
\label{app:cct_panel}

A four-persona panel could fail in two opposite ways. It could split into
unrelated caricatures, or it could collapse into one voice and hide the
within-country disagreement the controller reads. Cultural Consensus Theory
(CCT) gives a standard check for the first failure~\citep{pothugunta2026cct}.
We treat the four personas of a country as CCT informants who answer its 310
MultiTP scenarios (the side their order-symmetrised gap favours), build the
chance-corrected agreement matrix, and take the ratio of its first to second
eigenvalue. CCT reads a ratio above 3 as evidence of one shared cultural
model. We measure the second failure by the share of scenarios on which the
four personas do not all choose the same side.

\begin{table}[H]\centering\small
\caption{CCT diagnostics of the four-persona panels on the 20-country MultiTP
panel (seed-42 scenario grids). ``Single culture'' counts countries whose
first-to-second eigenvalue ratio exceeds 3.}
\label{tab:cct_panel}
\setlength{\tabcolsep}{4pt}
\resizebox{\linewidth}{!}{%
\begin{tabular}{@{}lcccc@{}}
\toprule
Backbone & Median eigenvalue ratio & Single culture & Scenarios not unanimous (mean, range) & Pairwise agreement \\
\midrule
Qwen2.5-7B  & 23.3 & 20/20 & 8.9\% (1.0 to 24.2\%) & .952 \\
Qwen2.5-14B & 22.2 & 20/20 & 10.9\% (0.6 to 32.9\%) & .940 \\
\bottomrule
\end{tabular}}
\end{table}

Every panel passes the CCT test by a wide margin, so the four personas share
one coherent consensus rather than four unrelated views. The panels do not
collapse either. On average one scenario in ten splits the personas, and in
some countries one in three does. The disagreement the controller reads is
therefore concentrated on a minority of scenarios.

\section{Baseline Implementation Details}
\label{app:r2_baselines}
\label{app:baselines}

This appendix describes every control of Table~\ref{tab:main_baseline_sanity}.
All controls run on Phi-4 over the same 20 countries and the same 310
scenarios per country as the headline run, use the same native-language
prompts and decision tokens, and never see the human AMCE vectors. They were
run as separate jobs from the DISCA run.

\subsection{Baseline implementations}
\label{app:baselines:sanity}
\label{app:baselines:args}
\label{app:baselines:additional}

\paragraph{Vanilla decoding.} The model sees the native-language scenario with
no persona prefix and no logit change. We read $\sigma(\delta(x)/T_{\text{dec}})$
and average over scenarios as in Eq.~\ref{eq:amce-estimate}.

\paragraph{WVS profile prompt.} One country-conditioned system prompt
concatenates the ten WVS descriptors of Table~\ref{tab:wvs_thresholds} for the
whole population (no cohort split) into one paragraph in the country's
language. This collapses the four-persona panel into one point and isolates
what persona \emph{diversity} adds on top of WVS grounding.

\paragraph{Persona-ensemble prompt and other prompt variants.}
A demographic persona-conditioning baseline in the spirit of pluralistic persona
prompting~\citep{santurkar2023opinions, sorensen2024roadmap}: the model answers
under six sociodemographic personas of the country, and the six sparing
probabilities are averaged with fixed weights; no WVS statistics and no human
AMCE enter the prompt. With long persona descriptions this baseline reaches .364
macro MIS on the Phi-4 grid, the strongest training-free baseline and still
behind DISCA (.346); with one-sentence descriptions it reaches .460. Two further
country-prompt variants land at .455 and .459, close to the WVS profile prompt
(.4525).

\paragraph{Fixed logit offset.} A training-free logit-steering control adds a
country-specific bias to the two decision-token logits. The bias is fixed
before evaluation, does not use the human targets, and is the same for every
scenario of that country, so the control can shift the country's average
preference but cannot respond to individual scenarios.

\paragraph{MC-Dropout calibration.} Following \citet{kwon2025dropouts}, who enable
attention dropout at inference (rates 0.05 and 0.1), we enable every dropout
module (rate $0.10$) and average the A/B probabilities of $8$ stochastic
forward passes per scenario. The control is country-agnostic, so its effect relative to vanilla
measures uncertainty inflation alone.

\paragraph{Controls that use the target are excluded.} Per-country temperature
or margin scaling fitted to the human target, and target-replay controls such
as an AMCE-mixing oracle, consume the evaluation target (DiffPO~\citep{chen2025diffpo}
instead trains a sentence-level refinement module; this oracle only borrows the
idea of moving outputs toward a target at inference time). Such controls do not fit the setting studied here. App.~\ref{app:decomposition}
reports a single label-chosen temperature on vanilla as an oracle diagnostic.

\paragraph{ARGS and controlled decoding (not run).} ARGS~\citep{khanov2024args}
and controlled decoding~\citep{mudgal2024controlled} need a trained reward
model, so a cross-cultural deployment would need one reward model per country,
which is the resource our setting rules out.

\subsection{Activation steering}
\label{app:baselines:activation}

Following the paired-prompt activation-addition recipe of
\citet{turner2023activation}, with difference-in-means extraction as in
\citet{arditi2024refusal} and \citet{zou2023representation}, a steering vector is computed for each country
as the difference of mean residual-stream activations between culturally
contrastive prompt pairs (for example ``Answer as someone from [country] who
values [WVS-high trait]'' against ``Answer as someone with the opposite
values''). The vector is extracted from the same Phi-4 model it steers and is
added with strength $\alpha{=}1.5$ to the residual stream at layer 32 during
inference. A stronger direction search
(multi-layer injection, refined prompt design) might improve this control; the
result in Table~\ref{tab:main_baseline_sanity} is consistent with a mismatch
between one fixed direction and a six-attribute calibration target that DISCA
adjusts per scenario.

\section{Hyperparameters and Sensitivity Sweeps}
\label{app:robustness_summary}
\label{app:hyperparams}
\label{app:temp_sensitivity}
\label{app:r2_hparam_sensitivity}

This appendix gives the full hyperparameter specification of DISCA and the sensitivity sweeps for its main choices. Persona count is reported separately because its shape, not its span, is the claim.

\subsection{Default Hyperparameters and Validation}

\begin{table}[H]
\caption{DISCA hyperparameters. The Prospect Theory parameters are the median estimates of \citet{tversky1992advances}. The same values are used for every backbone and country.}
\label{tab:hyperparams}
\centering\small
\setlength{\tabcolsep}{6pt}
\begin{tabularx}{\linewidth}{@{}>{\raggedright\arraybackslash}p{0.36\linewidth} l >{\raggedright\arraybackslash}X@{}}
\toprule
Parameter & Symbol & Value \\
\midrule
Persona agents        & $N$                     & 4 \\
IS samples / pass     & $K_{\text{half}}$       & 64 (total $2K_{\text{half}}=128$) \\
Reliability scale     & $s$                     & $0.04$ \\
Cooperation weight    & $\lambda_{\text{coop}}$ & $0.7$ \\
IS softmax temp.      & $\eta$                  & $0.5$ \\
ESS guard threshold   & $\rho_{\mathrm{eff}}$   & $0.1$ \\
PT curvature (gains and losses) & $\alpha$      & $0.88$ \\
PT loss aversion      & $\kappa$                & $2.25$ \\
Decision temp.        & $T_{\text{dec}}$        & $0.5$ \\
Per-attribute logit temp. (Species / Gender / other) & $T_c$ & $4.0$ / $3.5$ / $1.5$ \\
$T_c$ for a category without an entry & $T_c$   & $3.0$ \\
Proposal scale floor  & $\sigma_0$              & $0.3$ \\
ESS anchor blend      & $\alpha_{\mathrm{ess}}$ & on (off: $\bar{\delta}+\delta^{\star}$ only) \\
Country prior rate / warm-up / time scale & $\beta, t_0, \tau$ & $0.1$ / $50$ / $100$ \\
\bottomrule
\end{tabularx}
\end{table}

The controller hyperparameters were set during prototyping, on a panel of three backbones and five countries (BRA, CHN, DEU, JPN, USA), and then held constant for every backbone and country in the headline sweep. The persona count $N{=}4$ follows the cohort design (three age cohorts and a country aggregate), and the three-country sweep below supports it. The per-attribute temperatures $T_c$ follow the size of the raw logit gaps: Species and Gender produce systematically larger gaps and receive higher temperatures. Because the defaults were chosen during prototyping rather than by a held-out search, some residual bias is possible, so Table~\ref{tab:sensitivity_all} reports how much each choice moves the result.

\subsection{Sensitivity Sweeps}

Six sweeps probe the principal controller choices, and Table~\ref{tab:sensitivity_all} collects them in one place. The controller axes govern distinct stages of the pipeline: $s$ sets the steepness of the reliability gate (\S\ref{sec:method:estimator:gate}), $\lambda_{\text{coop}}$ trades individual-persona utility against the consensus utility, and $\sigma_0$ sets the floor of the proposal scale. We also sweep the overall scale of the per-attribute temperatures $T_c$, the IS budget, and a per-attribute versus global $T_c$ schedule, because a deployment would want to change them.

\begin{table}[H]\centering\scriptsize
\caption{\textbf{Sensitivity sweeps.} Macro MIS on a three-country panel (USA, VNM, DEU) with Phi-4, 310 scenarios per country. ``Span'' is max $-$ min over the grid; ``$\Delta$(def.\ $-$ best)'' is how much worse the default is than the best grid point, where it was recorded. Each row is a separate run, so minima and maxima are read within a row, not across rows. Persona count is in Table~\ref{tab:r3_persona_count}.}
\label{tab:sensitivity_all}
\setlength{\tabcolsep}{3pt}
\resizebox{\textwidth}{!}{%
\begin{tabular}{@{}lllllrrrr@{}}\toprule
Axis & Panel / backbone & Metric & Default & Grid & Min & Max & Span & $\Delta$(def.\ $-$ best) \\\midrule
$s$ (gate scale) & USA/VNM/DEU, Phi-4 & MIS & $0.04$ & $\{0.01,0.02,0.04,0.08,0.16\}$ & 0.4091 & 0.4164 & 0.0073 & - \\
$\lambda_{\text{coop}}$ & USA/VNM/DEU, Phi-4 & MIS & $0.70$ & $\{0.30,0.50,0.70,0.85,0.95\}$ & 0.4085 & 0.4173 & 0.0088 & - \\
$\sigma_0$ (proposal floor) & USA/VNM/DEU, Phi-4 & MIS & $0.30$ & $\{0.10,0.20,0.30,0.45,0.60\}$ & 0.4084 & 0.4122 & 0.0037 & - \\
$T_c$ scale & USA/VNM/DEU, Phi-4 & MIS & $1.0\times$ & $\{0.33,0.67,1.0,1.33,1.67\}\times$ & 0.4060 & 0.4149 & 0.0090 & - \\\midrule
IS budget $K_{\text{half}}$ & USA/VNM/DEU, Phi-4 & MIS & $64$ & $\{8,16,32,64,128,192\}$ & 0.4092 & 0.4147 & 0.0055 & $+$0.0011 \\
$T_c$ schedule & USA/VNM/DEU, Phi-4 & MIS & per-cat.\ 4.0/3.5/1.5 & per-cat.\ vs.\ global $\{1.0\ldots4.0\}$ & 0.4098 & 0.4138 & 0.0040 & $0.0000$ \\
\bottomrule
\end{tabular}}
\end{table}

Every axis moves macro MIS by less than $.01$ over its grid, and the defaults are not always the best grid point. The defaults are therefore stable operating choices rather than a sharp optimum, and prompt construction and panel content remain the larger sources of uncertainty (App.~\ref{app:prompt_surface}).

\paragraph{Persona count.}
The one axis whose \emph{shape} rather than span carries the argument is panel size. Sweeping $N\in\{2,3,4,5,6\}$ on the same USA/VNM/DEU panel (Table~\ref{tab:r3_persona_count}), $N{=}4$ is optimal within this sweep at macro MIS $0.4051$, with degradation at both smaller and larger panels ($0.4252$ at $N{=}2$, $0.4179$ at $N{=}6$). This absolute value is not ranked against the minima of the separate hyperparameter sweeps above, because those rows are independent validation executions. The non-monotonicity is the point: a larger panel is not simply a more expensive version of a smaller one on this audit. The theory motivates diminishing variance returns as $N$ grows, while the observed degradation at $N{=}5$ and $N{=}6$ is an empirical operating-point result, not a consequence of Corollary~\ref{cor:persona-count}.

\begin{table}[H]
\centering\small
\caption{Persona-count sweep on the USA/VNM/DEU panel (Phi-4, 310 scenarios per country). The four-persona default minimises macro MIS; larger panels add time without an alignment benefit. ``Flip rate'' is the share of scenarios whose decision changes between the two option orders before averaging.}
\label{tab:r3_persona_count}
\begin{tabular}{rcccc}\toprule
$N$ & Macro MIS $\downarrow$ & Std (across countries) & Flip rate & sec/scenario \\\midrule
2 & 0.4252 & 0.1045 & 0.275 & 0.072 \\
3 & 0.4194 & 0.0945 & 0.273 & 0.080 \\
\textbf{4} & \textbf{0.4051} & 0.0829 & 0.276 & 0.090 \\
5 & 0.4107 & 0.0745 & 0.276 & 0.097 \\
6 & 0.4179 & 0.0919 & 0.245 & 0.105 \\
\bottomrule
\end{tabular}
\end{table}

\section{Prompt Surface and Descriptor Order}
\label{app:prompt_surface}

This appendix reports the prompt-form evidence without extending it
beyond its tested scope. The completed audit uses Qwen2.5-7B, seed 42, all 310
MultiTP scenarios, and five countries: Brazil, China, Germany, Japan, and the
United States. The WVS descriptor values and controller are fixed. One factor
changes the instruction surface; a second changes only the order of the ten
descriptors.

\begin{table}[H]\centering\small
\caption{Partial five-country prompt audit. The audit is a sensitivity
diagnostic, not a full multilingual robustness result. Lower macro MIS is
better.}
\label{tab:prompt_partial}
\setlength{\tabcolsep}{4pt}
\begin{tabular}{@{}llrr@{}}
\toprule
Factor & Condition & Macro MIS $\downarrow$ & $\Delta$ vs. canonical \\
\midrule
Instruction surface & Current prose & .420339 & 0 \\
Instruction surface & Structured list & .444867 & +.024528 \\
Instruction surface & First-person narrative & \textbf{.407841} & -.012499 \\
\midrule
Descriptor order & Canonical & .420339 & 0 \\
Descriptor order & Reversed & \textbf{.417668} & -.002671 \\
Descriptor order & Seed-42 permutation & .425759 & +.005420 \\
\bottomrule
\end{tabular}
\end{table}

The instruction-surface span is $0.0370$ MIS, compared with $0.0081$ for
descriptor order. With five countries, no paired test is significant, before
or after correction for multiple tests. The result therefore identifies wording
as a material sensitivity, while descriptor order appears more
stable on this slice. It does not establish universal prompt robustness or a
multilingual register effect.

\paragraph{Persona enumeration order.} A separate end-to-end audit keeps the
same four personas and changes only their enumeration order. Fresh Qwen2.5-7B
BF16 forwards on all 310 USA scenarios at seed 42 produce zero difference in
the optimised gap, final correction, final probability, or hard decision at a
$10^{-6}$ tolerance. This establishes implementation invariance to enumeration
of the same panel on this audit; it does not imply invariance to wording,
descriptor content, or persona construction.

\section{Persona Construction Details}
\label{app:personas}

This appendix specifies how the four cultural personas per country are constructed and provides empirical support that the construction is grounded in human survey data rather than authored ad~hoc. \S\ref{app:examples} reproduces a generated persona verbatim; \S\ref{app:wvs_pipeline} documents the WVS Wave~7 processing pipeline; and \S\ref{app:wvs_linkage} discusses how the ten WVS cultural dimensions relate to the six MultiTP moral attributes.

\subsection{Persona Prompts}
\label{app:examples}

Each country yields four personas: three age-cohort personas (\emph{young}, \emph{middle}, \emph{older}) generated from the country-and-cohort-specific WVS-7 means, plus a population-wide aggregate that serves as a demographic anchor. Because every persona for a given country is rendered from the same ten-slot template and differs only in the descriptor levels the WVS means select, printing all four in full would repeat the same paragraph four times. We therefore print one United States persona verbatim and give the other three as the descriptor levels that differ.

\subsubsection*{United States, young cohort (verbatim)}
\begin{quote}
\small\itshape
You are a young adult from the United States, in your 20s and early 30s. Your worldview is shaped by the cultural values prevalent in your community. On matters of faith you are moderately religious. On raising children you are firmly oriented toward independence and imagination. On contested moral choices you are morally conservative on contested issues. In your dealings with strangers you have a guarded attitude toward strangers. Civically you are an active political participant who signs petitions, joins boycotts and takes part in lawful demonstrations. You are moderately proud of your country. Overall you are rather happy with your life. On the role of women in society you are moderately egalitarian on gender roles. In what you prioritise in life you are leaning materialist, prioritising economic and physical security. Toward people unlike yourself you are highly tolerant of outgroups such as immigrants, minorities and people with different lifestyles. When you face a moral dilemma, you weigh the choices through this set of values and answer in a way that is consistent with the worldview above.
\end{quote}

The other three United States personas are the identical paragraph with two clauses substituted, as Table~\ref{tab:usa_persona_deltas} records. All eight remaining descriptors, religiosity, child-rearing, social trust, political participation, happiness, gender equality, materialism, and tolerance, take the same level in all four, which is itself worth stating: on this country the four-persona panel's disagreement comes from a narrow part of the profile, and the panel consensus still moves the model on most backbones (Table~\ref{tab:decomposition}).

\begin{table}[H]\centering\small
\caption{The complete set of differences among the four United States personas. The age phrases are natural-language cohort labels, not sampling boundaries; the deterministic WVS assignment is documented in the pipeline below. Every other clause of the prompt is identical across the four, including the closing moral-dilemma instruction.}
\label{tab:usa_persona_deltas}
\setlength{\tabcolsep}{4pt}
\begin{tabular}{@{}lp{0.30\linewidth}lp{0.26\linewidth}@{}}\toprule
Cohort & Prompt label (approximate; not a sampling boundary) & National pride & Moral acceptability \\\midrule
Young     & a young adult \ldots in your 20s and early 30s & moderately proud & morally conservative on contested issues \\
Middle    & a middle-aged adult \ldots in your 40s or 50s  & moderately proud & strictly opposed to such contested moral acts \\
Older     & a senior citizen \ldots over 60 years old      & intensely proud  & strictly opposed to such contested moral acts \\
Aggregate & an adult citizen (no age clause)               & moderately proud & strictly opposed to such contested moral acts \\
\bottomrule
\end{tabular}
\end{table}

\paragraph{Aggregate (fourth) persona.}
The fourth agent (the aggregate row of Table~\ref{tab:usa_persona_deltas}) is constructed from the country's population-wide WVS profile, pooling respondents across age cohorts. It serves as a demographic anchor: by suppressing cohort-specific noise it ensures that the ensemble's consensus target $\bar{\delta}$ reflects the broadest available empirical signal rather than a single generational viewpoint. We fix this construction before the headline evaluation; alternative anchor designs are not used to select the reported headline results.

\paragraph{Faithfulness disclaimer.}
The descriptors above are read directly from the country's WVS-7 means through deterministic quartile cuts (Table~\ref{tab:wvs_thresholds}); they are not hand-selected to match cultural stereotypes. As a consequence, individual cohorts can present empirically grounded but locally counter-intuitive combinations (for example, a U.S.\ young-adult that is \emph{leaning materialist} and \emph{morally conservative on contested issues}), reflecting actual WVS-7 distributions rather than authorial bias.

\paragraph{Prompt labels versus sampling bands.}
The age phrases in the prompts are verbal cohort labels, not inclusion rules.
The deterministic assignment happens before prompt rendering: young is age
$<36$, middle is $36\leq\mathrm{age}\leq55$, and older is age $>55$. Thus the
phrase ``over 60'' should not be read as a literal filter on respondents; it is
the natural-language description attached to the older WVS band.

\subsection{WVS Data Processing Pipeline}
\label{app:wvs_pipeline}

We process WVS Wave~7 individual-level microdata following the ten-variable cultural-value scheme of \citet{greco2026personas}. The pipeline is fully deterministic: given the inverted WVS-7 CSV and the country code, it returns the four personas used at inference time without any tunable hyperparameters.
\begin{enumerate}
  \item \textbf{Variable extraction:} For each respondent, we extract 10 value dimensions from WVS items (Table~\ref{tab:wvs_thresholds}). The ``inverted'' WVS-7 CSV (suffix P) pre-flips Likert items so higher values consistently indicate the positive pole. Items without the P suffix (Q152, Q153, Q177-Q182) use original WVS coding.
  \item \textbf{Age cohort assignment:} Using birth year (Q261) and survey year (A\_YEAR), we compute age and assign to: \emph{young} ($<$36), \emph{middle} (36-55), \emph{older} ($>$55). The rendered prompt describes each cohort with an approximate age phrase (``in your 20s and early 30s'', ``in your 40s or 50s'', ``over 60'') rather than the exact cut-points; the WVS means that fill the descriptors are computed on the cut-point bands. Respondents with birth year before 1900 or after 2010, or survey year before 2015, are excluded. Negative WVS codes ($-1, -2, -4, -5$: refusal / don't know) are dropped; zero is retained for binary 0/1 items.
  \item \textbf{Cohort aggregation:} For each country and age cohort, each dimension's score is the unweighted mean of all valid (respondent, item) values associated with that dimension. For multi-item dimensions (e.g., Q177-Q182 for moral acceptability), every respondent contributes one value per item to a single pooled mean. For single-item dimensions this reduces to the standard across-respondent mean.
  \item \textbf{Normalisation and descriptor mapping:} Raw dimension means are normalised to $[0,1]$ via $(\text{value}-\text{lo})/(\text{hi}-\text{lo})$ where $(\text{lo},\text{hi})$ is the WVS scale range, then directionally flipped so higher $=$ positive pole. Four-level descriptors are assigned by quartile cuts: $\geq 0.75$ (strong positive), $\geq 0.50$, $\geq 0.25$, $< 0.25$ (strong negative). See Table~\ref{tab:wvs_thresholds}.
  \item \textbf{Fallback:} For countries with insufficient WVS coverage (e.g., Saudi Arabia), the pipeline can use manually authored native-language personas grounded in area-studies literature, covering the same demographic structure (3 age cohorts + 1 population-wide aggregate). This fallback is not used in the 20-country headline panel.
\end{enumerate}

\begin{table}[H]\centering\small
\caption{The 10 WVS cultural-value dimensions used for persona generation. The ten WVS variables follow \citet{greco2026personas}; the quartile cut is ours. All dimensions use uniform quartile cuts on the $[0,1]$-normalised score: $\geq 0.75$ (descriptor 1, strongest positive pole), $\geq 0.50$ (2), $\geq 0.25$ (3), $< 0.25$ (4, strongest negative pole). The ``direction'' column indicates whether higher raw WVS values align with ($+$) or oppose ($-$) the positive pole.}
\label{tab:wvs_thresholds}
\setlength{\tabcolsep}{3pt}\scriptsize
\begin{tabular}{llccl}\toprule
Dimension & WVS items & Range & Dir. & Descriptors ($\geq$0.75 / $\geq$0.50 / $\geq$0.25 / $<$0.25) \\\midrule
Religiosity & Q6P & 1-4 & $+$ & deeply / moderately religious / somewhat / highly secular \\
Child-rearing & Q17P & 0-1 & $-$ & independence-imagination $\leftrightarrow$ obedience-faith \\
Moral accept. & Q177-Q182 & 1-10 & $+$ & very permissive $\leftrightarrow$ strictly opposed \\
Social trust & Q57P & 1-2 & $+$ & very high trust $\leftrightarrow$ deep distrust \\
Polit.\ particip. & Q199P, Q200P & 1-3 & $+$ & active participant $\leftrightarrow$ non-participant \\
National pride & Q254P & 1-4 & $+$ & very proud $\leftrightarrow$ not proud \\
Happiness & Q46P & 1-4 & $+$ & very happy $\leftrightarrow$ not happy \\
Gender equality & Q29P, Q30P, Q31P, Q33P & 1-4 & $-$ & strongly egalitarian $\leftrightarrow$ traditional roles \\
Materialism & Q152, Q153 & 1-3 & $+$ & post-materialist $\leftrightarrow$ materialist \\
Tolerance & Q19P-Q23P & 0-1 & $-$ & very tolerant $\leftrightarrow$ intolerant \\
\bottomrule
\end{tabular}
\end{table}

\subsection{WVS-to-trolley dimension linkage}
\label{app:wvs_linkage}

The ten WVS dimensions do not map one to one onto the six MultiTP attributes,
and DISCA does not need them to. They act in two ways. Some overlap in theme:
religiosity can bear on the Species attribute through human exceptionalism,
and moral acceptability bears on trade-offs along the Fitness and Social Value
attributes. For Gender, \citet{awad2018moral} report that the preference for
sparing women is stronger in countries where women have better health and
survival prospects; this is a country-level association with a different
index, not a WVS item. Other dimensions, such as social trust, child-rearing
values and materialism, change the moral frame a persona adopts rather than
one attribute. The panel only needs these descriptors to produce different
decision gaps across personas; the controller then decides how to combine
them. We do not claim a causal path from any WVS item to any MultiTP attribute.

\section{AMCE Estimation Details}
\label{app:amce}

\paragraph{Uniform treatment of all dimensions.} All six dimensions (Species, Gender, Age, Fitness, Social Value, and Utilitarianism) are computed identically. For each dimension, we compute the empirical mean of $p_{\text{spare}}(\mathbf{x})$ across all scenarios in that dimension (Eq.~\ref{eq:amce-estimate}). For the five binary dimensions, each scenario presents a forced choice between two character groups differing on exactly one attribute. For Utilitarianism, scenarios vary the number of lives on each side; the binary indicator is whether the larger group (``More'') or the smaller group (``Less'') is spared. This uniform mean-based estimation is consistent with the \texttt{compute\_ACME} method in the MultiTP codebase~\citep{jin2025multitp}, which fits a no-intercept linear regression with a binary group indicator-mathematically equivalent to taking the mean of the saving probability for the preferred group.

\paragraph{Human AMCE scale.} The MultiTP ground-truth AMCEs in the MultiTP-released country-specific AMCE table are on a $[-1, 1]$ scale. We convert to a preference rate via $(1 + \text{AMCE}_{\text{raw}}) / 2$, where $0.5$ represents no preference and $1$ represents maximal preference for the ``preferred'' group; tables that show rates in percentage points multiply by 100. Model rates are computed on the same scale, and MIS is always the Euclidean distance between the two six-dimensional vectors on the $[0,1]$ scale.

\paragraph{Consistency.} Because the two groups of a MultiTP scenario differ on one attribute only, the mean sparing rate for that attribute is the same marginal-effect estimate that \citet{awad2018moral} compute for the human data.

\paragraph{Cross-lingual A/B token elicitation.}
\label{app:token_elicitation}
Eq.~\ref{eq:amce-estimate} requires a clean A vs.\ B logit gap on every scenario, which in turn requires that the decision tokens be elicited consistently across model families and across the languages represented in the 20-country panel. Each scenario is rendered entirely in the target country's native language: this includes (i)~scenario framing and context sentences (4 paraphrase variants per language), (ii)~all 22 character types with singular/plural forms, (iii)~road-lane labels and group identifiers in native script, and (iv)~closing questions with language-appropriate conjunctions and grammar. The prompt frame wraps the scenario in a native-language instruction that explicitly requests an English answer token, following the pattern: ``[moral dilemma preamble in native language] \textbackslash n \{scenario\} \textbackslash n [instruction to answer A or B in English]''. This keeps moral reasoning in the culturally native linguistic frame while maintaining a single decision interface across the panel.

\paragraph{Token verification.} We do not hard-code the IDs of bare ``A'' and
``B'' strings. For each backbone and language, we construct the actual
chat-templated prompt prefix, tokenize it, tokenize the prefix followed by
``A'' and ``B'', and take the first new token in each sequence. We assert that
each answer is represented by one token, decode and record the resolved IDs and
token text, and then extract logits at those two positions. This handles
space-prefixed and merged forms and avoids naive
\texttt{encode("A")[0]} or \texttt{encode("B")[0]} lookup.

\section{Dataset Preprocessing}
\label{app:dataset}
\label{app:dataset_sens}

\paragraph{Pipeline.} Each country's MultiTP file is processed in four steps.
(i) Only the canonical paraphrase of each scenario is kept. (ii) A quality
filter drops Utilitarianism scenarios in which the two groups have the same
size and every character belongs to the set \{pregnant woman, woman, large
woman\}; such pairs would inject a signal other than the number of lives into
an attribute that measures the number of lives. (iii) At most 80 scenarios are
kept per attribute, and a fixed-seed draw with near-equal quotas per attribute
then gives 310 scenarios per country. (iv) The pool is shuffled with seed 42, and the controller processes
the scenarios of a country in that order.

\paragraph{Reproducible left and right assignment.} When the MultiTP field
that records the paraphrase choice does not determine which group is shown on
the left, we fall back to a SHA-256 hash of the scenario's group descriptors
modulo 2, so the same scenario receives the same order on every run and
machine. The pipeline logs the fallback rate for every country.

\paragraph{Why both option orders are read.} Every scenario is evaluated with
the preferred group as option A and again as option B, and the two gaps are
averaged with opposite signs (Algorithm~\ref{alg:disca}). Without this step
the decision would depend on the position of the group as well as on the
group; with it the position has no degree of freedom left in the reported
decision. This does not make the model itself free of position bias.

\section{Computational Cost}
\label{app:trigger}

DISCA reads the decision-token logits of five prompts per scenario, the base
prompt and four personas, under both option orders. These ten prompt
evaluations run as two batched forward passes, against one or two for vanilla
decoding. The importance sampling and the country prior act on scalars and
add negligible time next to the forward passes. The controller runs on every
scenario; it does not first test whether the panel disagrees. In a latency
benchmark with Phi-4 on one H100 GPU (100 scenarios for each of three
countries), DISCA takes 0.083 to 0.086 seconds per scenario across proposal
budgets of 64, 128 and 256, against 0.023 seconds for vanilla decoding, which
is 3.57 to 3.72 times the vanilla wall time.

\subsection{Selective routing for inference efficiency}
\label{app:routing}

Because the controller matters most on uncertain items, we test a label-blind
entropy router on a separate, broader NormAd audit (1,321 items across 38
countries), not on the 622-item, 17-country set of
App.~\ref{app:normad_verified}. Every item first receives base-only inference.
Items are ranked by the normalised entropy of the base-only option
distribution, and only the most uncertain fraction is sent to the full
four-persona pipeline. The router never reads gold labels or DISCA
predictions.

\begin{table}[H]\centering\small
\caption{Entropy-based selective routing on a broader NormAd audit
(Qwen2.5-7B, 1,321 items, 38 countries, three controller seeds). Quality is
macro-F1; time is option-logit inference time only. n.t.: the 25\% setting
was replayed from cached predictions for quality and was not timed.}
\label{tab:routing_all}
\setlength{\tabcolsep}{4pt}
\begin{tabular}{lrrr}
\toprule
Routed fraction & Items routed & Macro-F1 gain retained & Time saved \\
\midrule
25\% & 331/1{,}321 & 92.6\% & n.t. \\
50\% & 661/1{,}321 & 100.0\% & 45.9\% \\
100\% & 1{,}321/1{,}321 & 100.0\% & 0.0\% \\
\bottomrule
\end{tabular}
\end{table}

Routing the most uncertain quarter of the items keeps most of the full NormAd
gain, and routing half of them keeps all of it while saving almost half of the
inference time. These timings exclude model loading, parsing, scoring and
input and output.

\section{Relationship to Persona-Dependent LLM Alignment}
\label{app:kim_comparison}

\citet{kim2025persona} is the closest prior work to DISCA. Table~\ref{tab:kim_comparison} states the nine axes on which the two differ; we do not restate them in prose.

\begin{table}[H]\centering\small
\caption{Comparison between \citet{kim2025persona} and DISCA.}
\label{tab:kim_comparison}
\setlength{\tabcolsep}{4pt}
\begin{tabular}{lll}\toprule
Dimension & Kim et al.\ (2025) & DISCA (this work) \\\midrule
Goal & Measure persona sensitivity & Correct cultural misalignment \\
Personas & 7 sociodemographic categories & 4 WVS-grounded agents per country \\
Persona source & Author-designed binary labels & Empirical WVS-7 microdata \\
Metric & MDD (persona-pair distance) & MIS ($\ell_2$ to human AMCE) \\
Correction & None (diagnostic only) & PT-IS + dual-pass reliability gate \\
Models & 3 (GPT-4o, GPT-3.5, Llama-2) & 7 checkpoints, 5 families, 2B-70B \\
Countries & Aggregate (Western vs.\ Eastern) & 20 individual countries \\
Scenarios & 9 dimensions, 10k synthetic & 6 dimensions, 310 per country \\
Theory & Partisan sorting (descriptive) & Variance-aware MSE-optimal shrinkage (Prop.~\ref{prop:shrinkage}) \\
\bottomrule
\end{tabular}
\end{table}

One finding of theirs is load-bearing for a design choice of ours rather than merely adjacent to it. \citet{kim2025persona} document that specific personas can entirely reverse an LLM's preference on certain moral dimensions (for example social status under a progressive persona in GPT-4o), and it is exactly that regime the dual-pass reliability gate (\S\ref{sec:method:estimator:gate}) is designed to moderate: when a persona-driven refinement is large enough to flip the decision, disagreement between the two IS passes can suppress the stochastic residual rather than committing to an unstable flip. On the two saved grids, Full DISCA harms 4 of 40 country-model cells against 13 of 40 for a consensus-only arm without the country prior, but matched on the prior the difference shrinks to 4 against 5 and does not come from the gated refinement (Table~\ref{tab:tail_safety}); the gate's protection against such flips is therefore a design intent that our data do not isolate. Their ``partisan sorting'' result-political orientation producing a far larger persona effect in LLMs than in human respondents-is likewise why our personas are grounded in WVS-7 microdata spread over ten axes (Table~\ref{tab:wvs_thresholds}) rather than in author-designed labels. We do not treat alternative demographic-axis constructions as validated evidence in this paper; the relevant design risk is that a coherent but miscalibrated panel can steer confidently in the wrong direction.

\section{Extended Limitations}
\label{app:limitations}

Each item below is a limit that the current evidence leaves open, followed by
the work it points to.

\begin{itemize}
\item \textbf{The country prior is transductive.} Each decision depends on the
earlier test scenarios of the same country and on their order, the prior pools
all six attributes, and the first 50 scenarios of a country receive no prior.
A deployment that sees one query at a time, or interleaves queries from many
countries, would not reproduce the prior's contribution, which is large
(Table~\ref{tab:decomposition}). A prior estimated once from a held-out set of
country prompts is the natural replacement to test.
\item \textbf{MIS rewards calibration.} Dividing the vanilla gap by a single
temperature $T{=}8$ chosen with the targets beats Full DISCA on four of the
seven backbones (Table~\ref{tab:calibration_sweep}). Part of the measured gain
is therefore reachable by making the model less confident. A calibration
control chosen without the targets belongs next to DISCA in future
comparisons.
\item \textbf{The direction of a persona shift is assumed, not validated.} The
controller assumes that persona shifts move toward the population's
preference. When a persona's shift is unrelated to the human target, the panel
can steer confidently in the wrong direction. A direct test of persona
directions against country-level human vectors, independent of the base
model, is still missing.
\item \textbf{Coverage depends on WVS microdata.} The evaluation covers 20 of
the more than 100 MultiTP countries. WVS Wave~7 covers about 66 countries and
territories, so much of the long tail has no microdata; there the pipeline can
only fall back to hand-written personas, which we do not evaluate. On
MENAValues we used a profile pooled from neighbouring countries for countries
without WVS data, and DISCA did not beat simple prompting there
(App.~\ref{app:mena_cdeval}).
\item \textbf{Few low-resource languages.} Only Amharic (Ethiopia) and
Burmese (Myanmar) in the panel are low-resource languages, and each country
uses one language, so linguistic and cultural effects are mixed. A panel with
several languages per country would separate them.
\item \textbf{Precision differs across the suite.} Llama-3.3-70B and
Qwen3-VL-8B ran from 4-bit (bitsandbytes) checkpoints and Qwen2.5-7B in
BF16. For the Phi-4, Phi-3.5-mini, Magistral and Gemma runs the precision
was not logged; the loader's default was 4-bit. 4-bit quantisation perturbs decision logits, so comparisons across
model sizes are also comparisons across precision.
\item \textbf{One controller seed for the headline suite.} The seven headline
backbones are reported at controller seed 42. Seed variation is measured only
on the two Qwen2.5 grids, where it is negligible, and a bootstrap over
countries and scenarios is available only for those two grids
(App.~\ref{app:gate_diagnostics}).
\item \textbf{The model suite is selected.} The seven backbones were chosen to
span families, sizes and failure cases. The range of effects is informative
about heterogeneity, not an unbiased estimate of the average effect over all
models.
\item \textbf{Free-form judging depends on the judge.} The MAKIEval pairwise
results use model judges, and the second-judge check covers only the
country-unnamed condition. The deterministic entity count carries the claim.
\item \textbf{General capability is not measured.} DISCA changes only the
scores of a fixed set of options. We do not test whether persona conditioning
changes helpfulness or safety behaviour, and we compare against no
training-based pluralistic method such as CultureSPA~\citep{xu2025culturespa},
only against its culture-aware prompt.
\end{itemize}

\subsection{Safety boundary}
\label{app:worst_group_boundary}

Country-level alignment is not a fairness guarantee. Within the controller,
loss aversion penalises candidate corrections that harm one persona, but it
does not prevent them. On the two Qwen2.5 grids, at least one of the four
personas has a negative minimum internal utility on 51.4\% of Qwen2.5-7B
scenarios and 53.3\% of Qwen2.5-14B scenarios (means over three controller
seeds), with mean minima of $-0.083$ and $-0.152$. These are diagnostics over
synthetic prompts, not demographic harm rates, but they rule out a claim that
the controller protects every persona, demographic minority or user group.

\section{Out-of-Family Transfer Audits}
\label{app:out_of_family}

The following audits use separate, explicitly named protocols. They extend the
evaluation beyond the MultiTP trolley format without pooling incomparable item
scopes or treating positive transfer as evidence of general cultural
understanding.

\subsection{NormAd country-context classification}
\label{app:normad_verified}

NormAd~\citep{rao2025normad} evaluates three-way country-context social-norm
classification with the labels \texttt{yes}, \texttt{no}, and \texttt{neutral}.
The evaluation set contains 622 items across 17 represented countries. Accuracy
and three-label macro-F1 are computed per country and then averaged equally.
All methods reuse the same model outputs: vanilla, an aggregate WVS profile,
four-persona probability averaging, hard majority vote, and Full DISCA.

For this three-way task we use the same score-space controller in its $K$-way
form, rather than running three unrelated binary controllers. Let
$z_0,z_1,\ldots,z_N\in\mathbb{R}^{K}$ be the base and persona option-logit
vectors, with the fixed option order $(\texttt{yes},\texttt{no},
\texttt{neutral})$ and $K=3$. We centre each vector,
$\tilde z_j=z_j-\frac{1}{K}(\mathbf{1}^{\top}z_j)\mathbf{1}$, and form the
persona anchor $a=N^{-1}\sum_{i=1}^{N}\tilde z_i$ and the persona residuals
$e_i=\tilde z_i-a$. The raw option logits enter without a logit temperature;
each is the average over four option orders mapped back to the canonical
order. Each PT-IS pass draws $K_{\text{half}}$ candidate perturbations
$\epsilon_k=\sum_i g_{k,i}\,e_i$ with $g_{k,i}\sim\mathcal{N}(0,1)$, rescales each
to root-mean-square size $\sigma$ per degree of freedom, and projects it onto
the zero-sum hyperplane. Proposals therefore move only along directions in
which the personas disagree (rank at most $N-1$), and they vanish when the
panel agrees exactly. The proposal scale is
$\sigma=\max\{(\sum_i\|e_i\|_2^2/((N-1)(K-1)))^{1/2},\,\sigma_0\}$. The gains
of Eq.~\ref{eq:gains} use the $\ell_1$ distance over the $K$ options,
$g_{i,k}=(\|\tilde z_0-\tilde z_i\|_1-\|a+\epsilon_k-\tilde z_i\|_1)/\sigma$,
and likewise for the consensus gain. The value function uses curvature
$\alpha{=}0.88$ on gains and $1.0$ on losses in this variant. The two passes
give $d_1,d_2\in\mathbb{R}^K$, the inter-pass statistic is
$V_r=\|d_1-d_2\|_2^2$, the gate is $r=\exp(-V_r/s)$, and
$d^\star=r(d_1+d_2)/2$. With $\alpha_{\mathrm{ess}}$ the ESS blend of
Algorithm~\ref{alg:disca}, the final option distribution is
\begin{equation}
 p_{\mathrm{final}}=\operatorname{softmax}\!\left(
 \frac{\alpha_{\mathrm{ess}}a+(1-\alpha_{\mathrm{ess}})\tilde z_0+d^\star}
 {T_{\mathrm{dec}}}\right),\qquad
 \hat y=\arg\max_{k\in\{1,\ldots,K\}}p_{\mathrm{final},k}.
 \label{eq:kway-disca}
\end{equation}
with $T_{\mathrm{dec}}=0.5$, and the final logits are re-centred before the
softmax. No country prior is applied. The reported accuracy and macro-F1 use
$\hat y$ against the released three-way label. This $K$-way variant differs from
the binary controller in three respects, the proposal distribution, the
loss curvature, and the absence of the prior, so its results are read as a
transfer check of the panel-plus-controller idea, not of the exact binary
recipe.

The $K$-way variant also uses loss aversion $\kappa{=}2.25$, $K_{\text{half}}{=}64$,
$\sigma_0{=}0.3$, $\lambda_{\text{coop}}{=}0.7$ and $s{=}0.04$.

\begin{table}[H]\centering\small
\caption{NormAd results (DISCA rows average controller seeds 42, 101 and 2026 replayed on one set of model outputs; the other rows are deterministic; 622 items, 17 countries; accuracy
and three-label macro-F1, both averaged with equal country weight; higher is
better). The country condition gives the model only the country, which is the
setting the panel is meant for. The other two NormAd conditions also state the
relevant value or rule of thumb in the prompt and serve as negative controls.
Bold marks the best method per backbone and column.}
\label{tab:normad_verified}
\setlength{\tabcolsep}{3.5pt}
\resizebox{\linewidth}{!}{%
\begin{tabular}{llcccccc}
\toprule
& & \multicolumn{2}{c}{Country} & \multicolumn{2}{c}{Country + value} & \multicolumn{2}{c}{Rule of thumb} \\
\cmidrule(lr){3-4}\cmidrule(lr){5-6}\cmidrule(lr){7-8}
Backbone & Method & Acc. & F1 & Acc. & F1 & Acc. & F1 \\
\midrule
Qwen2.5-7B & Vanilla & .5543 & .4613 & \textbf{.6336} & \textbf{.5343} & \textbf{.7476} & \textbf{.6916} \\
 & Aggregate WVS profile & .5739 & .4846 & .6305 & .5196 & .7073 & .6169 \\
 & Mean ensemble & .5756 & .4849 & .6315 & .5198 & .7079 & .6183 \\
 & Hard majority & .5735 & .4794 & .6315 & .5193 & .7045 & .6134 \\
 & DISCA & \textbf{.5762} & \textbf{.4859} & .6310 & .5204 & .7084 & .6191 \\
\midrule
Qwen2.5-14B & Vanilla & .5299 & .5014 & \textbf{.6488} & \textbf{.6030} & .7834 & \textbf{.7486} \\
 & Aggregate WVS profile & .5472 & .5064 & .6403 & .5816 & .7828 & .7394 \\
 & Mean ensemble & .5497 & .5038 & .6441 & .5858 & \textbf{.7844} & .7417 \\
 & Hard majority & .5463 & .4999 & .6417 & .5815 & .7815 & .7376 \\
 & DISCA & \textbf{.5541} & \textbf{.5088} & .6450 & .5876 & .7827 & .7396 \\
\bottomrule
\end{tabular}}
\end{table}

In the country condition, every way of using the persona panel improves
accuracy over vanilla on both backbones (only the hard majority falls slightly
below vanilla on Qwen2.5-14B macro-F1), and DISCA is the best method in both
columns. Its
margin over the persona mean ensemble is small, .0006 to .0050 in accuracy and
macro-F1, so the panel supplies most of the gain. In the country condition,
rerunning the controller with seeds 101 and 2026 moves DISCA by at most .003. When the prompt already
states the relevant value or rule, the panel no longer helps: vanilla is best
on seven of the eight cells of these two conditions, and DISCA tracks the mean
ensemble. We therefore read NormAd as transfer to country-conditioned
social-norm judgment when the country is the only context, not as a test of
factual cultural knowledge.

\subsection{GlobalOpinionQA distribution matching}
\label{app:goqa_tiera}

GlobalOpinionQA~\citep{durmus2023global} measures how closely a model matches a
country's reported opinion distribution. We use two protocols. The first, a
restricted protocol, evaluates Qwen2.5-7B on two pre-defined scopes: 331 native-binary
questions and 638 four-point questions with an audited midpoint collapse. The
combined scope contains 969 questions, 3,891 question-country targets, 15
countries, and three controller seeds. Distances are equal-country,
equal-seed square-root Jensen-Shannon distances in base 2; lower is better.
Probability averaging and Full DISCA reuse the same four persona forwards, so
the comparison isolates aggregation rather than an extra model-call budget.
Because every item in this protocol is binary after the midpoint collapse, it
runs the binary controller of Algorithm~\ref{alg:disca}, not the $K$-way variant:
both option orders are read and symmetrised, raw logits are divided by the
default $T_c{=}3.0$, and the country prior of \S\ref{sec:method:prior} is on,
applied in a fixed question order within each country. Probability averaging
receives no prior, so this comparison does not separate the aggregation rule
from the prior.

\begin{table}[H]\centering\small
\caption{GlobalOpinionQA results for Qwen2.5-7B under the restricted protocol. The two component scopes and their combined primary scope are reported separately; lower Jensen-Shannon distance is better.}
\label{tab:goqa_tiera}
\setlength{\tabcolsep}{4pt}
\begin{tabular}{llrrr}
\toprule
Scope & Method & $d_{\mathrm{JS},2}\downarrow$ & vs.~vanilla & vs.~mean \\
\midrule
Native binary & Vanilla & .3696 & & \\
 & Probability average & .3701 & & \\
 & \textbf{Full DISCA} & \textbf{.3423} & \gain{-.0273} & \gain{-.0278} \\
\midrule
Four-point collapse & Vanilla & .3620 & & \\
 & Probability average & .3509 & & \\
 & \textbf{Full DISCA} & \textbf{.3083} & \gain{-.0537} & \gain{-.0427} \\
\midrule
Combined primary & Vanilla & .3666 & & \\
 & Probability average & .3585 & & \\
 & \textbf{Full DISCA} & \textbf{.3192} & \gain{-.0474} & \gain{-.0393} \\
\bottomrule
\end{tabular}
\end{table}

Full DISCA is the best method in both scopes and in their combined
aggregate, including against probability averaging, which however receives no
country prior. One country in the combined scope has only 42 eligible items, so the result is a positive
equal-country aggregate rather than a claim that every country estimate is
equally powered. This protocol differs in item scope and output construction from other
GlobalOpinionQA evaluations, so its values are not comparable to them.

\paragraph{A second protocol: Pew questions with the $K$-way controller.}
A second GlobalOpinionQA protocol keeps the answer options as released and
uses the $K$-way variant of Eq.~\ref{eq:kway-disca}, with no country prior, as
on NormAd. It keeps the questions from the Pew Global Attitudes source that
cover at least five of our target countries and removes all questions taken
from the World Values Survey, so that the personas cannot echo their own
source data. This leaves 560 questions, 4{,}814 question-country pairs and the
same 15 countries, evaluated on Qwen2.5-7B and Qwen2.5-14B. The metric is the
same equal-country Jensen-Shannon distance. Table~\ref{tab:goqa_full} reports
the results next to Table~\ref{tab:goqa_tiera}.

\begin{table}[H]\centering\small
\caption{GlobalOpinionQA, Pew questions with the $K$-way controller (560
questions, 15 countries, no country prior; Jensen-Shannon distance, lower is
better). The DISCA row averages controller seeds 42, 101 and 2026 on one set
of model outputs; the other rows are deterministic. Bold marks the best method
per backbone.}
\label{tab:goqa_full}
\setlength{\tabcolsep}{4pt}
\begin{tabular}{lcc}
\toprule
Method & Qwen2.5-7B & Qwen2.5-14B \\
\midrule
Vanilla & .6130 & .6689 \\
Aggregate WVS profile prompt & .5931 & .6127 \\
Probability averaging & \textbf{.5692} & \textbf{.5879} \\
Majority vote & .6337 & .6322 \\
DISCA & .5838 & .6111 \\
\bottomrule
\end{tabular}
\end{table}

Under this protocol DISCA is closer to the survey distributions than vanilla
decoding, the aggregate WVS profile and majority vote on both backbones, but
probability averaging over the same personas is closer still. The two
protocols differ in three ways: the answer format (binary after a collapse
against the released multi-option answers), the controller (binary against
$K$-way) and the country prior (on against off). We therefore do not read the
first protocol as showing that the controller beats averaging. A target that
is a full answer distribution rewards averaging, while DISCA is built to return
one calibrated decision, so we read this protocol as a boundary of the method.

\subsection{Free-Form Generation on MAKIEval}
\label{app:makieval}

We report the country-unnamed condition as the primary missing-context test.
This condition removes the country name from the task prompt, leaving the
persona panel as the only country-specific signal. The other prompt conditions
are not used for the transfer claim because they already provide country
context and answer a different question. Five candidates are generated from
the vanilla view and four WVS personas under a shared sampling budget; DISCA
selects from the same candidate pool as vanilla.

\paragraph{Selection protocol.} For each prompt we generate exactly one
candidate under the vanilla context and one under each of the four persona
contexts, producing one shared pool of five candidates. Every context scores
every candidate by its mean token log-probability. Vanilla uses the candidate
generated under the vanilla context; vanilla best-of-5 takes the highest vanilla
score; the mean-ensemble arm takes the highest probability under the mean of the
four persona distributions; and DISCA takes the highest probability from the
$K$-way controller in Eq.~\ref{eq:kway-disca}. All choices use deterministic
argmax with the lowest candidate index resolving ties. Conditional on the
generated candidate pool, the controller's seeded internal perturbation draws
are the only stochasticity in DISCA's selection step. No
external judge is used for selection; judging is applied afterwards to the
selected passages. Empty generations are retried once with the documented
per-view retry seed.

\begin{table}[H]\centering\small
\caption{MAKIEval country-unnamed condition. Entity counts use the benchmark's released country-topic lexicon and are compared with exact paired sign tests over the prompts on which the two arms differ (better/worse counts in brackets). \emph{Mean ens.} selects the candidate with the highest probability under the mean of the four persona distributions, from the same pool. Vanilla best-of-5 equals vanilla on every unnamed prompt and is omitted. The Qwen2.5-7B row uses three seeds; the other rows use one seed.}
\label{tab:makieval_unnamed_verified}
\setlength{\tabcolsep}{3.5pt}\scriptsize
\resizebox{\linewidth}{!}{%
\begin{tabular}{@{}lrrrrll@{}}
\toprule
Backbone & Prompts & Vanilla & Mean ens. & DISCA & DISCA vs.\ vanilla, $p$ & DISCA vs.\ mean ens., $p$ \\
\midrule
Qwen2.5-7B & 288 & 8 & \textbf{49} & 47 & $1.5\times10^{-5}$ & $.75$ [4/6] \\
Qwen2.5-14B & 96 & 4 & \textbf{23} & 21 & $7.3\times10^{-4}$ & $1.0$ [5/5] \\
Phi-4 (14B) & 96 & 4 & \textbf{22} & 18 & $4.2\times10^{-3}$ & $.45$ [2/5] \\
Qwen3-8B & 96 & 0 & \textbf{7} & 5 & $.125$ & $.63$ [1/3] \\
\bottomrule
\end{tabular}}
\end{table}

The deterministic count increases on the unnamed condition for all four
backbones, with the smallest row remaining inconclusive. The persona mean
ensemble raises the count at least as much, significantly over vanilla on all
four backbones (including Qwen3-8B, $p=.031$), and no difference between it
and DISCA is significant. On this measure the gain belongs to the persona
panel, not to disagreement weighting. When the prompt
already names the country the count reverses: on Qwen2.5-7B the selected
passages name 245 target-country entities under vanilla and 189 under DISCA
(sign test $p=2\times10^{-3}$), and the same direction holds on the other
three backbones. This reversal is what makes the unnamed result credible as a
missing-signal effect rather than a general improvement.

\paragraph{Blinded pairwise judging.} Each pair holds the vanilla-selected and
the DISCA-selected passage for one prompt, in a blinded order, and the primary
judge (Gemini 3 Flash) answers a cultural-fit question and a writing-quality
question (Table~\ref{tab:makieval_judge_summary}). On country-unnamed pairs
the judge prefers the DISCA passage for cultural fit on 117 of 124 decided
pairs. It also prefers DISCA on writing quality overall, so quality is not
merely preserved; the test that separates the two is the discordant-pair
contrast, in which culture-win/quality-loss pairs outnumber
culture-loss/quality-win pairs 62 to 17, so the cultural-fit preference is not
explained by a general preference for one text. A second judge (GPT-5.6 Sol),
on 261 country-unnamed Qwen2.5-7B pairs from three seeds, prefers the DISCA
passage for cultural fit at a rate of .849 [.812, .883], counting a tie as
half a win (95\% bootstrap interval over pairs). On writing quality it shows
no clear preference (.533 [.471, .594]). The judged set contains only
pairs on which the two methods selected different passages, and the entity
count, not the judge, carries the mechanism claim. These results support the
narrow claim that the panel can supply missing country context; they do not
show that DISCA improves all free-form prompts or general cultural writing.

\begin{table}[H]\centering\small
\caption{MAKIEval blinded pairwise judging, DISCA versus vanilla (Qwen2.5-7B, single judge Gemini 3 Flash, seeds $\{42,101\}$, 515 informative pairs). Win rates are over decided pairs with ties dropped; intervals are Clopper-Pearson; the last row is an exact McNemar test on the 382 pairs with both questions decided.}
\label{tab:makieval_judge_summary}
\setlength{\tabcolsep}{5pt}
\begin{tabular}{@{}lrcr@{}}
\toprule
Question / condition & Wins/decided & Rate [95\% CI] & Exact $p$ \\
\midrule
Cultural fit, country unnamed & 117/124 & \textbf{.944} & $7.5\times10^{-27}$ \\
Cultural fit, country named & 88/130 & .677 & $6.7\times10^{-5}$ \\
Cultural fit, all pairs & 293/384 & .763~[.717,.805] & $7.3\times10^{-26}$ \\
Writing quality, all pairs & 334/513 & .651~[.608,.692] & $7.3\times10^{-12}$ \\
\midrule
\multicolumn{3}{@{}l}{Discordant pairs: culture win/quality loss $=62$ vs.\ culture loss/quality win $=17$} & $3.6\times10^{-7}$ \\
\bottomrule
\end{tabular}
\end{table}

\subsection{MENAValues and CDEval: where the panel does not beat prompting}
\label{app:mena_cdeval}

Two further transfer tests probe regions and constructs far from MultiTP. Both
use Qwen2.5-7B-Instruct in FP16, controller seed 42, both option orders, and
the same five arms: vanilla, a prompt that names the country, the
culture-aware prompt of CultureSPA~\citep{xu2025culturespa} used verbatim, the
persona mean ensemble, and DISCA. DISCA runs the $K$-way controller of
Eq.~\ref{eq:kway-disca} without the country prior. Every arm turns option
scores into a distribution with the same decision temperature $0.5$. Both runs
completed on every item.

\paragraph{MENAValues.} From MENAValues~\citep{zahraei2025menavalues} we keep
the Arab Opinion Index items for which a country answer distribution is
released. Our filtering leaves 502 items over 14 countries and 5{,}831
item-country targets.
Seven of the countries have WVS microdata and receive their own personas. The
other seven have none, so their personas come from the pooled WVS profile of
the seven covered countries. We report the square-root Jensen-Shannon
distance in base 2 between the predicted and the human answer distribution,
and the benchmark's normalised value alignment score (NVAS), both averaged
with equal country weight. Intervals are 2{,}000-draw item bootstraps.

\paragraph{CDEval.} We use a 599-item subsample (about 100 per Hofstede
dimension) of the 2{,}953 items of CDEval~\citep{wang2024cdeval}. Each item offers two options, one per pole of the
dimension. We condition on each of our 20 countries, compute the country's
orientation on each dimension, and correlate it across countries with the
country's Hofstede score (17 to 19 countries per dimension have a score).

\begin{table}[H]\centering\small
\caption{MENAValues (14 countries) and CDEval (20 countries), Qwen2.5-7B,
seed 42. JSD is lower-is-better, NVAS and the Spearman correlation with
Hofstede scores are higher-is-better. Vanilla has no country signal on CDEval,
so it has no correlation. Bold marks the best arm.}
\label{tab:mena_cdeval}
\setlength{\tabcolsep}{4pt}
\resizebox{\linewidth}{!}{%
\begin{tabular}{lccccc}
\toprule
Metric & Vanilla & Country prompt & CultureSPA prompt & Mean ensemble & DISCA \\
\midrule
MENAValues JSD $\downarrow$, 7 WVS countries & .6649 & .6479 & \textbf{.6362} & .6567 & .6709 \\
MENAValues JSD $\downarrow$, 7 pooled-profile countries & .6550 & .6365 & \textbf{.6261} & .6375 & .6535 \\
MENAValues JSD $\downarrow$, all 14 & .6600 & .6422 & \textbf{.6311} & .6471 & .6622 \\
MENAValues NVAS $\uparrow$, all 14 & .7369 & .7668 & \textbf{.7747} & .7652 & .7577 \\
\midrule
CDEval Spearman with Hofstede $\uparrow$, mean of 6 & n/a & \textbf{.268} & .223 & .137 & .127 \\
CDEval cross-country spread (mean SD) & $\approx 0$ & .045 & .036 & .102 & .102 \\
\bottomrule
\end{tabular}}
\end{table}

Neither benchmark supports DISCA over simple prompting. On MENAValues, DISCA
does not move the answer distribution closer than vanilla does (difference in
JSD over all 14 countries $[-.011,\,.015]$), and it is further from the human
distributions than the country prompt, the CultureSPA prompt and the mean
ensemble, with every one of those intervals above zero. On NVAS, DISCA does
improve on vanilla (interval $[.005,\,.037]$, 12 of 14 countries), but it
stays behind the CultureSPA prompt and the mean ensemble. On CDEval, DISCA
spreads the countries apart the most, together with the mean ensemble, yet it
tracks the Hofstede ordering the least of the four country-aware arms.

Two features of these benchmarks explain much of the gap. Half of the
MENAValues countries have no microdata of their own, so their personas are
copies of a regional profile and cannot carry a country signal. And the
Hofstede scores of CDEval measure a construct different from the WVS
descriptors that build the personas, while a prompt that names the country
lets the model use whatever it already knows about that country. We read these
two results as a boundary of the method: a WVS panel helps where the target
is close to what the WVS measures and where the country has its own
microdata, and it is not a general replacement for naming the country.

\section{Reproducibility and Use of Language Models}
\label{sec:reproducibility}
\label{sec:ai_use}

\paragraph{Reproducibility.} Algorithm~\ref{alg:disca} and
Table~\ref{tab:hyperparams} specify the controller that produced every
MultiTP number, including the country prior and the adaptive proposal scale.
Every run records the model identifier, the controller configuration, the
random seed, the country and scenario counts, and hashes of the human AMCE
input and the result files. The component decomposition of
Table~\ref{tab:decomposition} and the diagnostics of
App.~\ref{app:r2_reliability_audit} are recomputed from the logged
per-scenario controller state without new model calls, and the recomputation
reproduces every reported country MIS of the full method to within
$3\times10^{-16}$.

\paragraph{Language models as tools.} Gemini 3 Flash and GPT-5.6 Sol served
as the blinded pairwise judges of App.~\ref{app:makieval}. No model judgment
was treated as human ground truth. The authors checked the numerical results,
citations, proofs, protocols and implementation details, and take
responsibility for the manuscript.

\end{document}